\documentclass[10pt,letterpaper,twoside]{article}

\usepackage[margin=0.82in]{geometry}
\usepackage[T1]{fontenc}
\usepackage[utf8]{inputenc}
\usepackage{lmodern}
\usepackage{microtype}
\usepackage{graphicx}
\usepackage{subcaption}
\usepackage{booktabs}
\usepackage{tabularx}
\usepackage[numbers,square,sort&compress]{natbib}
\usepackage{tcolorbox}
\usepackage{bm}
\usepackage{amsmath}
\usepackage{amssymb}
\usepackage{mathtools}
\usepackage{amsthm}
\usepackage{algorithm}
\usepackage{algorithmic}
\usepackage{dsfont}
\usepackage{enumitem}
\usepackage{mathrsfs}
\usepackage{float}
\usepackage{xcolor}
\usepackage[textsize=tiny]{todonotes}
\usepackage{hyperref}

\usepackage[capitalize,noabbrev]{cleveref}

\theoremstyle{plain}
\newtheorem{theorem}{Theorem}[section]
\newtheorem{proposition}[theorem]{Proposition}
\newtheorem{lemma}[theorem]{Lemma}

\theoremstyle{definition}

\newtheorem{assumption}{Assumption}
\newtheorem{property}{Property}
\theoremstyle{remark}
\newtheorem{remark}[theorem]{Remark}

\newcommand{\safeincludegraphics}[2][]{%
  \IfFileExists{#2}{\includegraphics[#1]{#2}}{%
    \fbox{\parbox{0.85\linewidth}{\centering\small
      Missing figure: \texttt{\detokenize{#2}}}}%
  }%
}

\begin{document}

\twocolumn[{%
\begin{center}
{\LARGE\bfseries Learning to Allocate Resources with Censored Feedback\par}
\vspace{0.9em}
{\normalsize
\textbf{Giovanni Montanari}\textsuperscript{1,3} \quad
\textbf{Côme Fiegel}\textsuperscript{1,3} \quad
\textbf{Corentin Pla}\textsuperscript{1,2,3}\\[0.35em]
\textbf{Aadirupa Saha}\textsuperscript{4} \quad
\textbf{Vianney Perchet}\textsuperscript{1,2,3}\par}
\vspace{0.6em}
{\small
\textsuperscript{1} FairPlay Joint Team, Inria, France \quad
\textsuperscript{2} Criteo AI Lab, Paris, France\\
\textsuperscript{3} CREST, ENSAE, Institut Polytechnique de Paris\\
\textsuperscript{4} University of Illinois, Chicago (UIC)\par}
\end{center}
\vspace{0.2em}
\begin{minipage}{\textwidth}
\begin{abstract} We study the online resource allocation problem in which at each round, a budget $B$ must be allocated across $K$ arms under censored feedback. An arm yields a reward if and only if two conditions are satisfied: (i) the arm is activated according to an arm-specific Bernoulli random variable with unknown parameter, and (ii) the allocated budget exceeds a random threshold drawn from a parametric distribution with unknown parameter. Over  $T$ rounds, the learner must jointly estimate the unknown parameters and allocate the budget so as to maximize cumulative reward facing the exploration--exploitation trade-off. We prove a worst-case regret lower bound $\Omega(\sqrt{T})$ for the general model, even with known uniform threshold distributions. We then propose \texttt{RA-UCB}, an optimistic algorithm that leverages non-trivial parameter estimation and confidence bounds. When the budget $B$ is known at the beginning of each round and every arm receives a positive allocation at the optimum (interior optimum, with a known lower bound), \texttt{RA-UCB} achieves a regret bound of order $\widetilde{\mathcal{O}}(\sqrt{T})$, and even $\mathcal{O}(\mathrm{poly}\text{-}\log T)$ under stronger assumptions. For unknown, round-dependent budgets, we introduce \texttt{MG-UCB}, which allows within-round switching and infinitesimal allocations, and matches the regret guarantees of \texttt{RA-UCB}; we also discuss an implementable discretization. We then validate our theoretical results through experiments on real-world datasets.
\end{abstract}
\end{minipage}
\vspace{1.1em}
}]

\section{Introduction}
We consider the problem of allocating a limited resource budget across a finite set of options in order to maximize a given objective function. This problem is central in both theoretical and applied computer science, as it captures a wide range of real-world scenarios. Classically, allocations are computed once, from (possibly partial) system information, without adapting over time. The problem becomes substantially more challenging when allocations must be made online, i.e., sequentially over time, under uncertainty and partial feedback. In such settings, decisions must balance exploration, to learn unknown characteristics of the system, and exploitation, to optimize performance based on current knowledge. This naturally calls for online learning tools that update decisions from feedback.

In online advertising \cite{CTR-fondation,AdWords,click-model}, the resource corresponds to user attention--measured in time, impressions, or display opportunities--which must be allocated sequentially across ads with unknown effectiveness to maximize engagement metrics such as clicks or conversions. The main difficulty is uncertainty: ads appeal differently to users and their performance is learned only through noisy and sparse feedback. Since users arrive over time, decisions must be made online, continually balancing exploration of poorly known ads with exploitation of the most promising ones.
In mobile networks \cite{JULIAN}, bandwidth, power, or time slots must be allocated online to optimize throughput or latency, but feedback is often censored because failures may reflect either insufficient resources or unobserved channel conditions. Similarly, in cloud computing \cite{cloud}, providers allocate computational resources such as GPU time, memory, or storage to competing tasks under uncertain workloads, where incomplete performance feedback makes it difficult to disentangle resource insufficiency from external sources of variability. Beyond these classical domains, analogous challenges appear in education, where limited instructional time must be allocated across questions or exercises of varying difficulty to maximize learning outcomes under uncertainty about student knowledge, \cite{ed-1,ed-2,ed-3} and in monitoring, and protection tasks where time, personnel, or equipment must be assigned across locations to maximize detection, protection, or risk mitigation, based on partial or delayed feedback \cite{HEGAZY,CURTINpoliceplacement, NGUYENpoliceplacement, GHOLAMIpoliceplacement}. Across these domains, resources are limited and feedback is incomplete, motivating adaptive online allocation.

\subsection{Related Works}
Early work on resource allocation focused on offline models with known parameters, computing optimal static allocations under uncertainty in applications such as project management \cite{HEGAZY}, wireless communication \cite{JULIAN}, and multi-resource systems with fairness constraints \cite{JOE-WONG}. Subsequent work considered online resource allocation with sequential arrivals under known parameters \cite{DEVANUR,Mehta-bis}, but these frameworks do not address the exploration-exploitation trade-off when the environment must be learned.

Sequential resource allocation with learning naturally fits the multi-armed bandit paradigm \cite{LATTIMORE,SLIVKINS}, in particular models where success probabilities depend on allocated resources via threshold/cut-off effects. This includes the linear multi-resource allocation and semi-bandit settings of \cite{LATTIMORElinearmultiresourceallocation, LATTIMOREsemibandits}, as well as refinements improving regret bounds \cite{DAGAN} and the allocation of divisible resources with random rewards \cite{LICHEN2023}. Another related line assumes concave (diminishing returns) reward functions \cite{VIANNEY}. Our feedback is fundamentally different since rewards arise from a latent threshold mechanism with censored observations rather than smooth concave rewards. 

The closest setting is \citet{VERMACensoredsemibandits}, which studies threshold-based rewards under censored semi-bandit feedback. We generalize this model by introducing an arm-specific Bernoulli activation, creating stronger ambiguity since failures may result from either insufficient allocation or lack of activation. Related threshold-triggered and censored models include threshold bandits \cite{THRESHOLDBANDITS1,THRESHOLDBANDITS2,THRESHOLDBANDITS3} and dynamic pricing with censored feedback \cite{CENSOREDFEEDBACK1}, but they do not capture the combination of simultaneous multi-arm allocation, semi-bandit feedback, and latent activation. Latent Bernoulli conversions combined with censored random delays also appear in bandits with delayed conversions \cite{VERNADE2017,MANEGUEU2020}, where, however, censoring is not driven by a divisible budget shared across arms.

This problem can also be viewed as a special case of bandits with general reward/feedback functions \cite{ZUO}; in continuous allocation settings, the generic algorithm of \cite{ZUO} achieves $\mathcal{O}(T^{2/3})$ regret, whereas leveraging our threshold and censoring structure yields substantially sharper guarantees. This work is then related to Bandits with Knapsacks \cite{Badanidiyuru_2013, agrawal2014banditsconcaverewardsconvex}, but differs from classical BwK since we allocate a divisible resource simultaneously across arms under per-round constraints, and rewards involve latent activation with censored/ambiguous feedback.

Very recent work \cite{bachoc2026onlinebudget} studies a related problem where, as in our setting, reward probabilities depend on the allocated budget. Their success curves are known and only mean rewards are unknown, whereas we learn both activation probabilities and threshold distributions from feedback revealing thresholds only upon success and leaving failures ambiguous. While they consider a known budget, we also study budgets unknown at the start of each round, allowing within-round switching. Our complementary formulation is motivated by applications in education and advertising, with simulations calibrated on datasets from both domains.

\subsection{Contributions}

In Section \ref{sec: lower bound section}, we establish an information-theoretic \textit{regret lower bound} of order $\Omega(\sqrt{KT})$ for the general studied model. The bound already holds with known uniform threshold distributions and unknown activation probabilities. In this subclass, the reward is linear in the allocation, while the information available about each activation probability scales with the budget assigned to that arm. Thus, simultaneous fractional allocations and the observation of successful thresholds do not eliminate the classical exploration-exploitation difficulty.

In Section \ref{sec:algo}, we propose an \textit{algorithm}, \texttt{RA-UCB}, inspired by the classical UCB of \cite{AUER}, that tackles censored feedback by decoupling reward collection from parameter estimation. Rewards are collected at every interaction round, while parameter updates are performed only once every $K$ interactions using a separate estimation index. At each estimation round, the algorithm intentionally allocates a sufficiently large budget to a designated arm to obtain an informative (less censored) observation, and updates only that arm. This structured exploration yields regular high-quality samples that stabilize estimation while the algorithm exploits current estimates in the remaining rounds. We show that this decoupling is key to sharp regret guarantees under censored feedback. Finally, \texttt{RA-UCB} uses tailored estimators that explicitly separate the estimation of the Bernoulli activation probability from that of the threshold parameter, a central difficulty in this model.

We analyze the \textit{regret} of \texttt{RA-UCB} (Section  \ref{subsec:upperbound}) and show that, under an interior-optimum condition with a known lower bound $b$ on the optimal allocations, it achieves a bound of order $\widetilde{\mathcal{O}}(\sqrt{T})$, substantially improving upon the $\mathcal{O}(T^{2/3})$ regret guarantees obtained by more general sequential resource allocation algorithms \cite{ZUO}, by exploiting the specific structure of our problem. 
Under stronger assumptions on the distribution of the threshold and the set of parameters we show also a problem-dependent regret upper bound of $\mathcal{O}\big( (\log T)^2\big)$. 

Motivated by online advertising, where an ad can be shown to a user multiple times within a round for short time intervals, we consider the \textit{unknown budget model}, (Section \ref{sec: unknown budget}) where the learner does not observe in advance what is the budget available. In this setting we allow within-round switching between arms, i.e., the learner may allocate to an arm and later return to it. To address this scenario, we introduce \texttt{MG-UCB}, which implements, without prior knowledge of the budget, the same allocation principle as \texttt{RA-UCB}, and inherits its regret bounds when allocations can be made in infinitesimal increments. For its implementable discretization \texttt{MG-UCB-$\Delta$}, we heuristically argue that the discretization only adds a small, controlled error.

In Section \ref{sec:expe}, we finally complement our theoretical analysis with simulations on synthetic instances calibrated from two large-scale \textit{real-world datasets}. We consider (i) an education setting based on EdNet \cite{choi2020ednetlargescalehierarchicaldataset}, where a learner allocates a fixed time budget across questions for sequentially arriving students, and (ii) a display advertising setting based on the Criteo Attribution Modeling for Bidding dataset \cite{diemert2017attributionmodelingincreasesefficiency}, where each sequentially arriving user induces a known interaction budget that must be allocated across multiple advertising campaigns (arms). In both cases, the resulting benchmarks match our theoretical setting—divisible budgets, sequential decisions, and censored feedback—and empirically support the guarantees of \texttt{RA-UCB}.

\section{The Known Budget Model}\label{sec: known budget}

We consider a setting with $K$ arms and a horizon of $T$ rounds. At each round, the learner is given a budget $B$, representing the available amount of resources. For clarity, we first focus on the constant-budget case, where the same budget $B$ is available at every round. All our results extend straightforwardly to round-dependent budgets $B_t$, provided that this sequence is fixed and known in advance, and that each $B_t$ satisfies the conditions on the truncation level introduced in Section~\ref{subsec: estimation of lambda}.
The case of unknown budgets is deferred to Section \ref{sec: unknown budget}. 

At each round $t \in [T]$, the learner must decide how to allocate the available resource budget across the $K$ arms. This decision is represented by an allocation vector $\bm{x}_t = (x_{t,1}, \dots, x_{t,K})$, where $x_{t,i}$ denotes the amount of resource assigned to arm $i$ at round $t$. The allocation must satisfy the feasibility constraint $\sum_{i=1}^K x_{t,i} \leq B$.

Given an allocation vector, multiple arms may generate a reward. The objective of the learner is to select allocation vector $\bm{x}_t$ so as to maximize the cumulative number of successes over the $T$ rounds. Each arm $i \in [K]$ is characterized by two unknown parameters: (i) a success probability $p_i \in [0,1]$, capturing the intrinsic attractiveness of the arm, and (ii) a parameter $\lambda_i$ governing the distribution of the minimum amount of resource required for the arm to succeed, modeled as a random variable $X_{t,i} \sim \mathcal{G}(\lambda_i)$. These parameters jointly determine the reward generation process. Specifically, at round $t$, arm $i$ produces a reward equal to $1$ if and only if both of the following conditions are satisfied: (i) an activation Bernoulli random variable $Y_{t,i} \sim \mathcal{B}(p_i)$ equals $1$, and (ii) the allocated resource exceeds the required threshold, i.e., $x_{t,i} \geq X_{t,i}$. To ensure a well-defined problem, we assume that the threshold parameters are bounded.
In particular, for all arms $i \in [K]$, $\lambda_i \in \Lambda$, where $\Lambda = [m/B,M/B]$ with $M > m$.
Note that for each arm $i$, the pairs $(X_{t,i}, Y_{t,i})$ are drawn independently across rounds from a time invariant, arm-specific distribution.
At each round $t$, once an allocation vector $\bm{x}_t$ is chosen, the learner observes which arms generated a reward. For each successful arm, the realization of the threshold value $X_{t,i}$ is also revealed. For simplicity, we treat the activation variable $Y_{t,i}$ and the threshold variable $X_{t,i}$ as independent. This assumption is without loss of generality: since the feedback is censored, the value of $X_{t,i}$ is never observed when $Y_{t,i}=0$. Formally, this implies that the coupled random variables $(X_{t,i}, Y_{t,i})$ are, in practice, equivalent to two independent random variables $Y'_{t,i} \sim \mathcal{B}(p_i)$ and $X'_{t,i}$, where $X'_{t,i}$ has the law of $X_{t,i}$ conditionally on $Y_{t,i} = 1$.

Because the per-round budget $B$ is known, the order in which arms are considered is irrelevant. At each round, the learner only needs to select an allocation vector $\bm{x}_t$ based on its current estimates of the arm parameters. The main challenge comes from the asymmetric feedback. A success confirms both that the arm was activated and that the allocated resource exceeded the threshold, while a failure provides no such distinction. This ambiguity makes it difficult to estimate $p_i$ and $\lambda_i$ separately using standard empirical methods.

Let $G(x,\lambda) = \mathbb{P}(X_{t,i} \le x \mid \lambda)$ denote the cumulative distribution function of the threshold variable. We assume that $G$ satisfies standard regularity conditions: for any $x \in [0,B]$, the function $G(x,\lambda)$ is non-decreasing in $\lambda$, and there exists a constant $L_\lambda > 0$ such that for all $x \in [0,B]$ and all $\lambda, \lambda' \in \Lambda$,
$|G(x,\lambda) - G(x,\lambda')| \le L_\lambda |\lambda - \lambda'|.$
These assumptions are standard when $\lambda$ plays the role of an ease or scale parameter. Since the activation and threshold mechanisms are independent, the probability that arm $i$ generates a success at round $t$ given an allocated budget $x_{t,i}$ is
\begin{equation}
    \mathbb{P}(\text{success on arm } i \mid x_{t,i}) = p_i \, G(x_{t,i}, \lambda_i)
\end{equation}

\section{Regret Lower Bound}\label{sec: lower bound section}
We define as a benchmark the optimal allocation vector $\bm{x}^* = (x_1^*, \dots, x_K^*)$, which maximizes the expected reward assuming full knowledge of the arm parameters ${(p_i, \lambda_i)}_{i \in [K]}$. Since the environment is stationary, this optimal allocation is independent of the round. In contrast, at each round $t$, the learning algorithm selects an allocation vector $\bm{x}_t = (x_{t,1}, \dots, x_{t,K})$ based on its current estimates of the unknown parameters. We define the regret for round $t$, denoted $R_t$, as:
\begin{align}
R_t = \sum_{i=1}^K & [\mathds{1}\{\text{success on arm } i \text{ playing } x_i^*\}-\\
&\mathds{1}\{\text{success on arm } i \text{ playing } x_{t,i}\}] \notag
\end{align}
From this, we define the expected regret at round $t$, conditioned on the filtration $\mathcal{F}_{t-1}$ generated by past allocations and observations:
\begin{equation}\label{eq: E[R_t|F_{t-1}]}
\mathbb{E}[R_t|\mathcal{F}_{t-1}] = \sum_{i=1}^K \big[ p_iG(x_i^*,\lambda_i)-p_iG(x_{t,i},\lambda_i)\big]
\end{equation}
The conditioning on $\mathcal{F}_{t-1}$ is required since the allocation $\bm{x}_t$ is random, as it depends on both the algorithm’s internal randomness and the feedback observed up to round $t-1$. The total expected regret up to horizon $T$ is therefore :
\begin{equation}\label{eq: total cumulative regret}
\mathbb{E}[\mathcal{R}_T] = \sum_{t=1}^T \mathbb{E}[\mathbb{E}[R_t|\mathcal{F}_{t-1}]]= \sum_{t=1}^T \mathbb{E}[R_t]
\end{equation}
where the outer expectation is taken over the internal randomness of the algorithm and the observed feedback. We next give a worst-case lower bound over the general model of Section~\ref{sec: known budget}.

\begin{theorem}\label{thm: lower bound}
Let $2\leq K\leq T$. For every possibly randomized allocation algorithm, there exists an instance with independent thresholds $X_{t,i}\sim\mathrm{Unif}[0,B]$ and activation probabilities $p_i\in[1/2,3/4]$ such that
$$
\mathbb{E}[\mathcal{R}_T]\geq \frac{1}{32}\sqrt{KT}.
$$
The bound holds even when all threshold parameters $\lambda_i$ are known to the learner, so that only the activation probabilities $p_i$ remain unknown.
\end{theorem}

The proof is given in Appendix \ref{sec:app_lb}, which also embeds the uniform thresholds in a fixed parametric family satisfying the model's regularity conditions. The construction compares a reference instance with $p_i=1/2$ for every arm to alternatives in which one activation probability is increased by $\varepsilon$. Under uniform thresholds, the expected reward of arm $i$ is $p_i x_i/B$, so the optimal allocation assigns the entire budget to the improved arm. Conditional on success, the observed threshold is uniform on $[0,x_i]$ and contains no additional information about $p_i$. A change-of-measure argument with $\varepsilon$ of order $\sqrt{K/T}$ then gives the result. This lower bound holds for the general model of Section~\ref{sec: known budget} and shows that exploration cannot be avoided in general. The upper bounds of Section~\ref{subsec:upperbound} are instead established under additional structural assumptions on the instance, so the two results are complementary; within such structured subclasses, the achievable rate may also depend on the curvature of the reward (see Theorem~\ref{thm: regret bound of logT}).

\section{\texttt{RA-UCB} Algorithm}\label{sec:algo}

\subsection{Estimation of $\lambda_i$}\label{subsec: estimation of lambda}
We now describe how the parameter $\lambda_i$ is estimated from the censored feedback.
At each round $t$, the learner observes the following feedback variable on arm $i$:
\begin{equation}
F_{t,i} =
\begin{cases}
+\infty & \text{if no success on arm $i$ at round $t$}, \\
X_{t,i} & \text{otherwise}
\end{cases}
\end{equation}
A threshold is revealed only on success, and success requires $X_{t,i}\leq x_{t,i}$. Hence, for an allocation $x_{t,i}=a$ with $p_iG(a,\lambda_i)>0$, an observed threshold has conditional distribution, for $0\leq z\leq a$,
\begin{align*}
&\mathbb{P}\big(F_{t,i}\leq z \mid F_{t,i}<\infty,\ x_{t,i}=a,\ \mathcal{F}_{t-1}\big)\\
&\qquad=\frac{G(z,\lambda_i)}{G(a,\lambda_i)},
\end{align*}
which depends on the allocation $a$. Averaging observed thresholds collected under different allocations therefore does not estimate a fixed function of $\lambda_i$. We remove this dependence by truncating all observations used for estimating $\lambda_i$ at a common, known level, which is never larger than the allocations at which they are collected.

\begin{assumption}\label{ass: truncation level}
The learner knows a constant $b\in(0,B]$ such that $x_i^*\geq b$ for every $i\in[K]$.
\end{assumption}
Assumption \ref{ass: truncation level} requires a known lower bound on every coordinate of the optimal allocation; in particular, every arm is active at the optimum, $x_i^*>0$ for all $i\in[K]$. The assumption is used to guarantee that informative samples are collected at a linear rate, to give a known lower bound on the quantities $C_i^*$ appearing in the confidence radii (see Section~\ref{subsec:algo-radii}), and to keep all arms in the regret analysis. The validity of the estimator below does not depend on it.
Such a level is natural in the applications we target: in a quiz, every question is typically granted at least a minimal answering time, and in advertising, campaigns are usually guaranteed a minimal number of impressions. Moreover, $b$ need not be tight: any conservative (small) value satisfying Assumption~\ref{ass: truncation level} can be used, at the price of larger constants, since the quantities $q_i=p_iG(b,\lambda_i)$ and $C_b=G(b,m/B)$ entering the bounds decrease with $b$.

\paragraph{Selected observations.}
We call \emph{candidate rounds} of arm $i$ its $\lfloor\log T\rfloor$ initialization rounds and its estimation rounds, one in every block of $K$ main-phase rounds (Algorithm~\ref{algo: RA_UCB pseudocode}). This schedule is deterministic. Let $\mathcal{C}_i(t)\subseteq[t]$ denote the set of candidate rounds of arm $i$ up to round $t$, and $c_{t,i}:=|\mathcal{C}_i(t)|$. Estimators of arm $i$ are updated only at its candidate rounds.
Let $\mathcal{E}_{t,i}$ be the event that $t\in\mathcal{C}_i(t)$ and $x_{t,i}\geq b$. Since the allocation and the round schedule are determined before the outcome, $\mathcal{E}_{t,i}\in\mathcal{F}_{t-1}$. Define the finite-valued variables
\begin{gather*}
W_{t,i}:=\mathds{1}\{F_{t,i}\leq b\},\\
V_{t,i}:=\begin{cases}F_{t,i}, & F_{t,i}\leq b,\\ 0, & \text{otherwise},\end{cases}\qquad V_{t,i}\in[0,b].
\end{gather*}
On $\mathcal{E}_{t,i}$ we have $x_{t,i}\geq b$, so $\{F_{t,i}\leq b\}=\{Y_{t,i}=1,\ X_{t,i}\leq b\}$ does not depend on the allocation, and, on $\mathcal{E}_{t,i}$, conditionally on $\mathcal{F}_{t-1}$, for every $0\leq z\leq b$,
\begin{equation}\label{eq: truncated law}
\mathbb{P}\big(F_{t,i}\leq z \mid W_{t,i}=1,\ \mathcal{F}_{t-1}\big) = \frac{G(z,\lambda_i)}{G(b,\lambda_i)}.
\end{equation}
This conditional distribution depends only on the unknown parameter $\lambda_i$ and on the known level $b$. We define the deterministic mapping
\begin{equation}
\mu(\lambda) \coloneqq \mathbb{E}[X \mid X \le b],
\qquad X \sim G(\cdot,\lambda),
\end{equation}
and set $\mu_i:=\mu(\lambda_i)$.
\begin{assumption}\label{ass: lipschitz of mu-1}
    The function $\mu(\lambda)=\mathbb{E}_\lambda[X\mid X\leq b]$ is strictly monotone on the parameter domain and continuously differentiable. Moreover its inverse $\mu^{-1}$ is Lipschitz with constant $1/L_\mu$.
\end{assumption}
Define the number of selected informative samples on arm $i$ up to round $t$ as
$$
n_{t,i} = \sum_{u=1}^t \mathds{1}_{\mathcal{E}_{u,i}}W_{u,i}.
$$
We estimate $\mu_i$ using the empirical mean of the selected truncated observations:
\begin{equation}
    \hat{\mu}_{t,i} = \frac{1}{n_{t,i}} \sum_{u=1}^t \mathds{1}_{\mathcal{E}_{u,i}}V_{u,i}.
\end{equation}
The number $n_{t,i}$ is random and the indicators $W_{u,i}$ are not predictable; concentration is nevertheless obtained by a supermartingale argument with the centering $V_{u,i}-\mu_iW_{u,i}$ (Lemma \ref{lemma: conc ineq for mu_i}, Appendix~\ref{app_sec: concentration inequalities}). Finally, we define the estimator of $\lambda_i$ as:
\begin{equation}
    \hat{\lambda}_{t,i} = \mu^{-1}\big(\Pi_{\mu(\Lambda)}(\hat{\mu}_{t,i})\big)
\end{equation}
where $\Pi_{\mu(\Lambda)}$ denotes the projection onto the interval $\mu(\Lambda)$, so that $\hat{\lambda}_{t,i}\in\Lambda$ is always well defined; when $n_{t,i}=0$, $\hat{\lambda}_{t,i}$ is set to an arbitrary value in $\Lambda$. Since $\mu_i\in\mu(\Lambda)$ and the projection is $1$-Lipschitz, $|\Pi_{\mu(\Lambda)}(\hat{\mu}_{t,i})-\mu_i|\leq|\hat{\mu}_{t,i}-\mu_i|$.
Using Lipschitz continuity of $\mu^{-1}$ (Assumption \ref{ass: lipschitz of mu-1}), $\hat{\lambda}_{t,i}$ satisfies the following concentration bound.
\begin{lemma} \label{lemma: conc ineq lambda}
    For any arm $i \in [K]$, round $t \in [T]$, $\ell \in [t]$, and $\alpha \in (0,1)$, we have:
    $$
    \mathbb{P} \left( |\hat{\lambda}_{t,i} - \lambda_i| \geq \frac{b}{L_\mu} \sqrt{\frac{\log(1/\alpha)}{2n_{t,i}}} \text{ and } n_{t,i} = \ell \right) \leq 2\alpha,
    $$
    where $1/L_\mu$ is the Lipschitz constant of $\mu^{-1}$.
\end{lemma}
See Appendix \ref{app_sec: concentration inequalities} for the proof.

\subsection{Estimation of $p_i$}
Define the binary success indicator:
$$
f_{t,i} = \mathds{1}\{F_{t,i} \leq B\}
$$
Conditioning on the filtration $\mathcal{F}_{t-1}$ generated by all allocations and observed feedback up to round $t-1$:
$$
\mathbb{E}\big[f_{t,i} \mid \mathcal{F}_{t-1}\big] = p_i G(x_{t,i}, \lambda_i)
$$
Using estimate $\hat{\lambda}_{t,i}$, we define:
$$
\hat{r}_{u,i}(t) = G(x_{u,i},\hat{\lambda}_{t,i})
$$
We estimate $p_i$ from the candidate rounds of arm $i$ only:
\begin{equation}
    \hat{p}_{t,i} = \dfrac{\sum_{u\in\mathcal{C}_i(t)} f_{u,i}}{\sum_{u\in\mathcal{C}_i(t)} \hat{r}_{u,i}(t)}
\end{equation}
if $\hat{p}_{t,i} > 1$, it is clipped to $1$. Restricting to candidate rounds is what keeps the denominator large: on these rounds the allocation to arm $i$ is $B$ (initialization) or, on the good event, at least $x_i^*$ (estimation rounds, see Assumption~\ref{ass:OMA}), whereas on the other rounds it can be arbitrarily small.

We introduce the \emph{good event} $\mathcal{N}_t$, defined as the event that all estimators $\hat{\lambda}_{t,i}$ and $\hat{p}_{t,i}$ lie within their confidence intervals at round $t$. Let $\mathscr{G}_t := \mathcal{N}_1 \cap \mathcal{N}_2 \cap \ldots \cap \mathcal{N}_t$ denote the event that all good events occur up to time $t$. Recall that $\bm{x}^*$ denotes the optimal allocation. We then have the following concentration result for $\hat{p}_{t,i}$.
\begin{lemma}\label{lemma: conc ineq p_i}
    For any arm $i \in [K]$ such that the optimal solution $x_i^* > 0$, round $t \in [T]$ and $\alpha \in (0,1)$:
    \begin{align*}
    &\mathbb{P} \Big( |\hat{p}_{t,i}-p_i| \geq \frac{L_\lambda}{L_\mu}\frac{b (1+p_i)}{C_i^*}\sqrt{\frac{\log(1/\alpha)}{2n_{t,i}}} \; \cap \; \mathscr{G}_{t-1} \Big)\\
    &\qquad\leq 2\alpha^{\frac{b^2L_\lambda^2}{4L_\mu^2}} + 2\alpha t
    \end{align*}
    where $C_i^*=G(x_i^*,m/B)$.
\end{lemma}
See Appendix \ref{app_sec: concentration inequalities} for the proof.

\subsection{Algorithm}\label{subsec:algo-radii}
The benchmark we compare against is the solution of the following optimization problem:
\begin{equation} \label{opt problem: OPT maximization problem}
\mathcal{P}(\bm{\lambda}, \bm{p}): \quad
\begin{aligned}
&\underset{\boldsymbol{x} \in \mathbb{R}^K_{+}}{\text{arg max}} \quad \sum_{i=1}^K p_i G(x_i,\lambda_i) \\
&\text{subject to} \quad \sum_{i=1}^K x_i \leq B
\end{aligned}
\end{equation}
To strengthen exploration and ensure several key properties, we introduce an auxiliary parameter vector $\bm{\lambda}'$ together with a surrogate objective function, which jointly define the following core optimization problem:
\begin{equation} \label{opt problem: surrogate algorithm maximization problem}
\tilde{\mathcal{P}}(\bm{\lambda}, \bm{\lambda'}, \bm{p}): \quad
\begin{aligned}
&\underset{\boldsymbol{x} \in \mathbb{R}^K_{+}}{\text{arg max}} \quad \sum_{i=1}^K p_i \widetilde{G}(x_i,\lambda_i,\lambda_i') \\
&\text{subject to} \quad \sum_{i=1}^K x_i \leq B
\end{aligned}
\end{equation}
The surrogate function $\tilde{G}(x,\lambda,\lambda')$ of $G(x,\lambda)$ is constructed so as to satisfy the following properties, which will be used throughout the analysis.
\begin{property} \label{ass:what we want}
\begin{enumerate}[label=\textup{(\roman*)}, ref=\roman*]
    \item (Consistency)\label{prp: Consistency}  For all $x \in [0,B]$, $\lambda \in \Lambda$, $\widetilde{G}(x, \lambda, \lambda) = G(x, \lambda)$.
    \item (Boundedness and Monotonicity)\label{prp: Boundness and Monotonicity}
    For all $\lambda, \lambda' \in \Lambda$, $\widetilde{G}(x, \lambda, \lambda') \leq 1$ and $\widetilde{G}(x, \lambda, \lambda') $ is non-decreasing in $x$.
    \item (Boosting Monotonicity)\label{prp: Boosting Monotonicity}
    $\tilde{g}(x, \lambda, \lambda')=\partial_x\widetilde{G}(x, \lambda, \lambda')$ is non-decreasing in $\lambda'$ and non-increasing in $\lambda$ for all $x \in [0,B]$.
    \item (Uniform Bias Control)\label{prp: Uniform Bias Control}
    \setlength{\abovedisplayskip}{6pt}
    \setlength{\belowdisplayskip}{2pt}
    For all $\lambda,\lambda' \in \Lambda$,
    \begin{equation*}
    \sup_{x \in [0,B]}\big|\widetilde{G}(x,\lambda,\lambda')-G(x,\lambda)\big|
    \le L_\Delta|\lambda-\lambda'|.
    \end{equation*}
    \item (Lipschitz Continuity in $\lambda$ and $\lambda'$) \label{prp:lipschitz_tildeG}
\setlength{\abovedisplayskip}{6pt}
\setlength{\belowdisplayskip}{1pt}
There exist constants $L_{\widetilde{\lambda}},L_{\widetilde{\lambda}'}>0$ such that
for all $x\in[0,B]$ and all $\lambda,\bar\lambda,\lambda',\bar\lambda' \in \Lambda$,
\begin{equation*}
\big|\widetilde{G}(x,\lambda,\lambda')-\widetilde{G}(x,\bar\lambda,\bar\lambda')\big|
\le L_{\widetilde{\lambda}}\,|\lambda-\bar\lambda|
+ L_{\widetilde{\lambda}'}\,|\lambda'-\bar\lambda'|.
\end{equation*}
\end{enumerate}
\end{property}
These are not assumptions on the underlying model, but design requirements that guide the choice of the surrogate function.

The \emph{boosted} parameter estimates are defined using confidence radii with a fixed, convenient choice of $\alpha$ :
\begin{align*}
\hat{\lambda}_{t,i}^\pm &= \min\left\{\frac{M}{B},\max\left\{\frac{m}{B}, \hat{\lambda}_{t,i} \pm \frac{b}{L_\mu} \sqrt{\frac{3\log T}{2n_{t,i}}} \right\} \right\} \\
\hat{p}_{t,i}^\pm &= \min\left\{1, \max\left\{0, \hat{p}_{t,i} \pm \frac{L_\lambda}{L_\mu}\frac{2b}{C_b} \sqrt{\frac{3\log T}{2n_{t,i}}} \right\} \right\}
\end{align*}
where $C_b:=G(b,m/B)>0$ is a known constant. This corresponds to the confidence level $\alpha=T^{-3}$, used uniformly over all rounds. By Assumption~\ref{ass: truncation level} and monotonicity of $G$ in $x$, $C_i^*=G(x_i^*,m/B)\geq G(b,m/B)=C_b$, and $1+p_i\leq 2$. Hence $(1+p_i)/C_i^*\leq 2/C_b$, so the interval used by the algorithm for $p_i$ contains the one of Lemma~\ref{lemma: conc ineq p_i}, and its coverage follows directly from that lemma.

Algorithm \ref{algo: RA_UCB pseudocode} presents the pseudocode of Resource Allocation Upper Confidence Bound (\texttt{RA-UCB})
\begin{algorithm}[htbp]
\caption{\texttt{RA-UCB} (High-Level Pseudocode)}
\label{algo: RA_UCB pseudocode}
\begin{algorithmic}[1]
\STATE \textbf{\underline{Input}}: $T$, $K$, $B$, truncation level $b$.
\STATE $t \gets 1$ (reward index), $t' \gets 1$ (estimation index), $i \gets 1$
\STATE \textbf{\underline{Initialization Phase}}: ($t \in \{1,\ldots,K\lfloor \log T \rfloor\}$)
\STATE For each arm $k \in [K]$, allocate all $B$ on $k$ for $\lfloor\log T \rfloor $ rounds, observe rewards, and update parameters.
\STATE \textbf{\underline{Main Phase}}: ($t \in \{K\lfloor \log T \rfloor+1,\ldots,T\}$)
\WHILE{$t \leq T$}
    \STATE Compute confidence bounds $\{\hat{p}_{t',k}^\pm,\hat{\lambda}_{t',k}^\pm\}_{k \in [K]}$
    \STATE $(\bar{\lambda}_{t',i},\bar{\lambda}'_{t',i},\bar{p}_{t',i}) \gets
    (\hat{\lambda}_{t',i}^-,\hat{\lambda}_{t',i}^+,\hat{p}_{t',i}^+)$ and
    \STATE $(\bar{\lambda}_{t',k},\bar{\lambda}'_{t',k},\bar{p}_{t',k}) \gets
    (\hat{\lambda}_{t',k}^+,\hat{\lambda}_{t',k}^-,\hat{p}_{t',k}^-)$ $\forall k\neq i$
    \STATE $\bm{x}_{t} \gets \tilde{\mathcal{P}}(\bar{\bm{\lambda}}_{t'},\bar{\bm{\lambda}'}_{t'},\bar{\bm{p}}_{t'})$
    \STATE Play allocation $\bm{x}_{t}$, get feedback and reward
    \STATE Update the estimates of arm $i$ only ($\hat{p}_{t'+1,i}$ always, $\hat{\lambda}_{t'+1,i}$ if $x_{t,i}\geq b$ and $F_{t,i}\leq b$)
    \STATE $i \gets i+1$, \ $t \gets t+1$
    \IF{$i > K$}
        \STATE $i \gets 1$, \ $t' \gets t'+1$
    \ENDIF
\ENDWHILE
\end{algorithmic}
\end{algorithm}

The key innovation of \texttt{RA-UCB} lies in the use of two distinct notions of rounds.
The index $t \in \{1,\ldots,T\}$ denotes the actual \emph{interaction rounds}, during which rewards are collected.
In contrast, the index $t'$ refers to \emph{estimation rounds}, with $t' \in \{1,\ldots, \lfloor (T - K\lfloor \log T \rfloor)/K \rfloor \}$.
The estimation procedure is batched: after each block of $K$ interaction rounds, the algorithm can collect at most one new informative sample per arm. When an estimation round corresponding to arm $i$ occurs, the algorithm deliberately \emph{boosts} this arm to favor its exploration. Concretely, the algorithm selects the parameter set $\{\hat{\lambda}_{t',i}^-,\hat{\lambda}_{t',i}^+,\hat{p}_{t',i}^+\}$ for arm $i$ and $\{\hat{\lambda}_{t',k}^+,\hat{\lambda}_{t',k}^-,\hat{p}_{t',k}^-\}$ for all the other arms $k \neq i$ to give as input to $\widetilde{\mathcal{P}}$, which in practice induces the algorithm to allocate more budget to arm $i$, favoring its exploration. This argument is formalized by the following assumption.

\begin{assumption}[Oracle Monotonicity]
\label{ass:OMA}
Fix an arm $i\in[K]$. Assume two sets of parameters:
\begin{itemize}
    \item $(\bm{\lambda}, \bm{p}) =
   (\{\lambda_1, p_1\}, \ldots, \{\lambda_i, p_i\}, \ldots,$\\
   $\qquad\{\lambda_K, p_K\})$
    \item $(\bm{\bar{\lambda}}, \bm{\bar{\lambda}'}, \bm{\bar{p}}) = (\{\lambda_1^+, \lambda_1^-, p_1^-\}, \ldots,$\\
   $\qquad\{\lambda_i^-, \lambda_i^+, p_i^+\}, \ldots, \{\lambda_K^+, \lambda_K^-, p_K^-\})$
\end{itemize} 
where for all $k \in [K]$ the bounds satisfy $\lambda_k^- \leq \lambda_k \leq \lambda_k^+$ and $p_k^- \leq p_k \leq p_k^+$. Let $\bm{x} = \mathcal{P}(\bm{\lambda}, \bm{p})$ and $\bm{\bar{x}} = \widetilde{\mathcal{P}}(\bm{\bar{\lambda}}, \bm{\bar{\lambda}'}, \bm{\bar{p}})$. We assume that the oracle can be chosen so that it returns solutions satisfying $\bar{x}_i \geq x_i$.
\end{assumption}
In words, when arm $i$ is \textit{boosted} (and all other arms are made pessimistic), Oracle Monotonicity ensures that the oracle returns an optimal allocation assigning weakly more budget to arm $i$ than the allocation under the true parameters (on the good event).
\begin{lemma}\label{Lemma: structural property of the Opt Problem}
A sufficient condition for Assumption \ref{ass:OMA} to hold is that, for all $\lambda,\lambda'\in\Lambda$, $\widetilde{G}(x,\lambda,\lambda')$ is strongly concave in $x\in[0,B]$.
\end{lemma}
The proof leverages Boosting Monotonicity (Property~\ref{ass:what we want}) together with the diminishing-returns structure induced by strong concavity (see Appendix \ref{app: proof of structural property of Opt Problem}).
Assumption \ref{ass:OMA} is a key ingredient in our analysis: it guarantees that boosting arm $i$ forces the oracle to allocate at least as much budget to $i$ as the clairvoyant optimum, which in turn provides a deterministic lower bound on the expected number of successes collected on arm $i$ up to round $t$ and stabilizes exploration.

Several common families of distributions with cumulative distribution function $G(x,\lambda)$ admit the construction of a surrogate $\widetilde{G}$ that satisfies all the properties stated above and is moreover \emph{strongly concave} in $x$. In Lemma \ref{lemma:boosting_monotone_exp_family} we show that these properties hold for the whole class of distributions admitting an exponential-form representation. This includes, for instance, the Exponential distribution, as well as Weibull and Gamma distributions provided their shape parameters are fixed and satisfy $k\le 1$ (Weibull) and $\alpha\le 1$ (Gamma). More generally, many other distribution families fall within this class (Lomax, Beta and many others).
If strong concavity is not required and one instead relies on an oracle satisfying Assumption~\ref{ass:OMA}, then no restriction on the shape parameter is needed (e.g., for Weibull or Gamma). Appendix~\ref{app: how to build G tilde} explains how $\widetilde{G}$ can be constructed explicitly.

\begin{remark}
    A naive Explore-Then-Commit (ETC) approach, namely, first estimating the arm parameters and then committing to the allocation that is optimal with respect to these estimates, is not competitive in this setting. Indeed suppose we dedicate order $T^{\alpha}$ with $\alpha \in [0,1]$ rounds to exploring each arm, for a total exploration length of $KT^{\alpha}$. The statistical accuracy of the resulting parameter estimates would then be at most $T^{-\alpha/2}$. This estimation error, compounded over the remaining $T-KT^{\alpha}$ rounds, ultimately yields a regret of order $\tilde{\mathcal{O}}(T^\alpha+T^{1-\alpha/2})$.
    The value of $\alpha$ that minimizes the asymptotic behavior is $\alpha = 2/3$. Therefore, the final upper bound is, in the best case, $\tilde{\mathcal{O}}(T^{2/3})$.
\end{remark}

\subsection{Regret Upper Bounds}\label{subsec:upperbound}
We now derive upper bounds on the regret of \texttt{RA-UCB}, as defined in \eqref{eq: total cumulative regret}.
\begin{theorem}\label{thm: regret sqrt(T)}
Assume Assumption~\ref{ass: truncation level} (which implies $x_i^*>0$ for every $i\in[K]$, i.e., an interior optimal allocation) and that the truncation level and the Lipschitz constants satisfy $b L_\lambda/L_\mu \geq \sqrt{8/3}$.
Then \texttt{RA-UCB} achieves the regret bound:
$$
\mathbb{E}[\mathcal{R}_T] \leq C_{\text{sqrt}}\,\sqrt{K T \log T} + \mathcal{O}(K^2 \log T),
$$
where $C_{\text{sqrt}}>0$ is a problem-dependent constant.
\end{theorem}
We emphasize that $C_{\text{sqrt}}$ involves sums over $i\in[K]$; hence, when making the dependence on $K$ explicit, the bound scales as $\widetilde{\mathcal{O}}(K^{3/2}\sqrt{T})$.
The proof follows a standard UCB decomposition. Starting from \eqref{eq: E[R_t|F_{t-1}]}, we upper bound the per-round regret by the estimation errors of $(p_i,\lambda_i)$, controlled by the confidence intervals. The main technical step is to lower bound $\mathbb{E}[n_{t,i}]$, which in turn yields an upper bound on $\mathbb{E}\big[1/\sqrt{n_{t,i}}\big]$. This is possible because each arm is boosted once every $K$ rounds, and on the good event the Oracle Monotonicity Assumption \ref{ass:OMA} ensures $x_{t,i} \geq x_i^*$. Hence, by Assumption~\ref{ass: truncation level}, at each estimation round for arm $i$ we have $x_{t,i}\geq x_i^*\geq b$, so the round is selected ($\mathcal{E}_{t,i}$ holds) and yields an informative sample with probability $p_iG(b,\lambda_i)$,
and a coupling with a binomial variable gives, on the cumulative good event,
$$
\mathbb{E}\Big[1/\sqrt{\max\{n_{t,i},1\}}\Big] \;\lesssim\; \sqrt{K/(t q_i)}
$$
(Lemma~\ref{lemma: counting}), where $q_i:=p_iG(b,\lambda_i)>0$ is the probability that a selected round of arm $i$ yields an informative sample, i.e., a success with observed threshold at most $b$. See Appendix~\ref{app_sec: regret of sqrt(T)} for the full proof.

We now consider a \emph{diminishing-returns regime} with an interior optimal allocation, imposing the following additional assumptions.
\begin{assumption} \label{ass:uniform_sc_smooth}
    \begin{enumerate}[label=\textup{(\roman*)}, ref=\roman*]
    \item (Minimum activation probability) For all arms $i \in [K]$, $p_i \geq p_{\min}>0$.
    \item (Strong concavity and Smoothness in $x$) For all $\lambda,\lambda' \in \Lambda$, $\widetilde{G}(x,\lambda,\lambda')$ is twice continuously differentiable in $x$ on $[0,B]$ and there exist constants $\mu_G,L_G>0$ such that for all $x\in[0,B]$,
$$
-\partial_{xx}\widetilde{G}(x,\lambda,\lambda') \in [\mu_G,\,L_G]
$$
    \item (Gradient Lipschitzness in parameters) Let $\widetilde{g}:=\partial_x\widetilde{G}$. There exist constants $G_{\max}$, $L_{g,\lambda}$, $L_{g,\lambda'}>0$ such that for all $x\in[0,B]$ and all $\lambda$, $\lambda'$, $\bar\lambda$, $\bar\lambda'\in\Lambda$, we have $|\widetilde{g}(x,\lambda,\lambda')|\leq G_{\max}$ and
$$
|\widetilde{g}(x,\lambda,\lambda')-\widetilde{g}(x,\bar\lambda,\bar\lambda')| \leq L_{g,\lambda}|\lambda-\bar\lambda|+L_{g,\lambda'}|\lambda'-\bar\lambda'|
$$
    \end{enumerate}
\end{assumption}

Note that Assumption~\ref{ass: truncation level} implies an interior optimal allocation, $x_i^*\geq b>0$ for every $i\in[K]$. Interiority is the property used explicitly in the polylogarithmic argument, to cancel the first-order term of the regret. Interiority is relative to the simplex $\{\bm{x}\geq0:\sum_i x_i=B\}$; it does not mean that budget is left unused. The constants of the polylogarithmic bound may deteriorate as $b$ approaches zero.

For the following result, the oracle is taken to return a budget-saturating maximizer, $\sum_i x_i=B$, with remaining ties resolved deterministically. Such a maximizer can always be chosen by monotonicity.

\begin{theorem}\label{thm: regret bound of logT}
    Under the assumptions of Theorem~\ref{thm: regret sqrt(T)} and additional Assumption~\ref{ass:uniform_sc_smooth}, \texttt{RA-UCB} achieves the improved regret bound:
    $$
    \mathbb{E}[\mathcal{R}_T] \leq C_{\text{log}} \big((\log T)^2+\log T\big) + \mathcal{O}(K^2\log T)
    $$
    where $C_{\text{log}} >0$ is a problem-dependent constant.
\end{theorem}
Here again, factoring the dependence on $K$ out of $C_{\text{log}}$, the regret scales as $\widetilde{\mathcal{O}}\big((K\log T)^2\big)$.
The proof follows the same outline as Theorem \ref{thm: regret sqrt(T)}. Positive activation probabilities and strong concavity of $\widetilde{G}$ make the true objective strongly concave. Interiority of the true optimum cancels the first-order term along budget-preserving directions. Smoothness and stability under gradient perturbations then control regret by squared confidence radii. All arms receive positive allocation under the optimum, so no coordinates are discarded in this argument. The full proof is in Appendix \ref{app:logT-regret-ub-generic}. In Appendix \ref{app_sec: exponential distribution case}, we provide a complete treatment of the exponential-threshold case: including the construction of the estimators, the surrogate objective, and the resulting derivation of the regret bounds. Under Assumption~\ref{ass: truncation level}, the constant of Theorem~\ref{thm: regret sqrt(T)} is moreover uniform over the class of instances sharing the same truncation level $b$ (Appendix \ref{app_sec: exponential distribution case}).

\section{The Unknown Budget Model}\label{sec: unknown budget}
We now study an unknown budget variant where the available budget $B_t$ is not revealed at the start of round $t\in[T]$. The sequence $\{B_t\}_{t\in[T]}$ may be arbitrary (stochastic or adversarial), and we only assume the uniform bound $0 < B_t \leq B_{\max}$ for every $t\in[T]$.
All other aspects of the model remain unchanged. Throughout this section we assume that (i) within a round the learner may \textit{switch} between arms and return to previously visited arms, and (ii) allocations can be made in infinitesimal increments. Fix a round $t$ and the (boosted) parameters used by the algorithm in that round,
$\bm{\bar{\theta}}_{t-1}=(\bar{\bm{p}}_{t-1},\bar{\bm{\lambda}}_{t-1},\bar{\bm{\lambda}}'_{t-1})$.
Consider the surrogate objective:
$$
\widetilde{F}(\bm{x},\bm{\bar{\theta}}_{t-1}):=\sum_{i=1}^K \bar p_{t-1,i}\,\widetilde{G}(x_i,\bar\lambda_{t-1,i},\bar\lambda'_{t-1,i})
$$
over $\mathcal{X}(B):=\Big\{\bm{x}\in\mathbb{R}^K_+:\sum_{i=1}^K x_i\le B\Big\}$.
Assume that for all admissible $\lambda,\lambda'\in\Lambda$ the map $\widetilde{G}(x,\lambda,\lambda')$ is differentiable and concave in $x$ on $[0,B_{\max}]$ (in particular, its derivative is non-increasing), and that a deterministic tie-breaking rule is fixed when maximizers are not unique.\footnote{Under strong concavity of $x\mapsto \widetilde{G}(x,\lambda,\lambda')$, the optimizer is unique and tie-breaking is unnecessary.}
Define the marginal gains
\begin{align*}
m_i(x,\bm{\bar{\theta}}_{t-1})&:=\frac{\partial}{\partial x_i}\widetilde{F}(\bm{x},\bm{\bar{\theta}}_{t-1})\\
&=\bar p_{t-1,i}\,\widetilde{g}(x,\bar\lambda_{t-1,i},\bar\lambda'_{t-1,i})
\end{align*}
where $\widetilde{g}:=\partial_x\widetilde{G}$. Since $B_t$ is not known in advance, we cannot compute $\bm{x}_t$ beforehand by solving $\widetilde{\mathcal{P}}$ with budget $B_t$. Instead, we implement the corresponding KKT conditions via a water-filling dynamics.
Let $x_{t,i}(s)$ denote the cumulative allocation to arm $i$ up to intra-round step $s\in[0,B_t]$ (so $\sum_i x_{t,i}(s)=s$). The algorithm allocates an infinitesimal $ds$ to an arm $I(s)\in \underset{i \in [K]}{\text{arg max }}\; m_i\big(x_{t,i}(s),\bm{\bar{\theta}}_{t-1}\big)$ and sets $x_{t,I(s)}(s+ds)=x_{t,I(s)}(s)+ds$ (leaving the other coordinates unchanged), until $s=B_t$.
Under the concavity assumptions above, this water-filling rule produces an optimizer of $\max_{\bm{x}\in\mathcal{X}(B_t)}\widetilde{F}(\bm{x},\bm{\bar{\theta}}_{t-1})$. Hence it is equivalent to solving $\widetilde{\mathcal{P}}$ at budget $B_t$.

We define Marginal Greedy Upper Confidence Bound (\texttt{MG-UCB}) as follows: the estimation and boosting steps (confidence bounds, construction of $\bm{\bar{\theta}}_{t-1}$, and cycling over arms) are identical to \texttt{RA-UCB}, while the within-round allocation is executed through the above water-filling procedure, and therefore does not require prior knowledge of $B_t$ (see Algorithm \ref{algo: MG_UCB pseudocode} in Appendix \ref{app_sec: unknown budget}).
\begin{lemma}\label{lemma: equivalence RA-UCB and MG-UCB}
For any fixed budget B and conditioning on the filtration $\mathcal{F}_{t-1}$ (so that the estimators $\widehat{\bm{\theta}}_{t-1}$ are fixed), the allocation produced by \texttt{MG-UCB} coincides with the one produced by \texttt{RA-UCB} when the latter solves $\widetilde{\mathcal{P}}$ with budget $B$:
$$
\bm{x}_t^{\texttt{MG-UCB}}(B,\widehat{\bm{\theta}}_{t-1})=\bm{x}_t^{\texttt{RA-UCB}}(B,\widehat{\bm{\theta}}_{t-1})
$$
\end{lemma}
A formal argument is given in Appendix \ref{app_sec: unknown budget}.
The same equivalence clearly holds for the clairvoyant optimum computed with the true parameters $\bm{\theta}=(\bm{p},\bm{\lambda})$: $\bm{x}^{*}_{\texttt{MG-UCB}}(B,\bm{\theta})=\bm{x}^{*}_{\texttt{RA-UCB}}(B,\bm{\theta})$.
Moreover, \texttt{MG-UCB} does not require the budget to be known ex ante since its within-round allocation is computed online via water-filling and can be executed at any realized budget level. Together with Lemma \ref{lemma: equivalence RA-UCB and MG-UCB}, this yields a pointwise equivalence: in each round, conditional on the past, \texttt{MG-UCB} produces the same allocation as \texttt{RA-UCB} would if run with the realized budget as input. Hence, under the conditions on the budgets stated below, the regret guarantees proved for \texttt{RA-UCB} transfer to \texttt{MG-UCB} in the unknown-budget, switching setting.
The only subtlety is that the confidence radii in \texttt{RA-UCB} are written in terms of the (known) budget $B$, whereas here the budget $B_t$ is unknown and varies with $t$. Under the uniform budget bound, this causes no additional difficulty: since $B_t \leq B_{\max}$, all concentration bounds (and hence the resulting regret guarantees) remain valid by replacing $B$ with the uniform upper bound $B_{\max}$ wherever $B$ enters as a budget. The truncation level must now be valid uniformly over the realized budgets: we require Assumption~\ref{ass: truncation level} to hold with the same $b$ for the optimal allocation at every budget level $B_t$, i.e., $x_i^*(B_t)\geq b$ for every $i\in[K]$. This in particular requires $B_t\geq Kb$; the estimator of $\lambda_i$ is unchanged, since the law \eqref{eq: truncated law} does not depend on the allocation or on the budget. Since $x_{t,i}$ is only known once $B_t$ is realized, we assume that $B_t$ is independent of the thresholds and activations of round $t$ (e.g., an oblivious sequence), so that $\{x_{t,i}\geq b\}$ is measurable with respect to $\sigma(\mathcal{F}_{t-1},B_t)$ and the concentration argument applies conditionally on $B_t$.
\begin{remark}[Implementability]
\texttt{MG-UCB} is described in continuous time and assumes infinitesimal allocation increments, which is not implementable in practice. We therefore introduce a discretized variant, \texttt{MG-UCB-$\Delta$} (Algorithm~\ref{algo: MG_UCB_discretized pseudocode}), which allocates the budget in steps of size $\Delta>0$. The regret guarantees above are stated for \texttt{MG-UCB} only. For fixed estimates, the allocation of \texttt{MG-UCB-$\Delta$} is optimal for the surrogate objective up to an additive $\mathcal{O}(K\Delta)$ error, which heuristically suggests an additional regret of order $\Delta T$, negligible for $\Delta\lesssim T^{-1/2}$. We do not formalize this argument; see Appendix~\ref{app: discretization heuristic} for a discussion.
\end{remark}

\section{Experiments}
\label{sec:expe}

We evaluate \texttt{RA-UCB} on EdNet-KT3~\cite{choi2020ednetlargescalehierarchicaldataset} in a time-limited quiz design task. EdNet-KT3 is a large-scale dataset of real students solving  multiple-choice questions on a learning platform. It is particularly well-suited to our setting as it provides logs with both correctness and fully observed response times, allowing us to estimate arm-specific time-to-answer distributions offline even when the student answers incorrectly. Using real logs, we select 20 questions with approximately Weibull-distributed response times (see figure \ref{fig:response-time}) that will be our arms. Ground-truth parameters are estimated offline, and regret  is measured against an oracle policy with full knowledge of the parameters (see section \ref{sec:realdata} for the detailed environment construction). At each round $t$, a new student arrives with budget $B$, and the learner allocates $x_{t,1},\ldots,x_{t,K}$ across the $K=20$ questions under $\sum_i x_{t,i}\le B$ to maximize the number of correct answers. Once the learner allocates $x_{t,i}$, he receives censored feedback: the response time is observed only when the student succeeds within the allocated time; otherwise only a failure indicator is observed. Results are shown in Figure \ref{fig:image1}.
We also consider an advertising-inspired allocation task calibrated from Criteo logs \cite{diemert2017attributionmodelingincreasesefficiency}.
A user $u$ arrives with budget $B_u$ that is the number of times he will visit the platform. The learner must allocate impressions across campaigns under $\sum_i x_{u,i}\le B_u$. We build a semi-synthetic environment by fitting a discrete Weibull model to per-user impression counts for each campaign.
Again, regret is measured against an oracle policy with full knowledge of the parameters; results are reported in Figure~\ref{fig:criteo} (Appendix~\ref{sec:realdata}).
In all experiments, the truncation level $b$ of Assumption~\ref{ass: truncation level} is set to a small, arbitrarily chosen value, without any tuning. More details on the experiments can be found in Appendix~\ref{sec:expe-appendix}.

\begin{figure}[H]
    \centering
    \includegraphics[width=\linewidth]{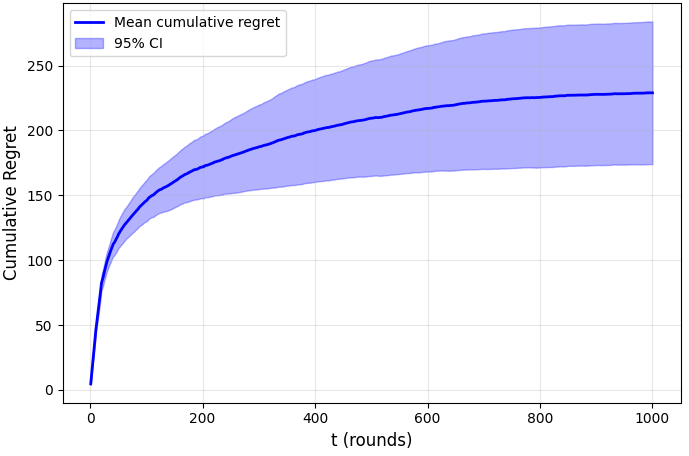}
    \caption{EdNet-KT3 quiz benchmark. Confidence intervals are computed over $5$ independent runs using $5$ batches of $1{,}000$ different users. $B=700s, K=20, T=1000$.}
    \label{fig:image1}
\end{figure}

\section{Conclusion and Future Directions}\label{sec:conclusion-censored}
We study a model of sequential resource allocation with $K$ arms over $T$ rounds, where an arm yields a success only if (i) an arm-specific Bernoulli activation occurs and (ii) the allocated resource exceeds a stochastic threshold. We propose \texttt{RA-UCB}, which decouples reward collection from parameter estimation. When the optimal allocation is interior with a known lower bound $b$, we prove problem-dependent regret bounds of $\widetilde{\mathcal{O}}(\sqrt{T})$ and, under stronger assumptions, $\mathcal{O}(\mathrm{poly}\text{-}\log T)$. For the general model, we establish a worst-case lower bound of $\Omega(\sqrt{T})$. Motivated by online advertising, we also extend the method and the analysis to the \textit{unknown per-round budget} setting, introducing an algorithm that, with infinitesimal allocation increments, inherits the same regret guarantees. A natural next step is to study \textit{unknown and stochastic} budgets without the switching possibility, like in the known budget setting. In this regime, even characterizing the optimal policy is nontrivial, since the allocation order matters and optimal decisions may depend on the full budget distribution. Another interesting extension is a \textit{cascading}-style reward \cite{KVETON}, where the round reward is binary: $1$ if at least one arm succeeds, and $0$ otherwise.

\subsection*{AI use statement}
Generative AI tools were used primarily for text editing, assistance in developing and formalizing mathematical arguments, optimization of simulation code and conducting a proper literature review. They were not used to define the model or design the algorithms. The authors take responsibility for all content, including AI-assisted text, proofs, and code.

\subsection*{Acknowledgments}
This work was supported in part by the French National Research Agency (ANR) through grants ANR-20-CE23-0007 and ANR-23-CE23-0002, and through the PEPR IA FOUNDRY project (ANR23-PEIA-0003). It was also supported by Hi! PARIS and the ANR/France 2030 program (ANR23-IACL-0005).

\bibliography{bibliography}


\clearpage
\onecolumn
\appendix
\begin{center}
{\Large\bfseries Supplementary Materials\par}
\end{center}
\vspace{0.5em}

\section{Experiments} \label{sec:expe-appendix}
\subsection{Simulations and comparison with other algorithms}

\begin{figure}[H]
    \centering
    \begin{minipage}{0.48\linewidth}
        \centering
        \includegraphics[width=\linewidth]{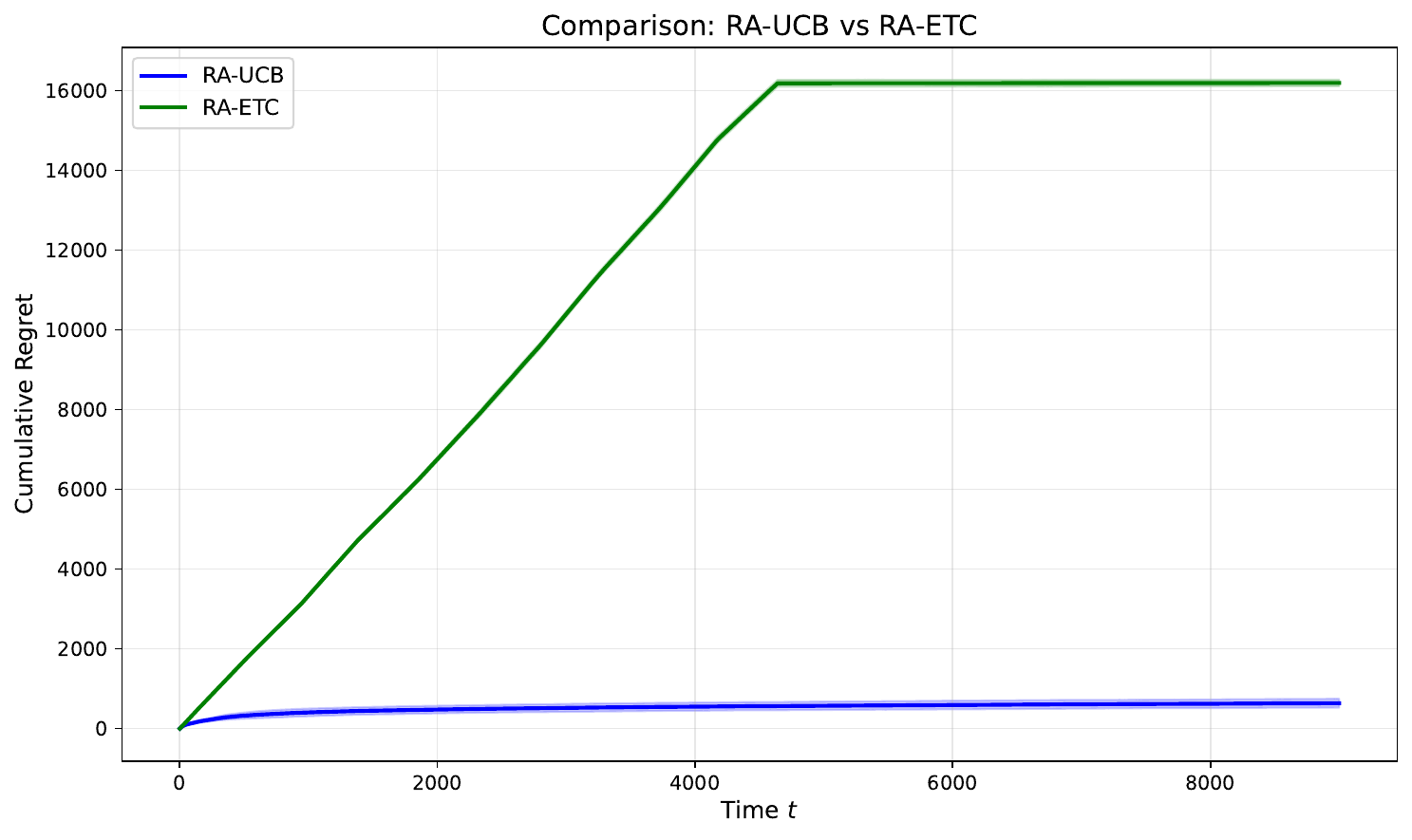}
        \caption{Comparison between \texttt{RA-UCB} (blue) and an Explore-Then-Commit baseline (\texttt{RA-ETC}) for $T=10{,}000$.}
        \label{fig:RA_ETC_simulations}
    \end{minipage}
    \hfill
    \begin{minipage}{0.48\linewidth}
        \centering
        \includegraphics[width=\linewidth]{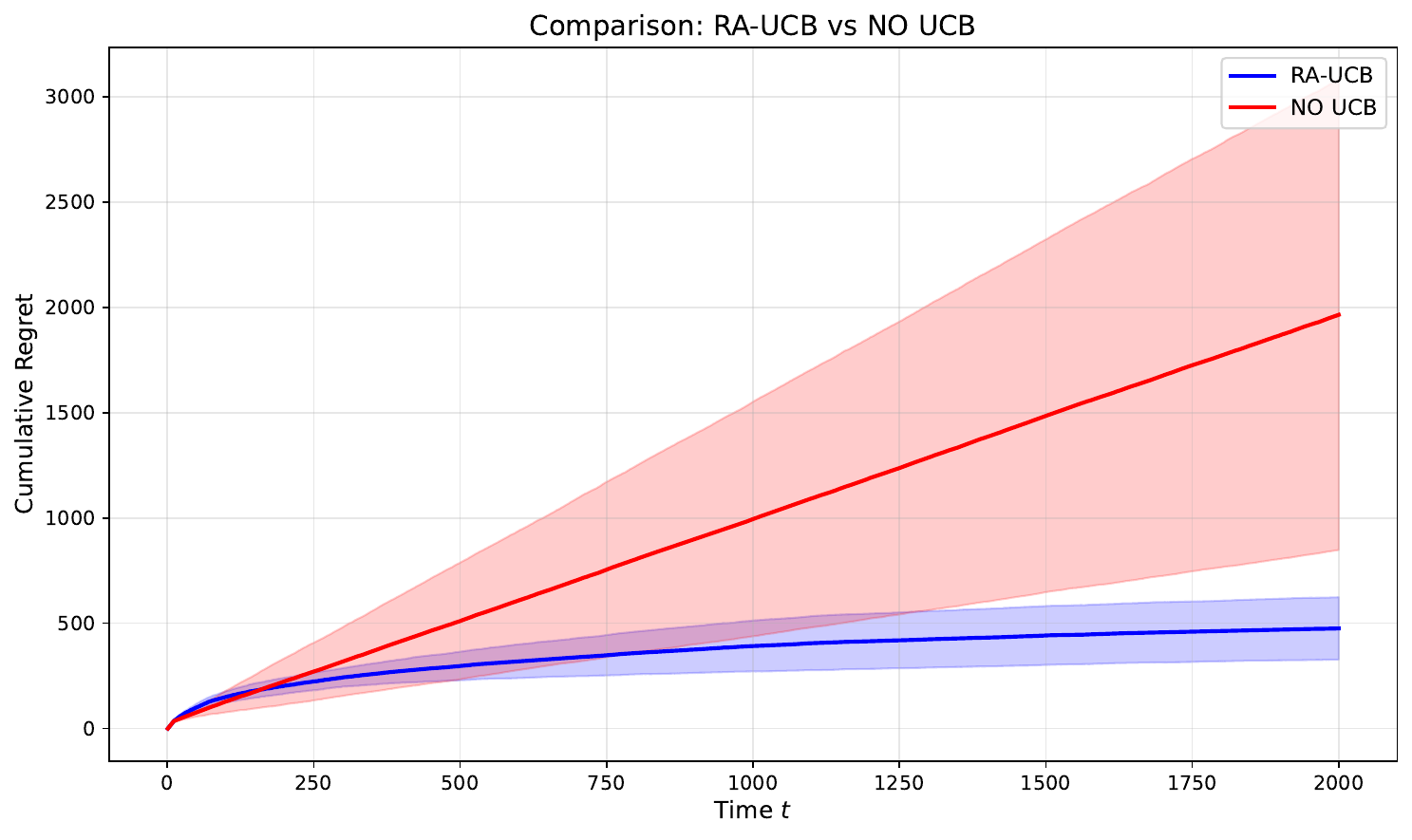}
        \caption{Comparison between \texttt{RA-UCB} (blue) and a naive baseline without confidence bounds (\texttt{NO UCB}).}
        \label{fig:NOUCB}
    \end{minipage}
\end{figure}
In this section, we simulate the behavior of \texttt{RA-UCB} in the exponential case, i.e., when the CDF $G(x,\lambda)$ corresponds to an exponential distribution with rate parameter $\lambda$.

The experiments are conducted with $K=10$ arms and budget $B=40$, where the parameters $\lambda_i$ of the exponential distribution are sampled uniformly in $[1/B,80/B]$ and $p_i\in(0,1]$. Results are averaged over 15 independent runs; the shaded regions represent one standard deviation. The worst-case lower bound in Theorem~\ref{thm: lower bound} is obtained on uniform-threshold instances and is not a pointwise lower bound for these exponential-threshold experiments.

Figure~\ref{fig:RA_ETC_simulations} presents a comparison between \texttt{RA-UCB} and an Explore-Then-Commit baseline (\texttt{RA-ETC}) under the same experimental setup, with $T=10{,}000$. The exploration length of \texttt{RA-ETC} scales as $T^{2/3}$. Curves correspond to averages over 15 runs, with shaded regions indicating standard deviation. While \texttt{RA-ETC} suffers from large regret during its exploration phase, \texttt{RA-UCB} consistently achieves lower regret by balancing exploration and exploitation throughout learning.

Finally, Figure~\ref{fig:NOUCB} compares \texttt{RA-UCB} to a naive baseline, denoted \texttt{NO UCB}, which relies solely on point estimates $(\hat{\lambda}_{t,i},\hat{p}_{t,i})$ without confidence bounds. Using the same instance parameters as in Figure~\ref{fig:RA_ETC_simulations}, over a horizon of $T=2{,}000$ rounds, we observe that ignoring optimism leads to nearly linear regret and substantially higher variance, highlighting the critical role of confidence bounds in achieving efficient learning.

\subsection{Experiments on real-world datasets} \label{sec:realdata}
\paragraph{EdNet environment construction (details).}
We build an offline-to-online benchmark from EdNet-KT3 logs.
We first select $K=20$ question types whose response-time distributions are best explained by a Weibull model, using the Kolmogorov--Smirnov (KS) statistic as a goodness-of-fit criterion.
For each selected question $i$, we retain exactly $5{,}000$ interactions $(\tau,y)$ where $\tau$ is the logged response time and $y\in\{0,1\}$ indicates correctness.
To mitigate logging artifacts (e.g., abandoned sessions), we fit a truncated Weibull model with parameters $(k_i,\lambda_i)$ by maximum likelihood estimation, discarding samples outside $[q_{0.025},q_{0.975}]$ (per question), and setting $p_i$ to the empirical correctness rate on the retained samples.
We then create $5{,}000$ pseudo-students by grouping interactions across the $20$ questions: each pseudo-student corresponds to one interaction directly extracted from the dataset per question, and the session budget is set to $B_t=\sum_{i=1}^K \tau_{t,i}$.
At each round $t$, the learner observes $B_t$ and chooses an allocation $x_{t,1:K}$ with $\sum_i x_{t,i}\le B_t$.
The learner receives censored feedback: the latent threshold $\tau_{t,i}$ is observed only upon success (i.e., when the student answers correctly within the allocated time), and otherwise remains unobserved (right-censored). The reward is defined as $r_{t,i}=\mathds{1}\{x_{t,i}\ge \tau_{t,i}\ \wedge\ y_{t,i}=1\}.$
Since the oracle allocation under Weibull rewards has no closed form, we compute the oracle numerically via constrained optimization.
The Weibull models are parametrized by their rate $\lambda_i=1/s_i$, where $s_i$ is the fitted scale (reported as $\lambda$ in the legend of Figure~\ref{fig:response-time}), so that $G(x,\lambda)$ is non-decreasing in $\lambda$ as required in Section~\ref{sec: known budget}. The fitted shape parameters may exceed $1$ (e.g., $k_i\approx2.3$ for the question of Figure~\ref{fig:response-time}), in which case $G(\cdot,\lambda_i)$ is S-shaped rather than concave near $0$. This benchmark therefore goes beyond the concave setting covered by our guarantees; both the oracle and the optimization step of \texttt{RA-UCB} are solved numerically.

\begin{table}[H]
\centering
\small
\setlength{\tabcolsep}{4pt}
\renewcommand{\arraystretch}{1.15}
\begin{tabularx}{\linewidth}{@{}lll>{\raggedright\arraybackslash}X@{}}
\toprule
\textbf{Column} & \textbf{Type} & \textbf{Example} & \textbf{Description} \\
\midrule
\texttt{pseudo\_user\_id} & string &
\texttt{pseudo\_u1} &
Identifier of a pseudo-student/session (one row per question attempt). \\

\texttt{question\_id} & string &
\texttt{q4135} &
Question identifier (arm). Each pseudo-user answers each selected question once. \\

\texttt{response\_time\_ms} & float &
\texttt{17632.0} &
Logged response time (in milliseconds), interpreted as the latent threshold
$\tau_{u,i}$. \\

\texttt{is\_correct} & binary &
\texttt{0}/\texttt{1} &
Correctness indicator $y_{u,i}\in\{0,1\}$ used to define the reward. \\
\bottomrule
\end{tabularx}
\caption{Structure of the EdNet-derived pseudo-user dataset. Each row corresponds
to one question attempt by a pseudo-user.}
\label{tab:pseudo_user_dataset}
\end{table}

\begin{figure}[H]
    \centering
    \includegraphics[width=1\linewidth]{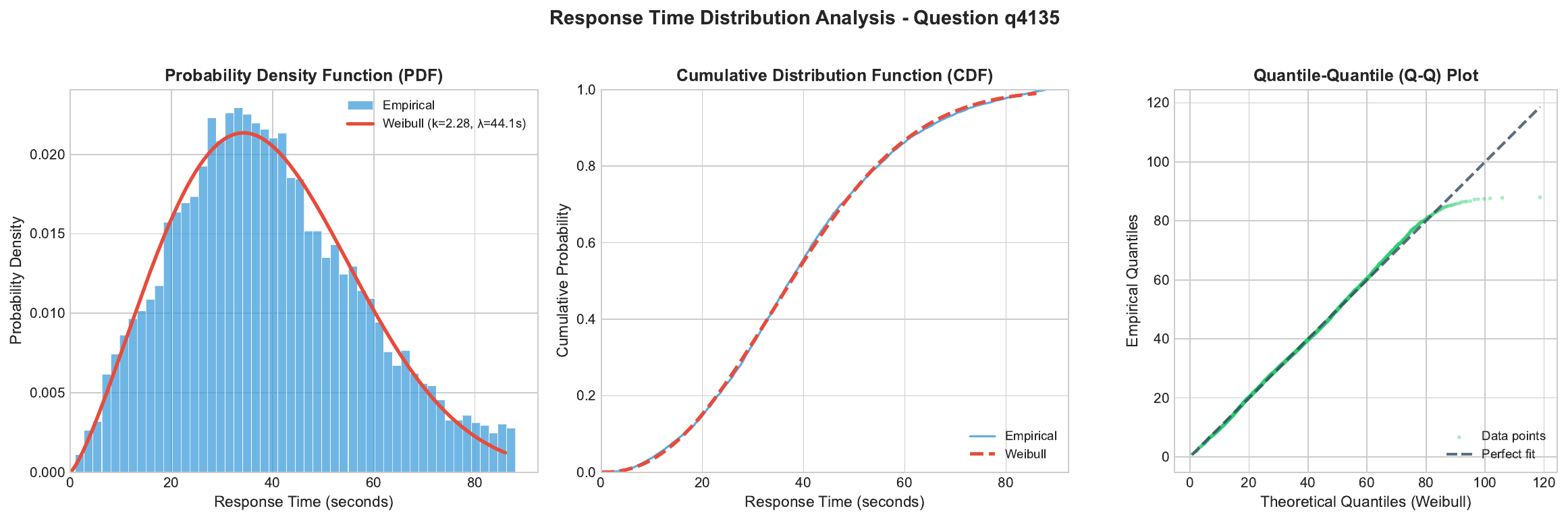}
    \caption{Response-time distribution analysis for a representative EdNet-KT3 question (q4135). Empirical PDF (left), CDF (middle), and Q–Q plot (right) with fitted truncated Weibull model. Response times are fully observed, enabling direct validation of the Weibull threshold assumption.}
    \label{fig:response-time}
\end{figure}

In the EdNet-KT3 dataset, response times provide direct observations of the latent time thresholds associated with each question–student interaction. We assume that students spend the time required to answer a question regardless of whether their answer is correct, so that the response time $\tau_{t,i}$ is fully known offline. This allows us to directly model and validate the distribution of the thresholds without censoring. As shown in Figure \ref{fig:response-time}, the empirical response-time distribution for individual questions is well captured by a truncated Weibull model, as evidenced by the close agreement between empirical and fitted probability density functions, cumulative distribution functions, and quantile–quantile plots. This motivates the use of Weibull threshold models in our offline-to-online benchmark. The experimental results are provided in Figure~\ref{fig:image1}

\paragraph{Semi-synthetic advertising environment calibrated from Criteo logs (details).}

We complement our EdNet experiments with an advertising-inspired benchmark calibrated from Criteo display logs.
The decision problem is an impression allocation task: for each arriving user $u$, the learner observes an opportunity budget $B_u$--defined as the number of ad opportunities generated by the user session (i.e., the number of times the platform observes the user, treated as exogenous)--and allocates impressions $x_{u,1},\dots,x_{u,K}$ across $K$ campaigns under the constraint $\sum_i x_{u,i}\le B_u$.

Figure~\ref{fig:criteo_distrib} shows the empirical distribution of the number of impressions observed before the first click for three representative campaigns in the Criteo logs. Crucially, this distribution can only be computed conditional on a click occurring, since the (latent) exposure threshold is never observed for non-clicked users.  Yet, for both estimation and decision-making, the only relevant and observable quantity is the conditional law $X_{t,i}\mid Y_{t,i}=1$. We find that this conditional distribution is consistently well approximated by a discrete Weibull model across campaigns (Weibull models are also commonly used for user response times in display advertising \cite{SAFARIweibulldistribution,JImixtureweibulldistribution}), with campaign-dependent shape parameters capturing heterogeneous click dynamics. This empirical evidence motivates the use of a Weibull family as a model for display-to-click conditional on activation, without implying any specific generative process for exposure.

We define latent threshold $X_{u,i}$ as the minimum number of impressions required for user $u$ to click on campaign $i$, and model clickability through an independent activation variable $Y_{u,i}$.
We assume campaign-wise independence given the user budget, i.e., $(X_{u,i},Y_{u,i})_{i=1}^K$ are independent across campaigns.
Finally, we evaluate online allocation policies by reporting regret with respect to an oracle that knows all calibrated ground-truth parameters and solves the resulting allocation problem numerically.
\begin{table}[htbp]
\centering
\caption{Example of logged user--campaign interactions. Each row corresponds to a user $u$ exposed to a campaign $i$, with the observed number of exposures before the first click (if any), and a binary click indicator.}
\label{tab:criteo_example}
\begin{tabular}{rrrr}
\toprule
\textbf{uid} & \textbf{campaign} & \textbf{exposures\_before\_click} & \textbf{clicked} \\
\midrule
150 & 10341182 & 1 & 0 \\
150 & 497593   & 3 & 1 \\
309 & 17686799 & 0 & 1 \\
321 & 17686799 & 1 & 0 \\
\bottomrule
\end{tabular}
\end{table}

\begin{figure}[htbp]
    \centering
    \includegraphics[width=1\linewidth]{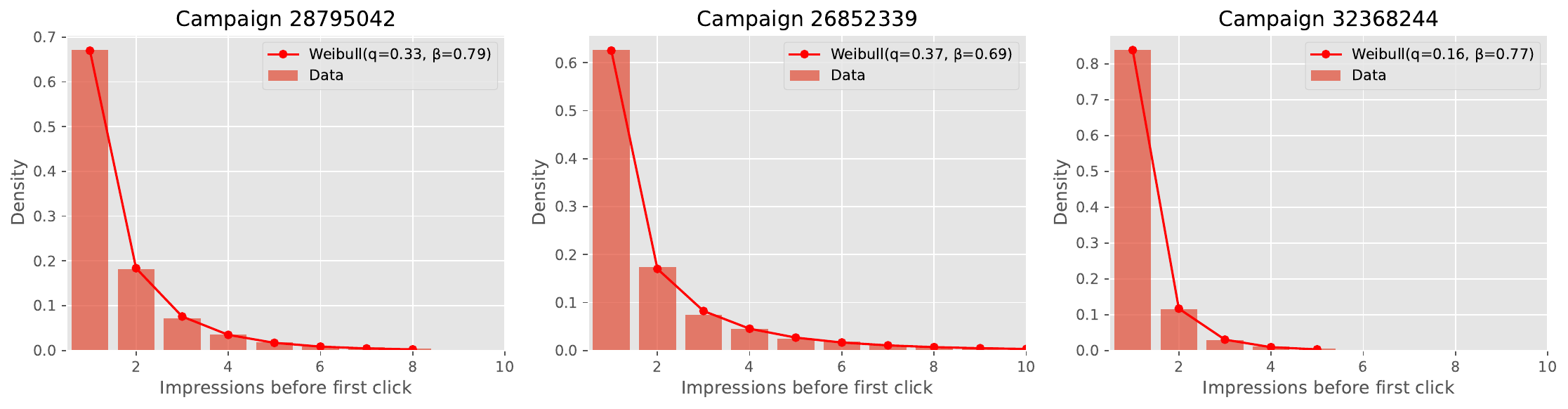}
    \caption{Empirical distribution of the number of impressions observed before the first click for three representative campaigns in the Criteo dataset. Histograms are computed given a click event; solid lines show fitted discrete Weibull models. The discrete Weibull assumption models $X_{t,i}\mid Y_{t,i}=1$ only and does not characterize the underlying exposure policy.}
    \label{fig:criteo_distrib}
\end{figure}

\begin{figure}[H]
    \centering
    \includegraphics[width=0.7\linewidth]{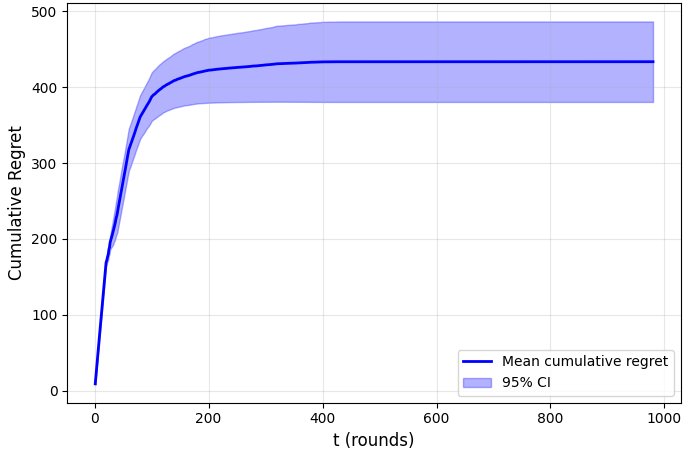}
    \caption{Criteo logs benchmark: regret of \texttt{RA-UCB} with user-dependent budgets $B_u$, observed at the arrival of each user. Confidence intervals are computed over $20$ independent runs using $5$ batches of $1{,}000$ different users. $K=15$, $T=1000$.}
    \label{fig:criteo}
\end{figure}

\section{Problem-Independent Lower Bound}\label{sec:app_lb}

We prove Theorem~\ref{thm: lower bound} using a fixed family of uniform-threshold instances. Throughout this appendix, logarithms are natural, $D_{\mathrm{KL}}$ denotes Kullback--Leibler divergence, and $\|P-Q\|_{\mathrm{TV}}=\sup_A|P(A)-Q(A)|$ denotes the \textit{total variation distance}.

\paragraph{Compatibility with the parametric model.}
For any fixed $B>0$, let $\Lambda=[1/B,5/(4B)]$ and define
\begin{equation}\label{eq:lb-parametric-family}
G(x,\lambda)=
\begin{cases}
0, & x<0,\\
x/B+(B\lambda-1)x/B(1-x/B), & 0\leq x\leq B,\\
1, & x>B.
\end{cases}
\end{equation}
On $[0,B]$, its density is
$$
\partial_xG(x,\lambda)
=\frac{1}{B}\bigl[1+(B\lambda-1)(1-2x/B)\bigr]
\geq\frac{3}{4B}.
$$
Thus $G(\cdot,\lambda)$ is a valid CDF. Moreover,
$$
0\leq\partial_\lambda G(x,\lambda)
=B(x/B)(1-x/B)\leq B/4,
$$
so the family satisfies monotonicity and Lipschitz continuity in $\lambda$, with fixed $m=1$, $M=5/4$, and $L_\lambda=B/4$. At $\lambda=1/B$, it is the uniform distribution on $[0,B]$. The family also has an identifiable conditional mean:
$$
\mu(\lambda)=\mathbb{E}_{\lambda}[X\mid X\leq B] =\frac{B(4-B\lambda)}{6}, \qquad \mu^{-1}(v)=\frac{4}{B}-\frac{6v}{B^2}.
$$
The lower bound holds for the general model of Section~\ref{sec: known budget} and does not rely on the interior-optimum condition of Assumption~\ref{ass: truncation level}. For the truncation level $b=B$, the truncated mean $\mathbb{E}_\lambda[X\mid X\leq B]$ is the displayed $\mu(\lambda)$, which satisfies the invertibility condition of Assumption~\ref{ass: lipschitz of mu-1} with $L_\mu=B^2/6$. Neither this family nor its parameter domain depends on $T$. The lower bound below uses only $\lambda_i=1/B$.

\begin{proof}[Proof of Theorem~\ref{thm: lower bound}]
Fix an arbitrary possibly randomized, non-anticipating algorithm. In every instance below, the thresholds are independent and uniform on $[0,B]$, and all activation variables are independent across arms and rounds and independent of the thresholds. We may reveal the threshold distributions to the learner. Write $a_{t,i}=x_{t,i}/B$, so that $a_{t,i}\geq0$ and $\sum_i a_{t,i}\leq1$.

\paragraph{Information in one observation.}
For an arm with activation probability $p\in(0,1)$ receiving allocation $x=aB$, denote its feedback law by $Q_p^a$. The feedback is the threshold value upon success and $+\infty$ otherwise. Its law is
\begin{equation}\label{eq:lb-feedback-law}
Q_p^a = pa\,\mathrm{Unif}[0,aB] + (1-pa)\,\delta_{+\infty}.
\end{equation}
For $a>0$, success has probability $pa$, and the threshold conditional on success is uniform on $[0,aB]$, independently of $p$. Hence, for $p,q\in(0,1)$,
\begin{align}\label{eq:lb-feedback-kl}
D_{\mathrm{KL}}(Q_p^a\|Q_q^a) &=pa\log\frac{p}{q}+(1-pa)\log\frac{1-pa}{1-qa}\notag\\
&=\operatorname{kl}(pa,qa),
\end{align}
where $\operatorname{kl}$ is the binary relative entropy. If $a=0$, both laws are point masses at $+\infty$ and the divergence is zero. In particular, the observed threshold values are fully accounted for by~\eqref{eq:lb-feedback-kl}.

For $0<\varepsilon\leq1/4$ and $a\in(0,1]$, the inequality $\operatorname{kl}(u,v)\leq (u-v)^2/[v(1-v)]$ gives
\begin{align}\label{eq:lb-one-step-bound}
D_{\mathrm{KL}}(Q_{1/2}^a\|Q_{1/2+\varepsilon}^a) &\leq\frac{\varepsilon^2a^2} {(1/2+\varepsilon)a\,[1-(1/2+\varepsilon)a]}\notag\\
&\leq8\varepsilon^2a,
\end{align}
because $(1/2+\varepsilon)a\geq a/2$ and $1-(1/2+\varepsilon)a\geq1/4$. The same bound holds at $a=0$.

\paragraph{Reference and alternative instances.}
Let instance $0$ have $p_i=1/2$ for every arm. For each $j\in[K]$, let instance $j$ have $p_j=1/2+\varepsilon$ and $p_i=1/2$ for $i\neq j$. Let $\mathbb{P}_j$ be the law of the complete interaction history, including allocations and feedback, and let $\mathbb{E}_j$ denote its expectation. Define
$$
A_j=\sum_{t=1}^T a_{t,j}.
$$
The budget constraint implies $\sum_j\mathbb{E}_0[A_j]\leq T$. Therefore there exists $j^*\in[K]$ such that
\begin{equation}\label{eq:lb-budget-pigeonhole}
\mathbb{E}_0[A_{j^*}]\leq T/K.
\end{equation}

\paragraph{Change of measure.}
Conditional on the preceding history, the algorithm uses the same action-selection kernel in every instance. Given the selected allocation, the feedback coordinates are independent and only coordinate $j^*$ differs between instances $0$ and $j^*$. The chain rule for relative entropy and~\eqref{eq:lb-one-step-bound} thus yield
\begin{align}\label{eq:lb-history-kl}
D_{\mathrm{KL}}(\mathbb{P}_0\|\mathbb{P}_{j^*}) &=\sum_{t=1}^T\mathbb{E}_0\left[ D_{\mathrm{KL}}\bigl(Q_{1/2}^{a_{t,j^*}}\|Q_{1/2+\varepsilon}^{a_{t,j^*}}\bigr) \right]\notag\\
&\leq8\varepsilon^2\mathbb{E}_0[A_{j^*}] \leq8\varepsilon^2T/K.
\end{align}
This argument allows adaptive allocations and randomized action-selection kernels. By Pinsker's inequality,
\begin{equation}\label{eq:lb-tv}
\|\mathbb{P}_0-\mathbb{P}_{j^*}\|_{\mathrm{TV}} \leq2\varepsilon\sqrt{T/K}.
\end{equation}
Set $\varepsilon=\frac18\sqrt{K/T}$. Since $K\leq T$, we have $\varepsilon\leq1/8$, and~\eqref{eq:lb-tv} is at most $1/4$. 

We now transfer the bound on the expected allocation from the reference instance to the alternative instance. The quantity $A_{j^*}$ is the same measurable function of the interaction history in both instances: it sums the normalized allocations to arm $j^*$. Its distribution, however, depends on the instance.

For any random variable $Z$ taking values in $[0,T]$, the tail-integral identity gives
$$
\mathbb{E}_{j^*}[Z]-\mathbb{E}_0[Z] = \int_0^T \left( \mathbb{P}_{j^*}(Z>u)-\mathbb{P}_0(Z>u) \right)\,du.
$$
By the definition of total variation distance, the absolute
difference between the probabilities of any event is at most
$\|\mathbb{P}_0-\mathbb{P}_{j^*}\|_{\mathrm{TV}}$.
Consequently,
$$
\left|\mathbb{E}_{j^*}[Z]-\mathbb{E}_0[Z]\right| \leq T\|\mathbb{P}_0-\mathbb{P}_{j^*}\|_{\mathrm{TV}}.
$$

Applying this inequality to $Z=A_{j^*}$, which lies in $[0,T]$ by the budget constraint, yields
\begin{align}\label{eq:lb-transfer-allocation}
\mathbb{E}_{j^*}[A_{j^*}] &\leq \mathbb{E}_0[A_{j^*}] +T\|\mathbb{P}_0-\mathbb{P}_{j^*}\|_{\mathrm{TV}} \notag\\
&\leq \frac{T}{K}+\frac{T}{4},
\end{align}
where the last inequality follows from \eqref{eq:lb-budget-pigeonhole} and the total variation bound established above. Thus, even when arm $j^*$ is the best arm, its expected cumulative allocation cannot exceed its allocation in the reference instance by more than $T/4$.

\paragraph{Regret in the alternative instance.}
In instance $j^*$, the thresholds are uniform on $[0,B]$, so allocating $x_{t,i}=Ba_{t,i}$ to arm $i$ yields expected reward $p_i a_{t,i}$. Since $p_{j^*}=1/2+\varepsilon$ and $p_i=1/2$ for $i\neq j^*$, the expected reward of any feasible normalized allocation $\bm a$ is
$$
\sum_{i=1}^K p_i a_i = \frac12\sum_{i=1}^K a_i+\varepsilon a_{j^*} \leq \frac12+\varepsilon,
$$
where we used $\sum_i a_i\leq1$ and $a_{j^*}\leq1$. Equality is attained by setting $a_{j^*}=1$ and $a_i=0$ for all $i\neq j^*$. Thus, allocating the entire budget to arm $j^*$ is optimal, with expected reward $1/2+\varepsilon$ per round.

Subtracting the expected reward of the learner's allocation from this optimal value gives
$$
\left(\frac12+\varepsilon\right) -\left(\frac12\sum_{i=1}^K a_{t,i} +\varepsilon a_{t,j^*}\right) = \frac12\left(1-\sum_{i=1}^K a_{t,i}\right) +\varepsilon(1-a_{t,j^*}).
$$
The first term is nonnegative by the budget constraint.
Consequently,
\begin{align*}
\mathbb{E}_{j^*}[\mathcal{R}_T] &=\mathbb{E}_{j^*}\sum_{t=1}^T \left[\frac12\left(1-\sum_i a_{t,i}\right) +\varepsilon(1-a_{t,j^*})\right]\\
&\geq\varepsilon\left(T-\mathbb{E}_{j^*}[A_{j^*}]\right)\\
&\geq\varepsilon T\left(1-\frac1K-\frac14\right)\\
&\geq\frac{\varepsilon T}{4} =\frac1{32}\sqrt{KT},
\end{align*}
where the last inequality uses $K\geq2$. All activation probabilities lie in $[1/2,3/4]$, completing the proof.
\end{proof}

\subsection{An Additional Lower Bound for Exponential Thresholds}
\label{subsec:exponential-lower-bound}

As an additional result, we establish a regret lower bound of $\Omega(K^{2/3}T^{1/3})$ when the threshold distributions are restricted to the exponential family. Although weaker than the lower bound for the general model, this result shows that a polynomial dependence on the horizon remains unavoidable even within this specific family. The construction uses unit budget, deterministic activations, and positive exponential rates that
may approach zero as the horizon increases. Accordingly, the result concerns the unrestricted exponential family, rather than a parameter domain bounded away from zero uniformly in $T$.

\begin{theorem}
\label{thm:exponential-lower-bound}
Let $2\leq K\leq T$ and $B=1$. For every possibly randomized, non-anticipating allocation algorithm, there exists an instance with independent exponential thresholds and activation probabilities $p_i=1$ for all $i\in[K]$ such that
$$
    \mathbb{E}[\mathcal{R}_T] \geq \frac{1}{96}K^{2/3}T^{1/3}.
$$
\end{theorem}

The proof compares a reference instance with identical arms to alternatives in which one arm has a slightly larger exponential rate. The rates are chosen so that allocating the entire budget to the improved arm is optimal. However, identifying this arm requires observations whose information content is controlled by the budget allocated to it.

\paragraph{Information in a censored exponential observation.}

For $\lambda>0$ and $x\geq0$, let $Q_\lambda^x$ denote the law of $\min\{Y,x\}$, where $Y\sim\mathrm{Exp}(\lambda)$. For a fixed allocation $x$, this observation is equivalent to the feedback in our model when $p=1$: values below $x$ reveal the threshold, while the atom at $x$ represents failure. The event $Y=x$ has probability zero.

\begin{lemma}
\label{lem:exponential-censored-kl}
For $\lambda>0$, $\varepsilon>0$, and $x\geq0$,
\begin{align}
D_{\mathrm{KL}} \bigl(Q_\lambda^x\|Q_{\lambda+\varepsilon}^x\bigr) &= (1-e^{-\lambda x}) \left[ \frac{\varepsilon}{\lambda} -\log\left(1+\frac{\varepsilon}{\lambda}\right) \right]
\label{eq:exponential-kl-exact}\\
&\leq \frac{\varepsilon^2}{2\lambda}\,x.
\label{eq:exponential-kl-bound}
\end{align}
\end{lemma}

\begin{proof}
For $x>0$, the law of the censored observation is
$$
Q_\lambda^x(dz) = \lambda e^{-\lambda z}\mathds{1}_{[0,x)}(z)\,dz + e^{-\lambda x}\delta_x(dz).
$$
Consequently,
\begin{align*}
D_{\mathrm{KL}} \bigl(Q_\lambda^x\|Q_{\lambda+\varepsilon}^x\bigr) &= \int_0^x \left[ \log\frac{\lambda}{\lambda+\varepsilon} +\varepsilon z \right]\lambda e^{-\lambda z}\,dz +\varepsilon x e^{-\lambda x}\\
&= (1-e^{-\lambda x}) \log\frac{\lambda}{\lambda+\varepsilon} +\frac{\varepsilon}{\lambda}(1-e^{-\lambda x}),
\end{align*}
where the second equality follows from integration by parts. This proves \eqref{eq:exponential-kl-exact}.

Using
$$
u-\log(1+u)\leq \frac{u^2}{2}, \qquad u\geq0,
$$
and $1-e^{-\lambda x}\leq\lambda x$, we obtain
$$
D_{\mathrm{KL}} \bigl(Q_\lambda^x\|Q_{\lambda+\varepsilon}^x\bigr) \leq (1-e^{-\lambda x})\frac{\varepsilon^2}{2\lambda^2} \leq \frac{\varepsilon^2}{2\lambda}\,x.
$$
For $x=0$, both laws are point masses and the divergence is zero.
\end{proof}

\begin{proof}[Proof of Theorem~\ref{thm:exponential-lower-bound}]
Fix an arbitrary possibly randomized, non-anticipating algorithm. Set
$$
\varepsilon=\frac{1}{12}\left(\frac{K}{T}\right)^{2/3}, \qquad \lambda=\sqrt{\frac{\varepsilon}{2}}-\varepsilon.
$$
Since $K\leq T$, we have $0<\varepsilon\leq1/12<1/8$. In particular, $\lambda\geq\varepsilon>0$ and
$$
(\lambda+\varepsilon)^2=\frac{\varepsilon}{2}.
$$

\paragraph{Step 1: Reference and alternative instances.}

In instance $0$, all arms have exponential rate $\lambda$. For each $j\in[K]$, instance $j$ differs only in arm $j$, whose rate is $\lambda+\varepsilon$. In every instance, activation probabilities equal $1$, and
thresholds are independent across arms and rounds.

Let $\mathbb{P}_j$ denote the law of the complete interaction history, including allocations and feedback, in instance $j$, and let $\mathbb{E}_j$ denote the corresponding expectation.
Define
$$
A_i=\sum_{t=1}^T x_{t,i}, \qquad i^*\in\operatorname*{arg\,min}_{i\in[K]} \mathbb{E}_0[A_i].
$$
The budget constraint gives $\sum_i A_i\leq T$, hence
\begin{equation}
\label{eq:exponential-reference-allocation}
\mathbb{E}_0[A_{i^*}]\leq \frac{T}{K}.
\end{equation}

\paragraph{Step 2: Comparing the history distributions.}

Conditional on the past, the algorithm uses the same action-selection kernel in every instance. Given the selected allocation, feedback coordinates are independent, and only coordinate $i^*$ differs between instances $0$ and $i^*$. The chain rule for relative entropy and Lemma~\ref{lem:exponential-censored-kl} therefore yield
\begin{align*}
D_{\mathrm{KL}}(\mathbb{P}_0\|\mathbb{P}_{i^*}) &= \sum_{t=1}^T \mathbb{E}_0\left[ D_{\mathrm{KL}} \bigl( Q_\lambda^{x_{t,i^*}} \|Q_{\lambda+\varepsilon}^{x_{t,i^*}} \bigr) \right]\\
&\leq \frac{\varepsilon^2}{2\lambda} \mathbb{E}_0[A_{i^*}].
\end{align*}
Since $\varepsilon\leq\lambda$, we have
$$
2\lambda\geq\lambda+\varepsilon =\sqrt{\frac{\varepsilon}{2}},
$$
and thus
$$
D_{\mathrm{KL}}(\mathbb{P}_0\|\mathbb{P}_{i^*}) \leq \sqrt{2}\,\varepsilon^{3/2}\frac{T}{K} = \frac{\sqrt{2}}{12^{3/2}} \leq \frac18.
$$
Pinsker's inequality implies
$$
\|\mathbb{P}_0-\mathbb{P}_{i^*}\|_{\mathrm{TV}} \leq \sqrt{\frac12 D_{\mathrm{KL}}(\mathbb{P}_0\|\mathbb{P}_{i^*})} \leq \frac14.
$$

Because $A_{i^*}$ is a measurable function of the history taking values in $[0,T]$, its expectations satisfy
\begin{align}
\mathbb{E}_{i^*}[A_{i^*}] &\leq \mathbb{E}_0[A_{i^*}] + T\|\mathbb{P}_0-\mathbb{P}_{i^*}\|_{\mathrm{TV}} \notag\\
&\leq \frac{T}{K}+\frac{T}{4}.
\label{eq:exponential-alternative-allocation}
\end{align}

\paragraph{Step 3: Translating allocation into regret.}

In instance $i^*$, the expected reward of a feasible allocation $\bm{x}$ is
$$
f_{i^*}(\bm{x}) = \sum_{i\neq i^*}(1-e^{-\lambda x_i}) + 1-e^{-(\lambda+\varepsilon)x_{i^*}}.
$$
Let $\bm{e}_{i^*}$ denote the allocation assigning the entire unit budget to arm $i^*$. Using $1-e^{-\lambda x_i}\leq\lambda x_i$ and $\sum_{i\neq i^*}x_i\leq1-x_{i^*}$, we obtain
\begin{align*}
f_{i^*}(\bm{e}_{i^*}) -f_{i^*}(\bm{x}) &\geq \int_{x_{i^*}}^1 \left[ (\lambda+\varepsilon)e^{-(\lambda+\varepsilon)u} -\lambda \right]du\\
&\geq (1-x_{i^*}) \left[ (\lambda+\varepsilon)e^{-(\lambda+\varepsilon)} -\lambda \right]\\
&\geq (1-x_{i^*}) \left[ \varepsilon-(\lambda+\varepsilon)^2 \right]\\
& = \frac{\varepsilon}{2}(1-x_{i^*}),
\end{align*}
where the third inequality uses $e^{-v}\geq1-v$.

In particular, this difference is nonnegative for every feasible allocation, so $\bm{e}_{i^*}$ is optimal. Summing the same inequality over rounds and applying
\eqref{eq:exponential-alternative-allocation} gives
\begin{align*}
\mathbb{E}_{i^*}[\mathcal{R}_T] &\geq \frac{\varepsilon}{2} \left(T-\mathbb{E}_{i^*}[A_{i^*}]\right)\\
&\geq \frac{\varepsilon T}{2} \left(1-\frac1K-\frac14\right)\\
&\geq \frac{\varepsilon T}{8}\\
&= \frac{1}{96}K^{2/3}T^{1/3},
\end{align*}
where the third inequality uses $K\geq2$.
\end{proof}

\section{Concentration Inequalities for the Estimators with Generic Distribution}\label{app_sec: concentration inequalities}

\subsection{Estimation of $\lambda_i$}\label{app: estimation of lambda}

\begin{lemma} \label{lemma: conc ineq for mu_i}
    For any arm $i \in [K]$, round $t \in [T]$, $\ell \in [t]$, and $\alpha \in (0,1)$, we have:
    $$
    \mathbb{P} \left( |\hat{\mu}_{t,i} - \mu_i| \geq b \sqrt{\frac{\log(1/\alpha)}{2n_{t,i}}} \text{ and } n_{t,i} = \ell \right) \leq 2\alpha.
    $$
\end{lemma}

The difficulty is that $n_{t,i}$ is random and that the indicator $W_{u,i}=\mathds{1}\{F_{u,i}\leq b\}$ of an informative observation is revealed only after the outcome of round $u$, so it is not $\mathcal{F}_{u-1}$-measurable. The argument below only uses the predictability of the selection event $\mathcal{E}_{u,i}$, and treats $W_{u,i}$ as part of the outcome.

\begin{lemma}\label{lemma: conc ineq for mu_i complete}
    Let $(\mathcal{F}_t)_{t\geq0}$ be a filtration, let $(E_t)_{t\geq1}$ be events with $E_t\in\mathcal{F}_{t-1}$, and let $(W_t,V_t)_{t\geq1}$ be $\mathcal{F}_t$-adapted random variables with $W_t\in\{0,1\}$ and $V_t=W_tU_t$, where $U_t\in[0,b]$. Assume that there exists $\bar\mu\in[0,b]$ such that, on $E_t$, conditionally on $\mathcal{F}_{t-1}$: $W_t$ is Bernoulli with parameter $q_t\in[0,1]$, and conditionally on $\{W_t=1\}$ the variable $U_t$ has mean $\bar\mu$ and support in $[0,b]$. Define
    $$
    n_t=\sum_{u\leq t}\mathds{1}_{E_u}W_u, \qquad Z_t=\sum_{u\leq t}\mathds{1}_{E_u}\big(V_u-\bar\mu W_u\big),
    $$
    and $\hat\eta_t=Z_t/n_t$ when $n_t\geq1$. Then for any $t,\ell\in\mathbb{N}_*$ and $\alpha\in(0,1)$:
    $$
    \mathbb{P} \left (|\hat{\eta}_{t}| \geq b \sqrt{\frac{\log(1/\alpha)}{2n_{t}}} \text{ and } n_{t}=\ell  \right) \leq 2\alpha
    \qquad\text{and}\qquad
    \mathbb{P} \left (|\hat{\eta}_{t}| \leq b \sqrt{\frac{\log(1/\alpha)}{2n_{t}}} \text{ or } n_{t}=0 \right) \geq 1-2t\alpha .
    $$
\end{lemma}

\begin{proof}
    For $x \in \mathbb{R}$, define
    $$
    M_t = \exp \left( xZ_t-\frac{x^2b^2n_{t}}{8} \right), \qquad M_0=1.
    $$
    We show that $(M_t)$ is a supermartingale. We have $M_t=M_{t-1}\exp(\mathds{1}_{E_t}W_t\,\xi_t)$ with $\xi_t:=x(U_t-\bar\mu)-x^2b^2/8$, because $V_t-\bar\mu W_t=W_t(U_t-\bar\mu)$. On $E_t^c$ the factor equals $1$. On $E_t$, since $W_t\in\{0,1\}$,
    $$
    \mathbb{E}\big[e^{W_t\xi_t}\mid\mathcal{F}_{t-1}\big]
    =(1-q_t)+q_t\,\mathbb{E}\big[e^{x(U_t-\bar\mu)}\mid W_t=1,\mathcal{F}_{t-1}\big]e^{-x^2b^2/8}
    \leq (1-q_t)+q_t=1,
    $$
    where we used the conditional version of Hoeffding's lemma: if $S$ takes values in $[a_1,a_2]$, then
    \begin{equation}\label{thm: conditional verson of Hoeffding lemma}
        \mathbb{E}[e^{\lambda (S - \mathbb{E}[S|\mathcal{F}])}|\mathcal{F}] \leq \exp \left( \frac{\lambda^2}{8} (a_2-a_1)^2 \right),
    \end{equation}
    applied to $U_t$ conditionally on $\{W_t=1\}$ and $\mathcal{F}_{t-1}$, with mean $\bar\mu$ and range $b$. Since $E_t\in\mathcal{F}_{t-1}$ and $M_{t-1}$ is $\mathcal{F}_{t-1}$-measurable, $\mathbb{E}[M_t\mid\mathcal{F}_{t-1}]\leq M_{t-1}$, hence $\mathbb{E}[M_t]\leq1$. Note that the non-predictable indicator $W_t$ appears jointly in $Z_t$ and in the compensator $x^2b^2n_t/8$, so it never needs to be taken outside a conditional expectation.

    Using a Markov--Chernoff argument, for $x>0$,
    $$
    \mathbb{P}(Z_t \geq \epsilon \text{ and } n_{t}= \ell)
    \leq e^{-\epsilon x} \mathbb{E}\left[e^{xZ_t} \mathds{1}\{n_{t}=\ell\}\right]
    = e^{-\epsilon x+\frac{x^2b^2 \ell}{8}} \mathbb{E}\left[M_t \mathds{1}\{n_{t}=\ell\}\right]
    \leq e^{-\epsilon x+\frac{x^2b^2}{8}\ell}.
    $$
    By choosing $\epsilon = b \sqrt{\frac{\ell \log (1/\alpha)}{2}}$ and $x=\frac{4\epsilon}{\ell b^2}$:
    \begin{equation}\label{eq: P(Z_t >...}
        \mathbb{P}\left(Z_t \geq b \sqrt{\frac{\ell \log(1/\alpha)}{2}} \text{ and } n_{t}=\ell\right) \leq \alpha .
    \end{equation}
    The same argument with $x<0$ gives
    \begin{equation}\label{eq: P(-Z_t >...}
        \mathbb{P}\left(Z_t \leq -b \sqrt{\frac{\ell \log(1/\alpha)}{2}} \text{ and } n_{t}=\ell\right) \leq \alpha .
    \end{equation}
    Since $Z_t=\hat\eta_tn_t$, combining the two bounds proves the first claim. Summing the first claim over $\ell\in\{1,\ldots,t\}$ gives $\mathbb{P}\big(|\hat\eta_t|> b\sqrt{\log(1/\alpha)/(2n_t)} \text{ and } n_t\geq1\big)\leq 2t\alpha$, which is the second claim.
\end{proof}

\paragraph{Proof of Lemma \ref{lemma: conc ineq for mu_i}.}
Apply Lemma~\ref{lemma: conc ineq for mu_i complete} with $E_u=\mathcal{E}_{u,i}$, $W_u=W_{u,i}$, $U_u=X_{u,i}$ and $V_u=V_{u,i}$. The event $\mathcal{E}_{u,i}$ is $\mathcal{F}_{u-1}$-measurable. On $\mathcal{E}_{u,i}$ we have $x_{u,i}\geq b$, hence $W_{u,i}=Y_{u,i}\mathds{1}\{X_{u,i}\leq b\}$ and $V_{u,i}=W_{u,i}X_{u,i}$. Because $(X_{u,i},Y_{u,i})$ is independent of $\mathcal{F}_{u-1}$ and $Y_{u,i}$ is independent of $X_{u,i}$, conditionally on $\mathcal{F}_{u-1}$ the indicator $W_{u,i}$ is Bernoulli with parameter $p_iG(b,\lambda_i)$, and conditionally on $\{W_{u,i}=1\}$ the variable $X_{u,i}$ has law \eqref{eq: truncated law}, with mean $\mu(\lambda_i)=\mu_i$ and support in $[0,b]$. Since $Z_t=n_{t,i}(\hat{\mu}_{t,i}-\mu_i)$, the claim follows.

\hfill$\square$

\paragraph{Proof of Lemma \ref{lemma: conc ineq lambda} (concentration of $\lambda_i$)}

By Assumption \ref{ass: lipschitz of mu-1} we know that $\mu^{-1}$ is Lipschitz with constant $1/L_\mu$. Since $\mu_i\in\mu(\Lambda)$ and the projection onto $\mu(\Lambda)$ is $1$-Lipschitz:
$$
|\hat{\lambda}_{t,i}-\lambda_i| \leq \frac{1}{L_\mu}\big|\Pi_{\mu(\Lambda)}(\hat{\mu}_{t,i})-\mu_i\big| \leq \frac{1}{L_\mu}|\hat{\mu}_{t,i}-\mu_i|
$$
Applying the concentration inequality from Lemma~\ref{lemma: conc ineq for mu_i}, we conclude that
$$
\mathbb{P}\left(|\hat{\lambda}_{t,i}-\lambda_i| \geq \frac{b}{L_\mu} \sqrt{\frac{\log(1/\alpha)}{2n_{t,i}}} \text{ and } n_{t,i} = \ell \right)
\leq
\mathbb{P}\left(|\hat{\mu}_{t,i}-\mu_i| \geq b\sqrt{\frac{\log(1/\alpha)}{2n_{t,i}}} \text{ and } n_{t,i} = \ell \right)
\leq 2\alpha.
$$
This concludes the proof.

\hfill$\square$

\subsection{Proof of Lemma \ref{lemma: conc ineq p_i} (concentration of $p_i$)}
We restrict our analysis to the arms $i \in [K]$ such that $x_i^* > 0$ (under Assumption~\ref{ass: truncation level}, these are all arms), so that $C_i^* = G(x_i^*,m/B)>0$. Fix such an arm $i$ and a round $t$. Throughout, the deviation event is taken to be empty when $n_{t,i}=0$ (the radius is then $+\infty$), so we only consider $n_{t,i}\geq1$.

Recall that $\mathcal{C}_i(t)$ is the (deterministic) set of candidate rounds of arm $i$ up to round $t$ and $c_{t,i}=|\mathcal{C}_i(t)|$. Since $n_{t,i}$ only counts candidate rounds, $n_{t,i}\leq c_{t,i}$. Write $r_{u,i}:=G(x_{u,i},\lambda_i)$, so that $\mathbb{E}[f_{u,i}\mid\mathcal{F}_{u-1}]=p_ir_{u,i}$, and define
$$
\tilde{p}_{t,i} = \frac{\sum_{u\in\mathcal{C}_i(t)} \mathbb{E}[f_{u,i}\mid\mathcal{F}_{u-1}]}{\sum_{u\in\mathcal{C}_i(t)} \hat{r}_{u,i}(t)} = p_i \cdot \frac{\sum_{u\in\mathcal{C}_i(t)} r_{u,i}}{\sum_{u\in\mathcal{C}_i(t)} \hat{r}_{u,i}(t)}.
$$
Clipping $\hat p_{t,i}$ to $[0,1]$ can only decrease $|\hat p_{t,i}-p_i|$, so it suffices to prove the bound for the unclipped ratio.

\paragraph{Lower bound on the denominator.}
On $\mathscr{G}_{t-1}$, every candidate round $u\in\mathcal{C}_i(t)$ satisfies $x_{u,i}\geq x_i^*$. Indeed, on an initialization round $x_{u,i}=B\geq x_i^*$; on an estimation round, $u\leq t$ and $\mathscr{G}_{u-1}\supseteq\mathscr{G}_{t-1}$, so Oracle Monotonicity (Assumption~\ref{ass:OMA}) applied on the good event $\mathcal{N}_{u-1}$ gives $x_{u,i}\geq x_i^*$. Since $\hat{\lambda}_{t,i}\in\Lambda$ (by the projection in its definition), $\hat{\lambda}_{t,i}\geq m/B$, and $G$ is non-decreasing in both arguments, we obtain on $\mathscr{G}_{t-1}$
\begin{equation}\label{eq: denominator lower bound p}
\sum_{u\in\mathcal{C}_i(t)} \hat{r}_{u,i}(t) = \sum_{u\in\mathcal{C}_i(t)} G(x_{u,i},\hat{\lambda}_{t,i})\geq c_{t,i}\, G(x_i^*,m/B) = c_{t,i}\, C_i^*.
\end{equation}
Only candidate rounds enter the sums: on the other rounds the allocation to arm $i$ is not controlled and \eqref{eq: denominator lower bound p} would fail.

\paragraph{Splitting.}
For any $\delta \in (0,1)$ and any (possibly random) radius $\eta>0$,
\begin{equation}\label{Eq: splitting proba}
\mathbb{P}(|\hat{p}_{t,i}-p_i| \geq \eta \cap \mathscr{G}_{t-1}) \leq \mathbb{P}(|\hat{p}_{t,i}-\tilde{p}_{t,i}| \geq \delta \eta\cap \mathscr{G}_{t-1}) + \mathbb{P}(|\tilde{p}_{t,i}-p_i| \geq (1-\delta)\eta\cap \mathscr{G}_{t-1}).
\end{equation}
We choose
\begin{equation}\label{Eq: eta}
\delta=\frac{1}{1+p_i},\qquad
\eta = \frac{bL_\lambda p_i}{(1-\delta)C_i^* L_\mu} \sqrt{\frac{\log(1/\alpha)}{2n_{t,i}}}=\frac{L_\lambda}{L_\mu}\frac{b(1+p_i)}{C_i^*}\sqrt{\frac{\log(1/\alpha)}{2n_{t,i}}},
\end{equation}
which is the radius in the statement of Lemma~\ref{lemma: conc ineq p_i}. Note that $\delta p_i/(1-\delta)=1$.

\paragraph{Second term (error due to $\hat\lambda_{t,i}$).}
By \eqref{eq: denominator lower bound p}, on $\mathscr{G}_{t-1}$,
$$
|\tilde{p}_{t,i}-p_i| = \frac{p_i\big|\sum_{u\in\mathcal{C}_i(t)} (r_{u,i}-\hat{r}_{u,i}(t))\big|}{\sum_{u\in\mathcal{C}_i(t)} \hat{r}_{u,i}(t)}
\leq \frac{p_i\, c_{t,i}\, L_\lambda |\hat{\lambda}_{t,i}-\lambda_i|}{c_{t,i}\, C_i^*}
= \frac{p_iL_\lambda}{C_i^*}|\hat{\lambda}_{t,i}-\lambda_i|,
$$
where we used the Lipschitz continuity of $G$ in $\lambda$, $|r_{u,i}-\hat r_{u,i}(t)|\leq L_\lambda|\hat\lambda_{t,i}-\lambda_i|$. Hence, with the choice \eqref{Eq: eta},
\begin{equation}\label{Eq: 2nd term bound}
\begin{aligned}
\mathbb{P} \left( |\tilde{p}_{t,i}-p_i| \geq (1-\delta)\eta \cap \mathscr{G}_{t-1} \right)
&\leq \mathbb{P} \left( |\hat{\lambda}_{t,i}-\lambda_i| \geq \frac{(1-\delta)\eta\, C_i^*}{L_\lambda p_i} \right)
= \mathbb{P} \left( |\hat{\lambda}_{t,i}-\lambda_i| \geq \frac{b}{L_\mu}\sqrt{\frac{\log(1/\alpha)}{2n_{t,i}}} \right) \\
& = \sum_{s=1}^{t} \mathbb{P} \left( |\hat{\lambda}_{t,i}-\lambda_i| \geq \frac{b}{L_\mu}\sqrt{\frac{\log(1/\alpha)}{2n_{t,i}}} \text{ and } n_{t,i}=s \right)
\leq 2\alpha t,
\end{aligned}
\end{equation}
by Lemma~\ref{lemma: conc ineq lambda}.

\paragraph{First term (sampling noise of the successes).}
By \eqref{eq: denominator lower bound p}, on $\mathscr{G}_{t-1}$,
$$
|\hat{p}_{t,i}-\tilde{p}_{t,i}| = \frac{|\zeta_t|}{\sum_{u\in\mathcal{C}_i(t)} \hat{r}_{u,i}(t)}\leq\frac{|\zeta_t|}{c_{t,i}C_i^*},
\qquad
\zeta_t := \sum_{u\in\mathcal{C}_i(t)} \big(f_{u,i}-\mathbb{E}[f_{u,i}\mid\mathcal{F}_{u-1}]\big).
$$
Therefore, using $\delta p_i/(1-\delta)=1$ and then $n_{t,i}\leq c_{t,i}$,
$$
\Big\{|\hat{p}_{t,i}-\tilde{p}_{t,i}| \geq \delta \eta\Big\}\cap \mathscr{G}_{t-1}
\subseteq
\Big\{|\zeta_t| \geq \frac{bL_\lambda}{L_\mu}\,c_{t,i}\sqrt{\frac{\log(1/\alpha)}{2n_{t,i}}}\Big\}
\subseteq
\Big\{|\zeta_t| \geq \frac{bL_\lambda}{L_\mu}\sqrt{\frac{c_{t,i}\log(1/\alpha)}{2}}\Big\},
$$
and the last threshold is deterministic, since $c_{t,i}$ is. Since the set of candidate rounds is deterministic, $\zeta_t=\sum_{u\leq t}\mathds{1}\{u\in\mathcal{C}_i(t)\}\big(f_{u,i}-\mathbb{E}[f_{u,i}\mid\mathcal{F}_{u-1}]\big)$ is a martingale with respect to $(\mathcal{F}_t)$, with increments bounded by $1$ in absolute value, at most $c_{t,i}$ of which are non-zero. The Azuma--Hoeffding inequality gives $\mathbb{P}(|\zeta_t|\geq\epsilon)\leq2\exp(-\epsilon^2/(2c_{t,i}))$, hence
\begin{equation}\label{Eq: 1st term bound}
\mathbb{P} \left( |\hat{p}_{t,i} - \tilde{p}_{t,i}| \geq \delta \eta \cap \mathscr{G}_{t-1} \right)
\leq 2 \exp \left(-\frac{b^2L_\lambda^2}{L_\mu^2}\,\frac{c_{t,i}\log(1/\alpha)}{2}\cdot\frac{1}{2c_{t,i}} \right)
= 2 \alpha^{\frac{b^2L_\lambda^2}{4L_\mu^2}}.
\end{equation}

Plugging \eqref{Eq: 2nd term bound} and \eqref{Eq: 1st term bound} into \eqref{Eq: splitting proba} with the choice \eqref{Eq: eta} proves the lemma.

\hfill$\square$

\clearpage

\section{Algorithm \texttt{RA-UCB}}
Below we present the extended pseudocode of \texttt{RA-UCB},
with explicit parameter estimation steps. The estimators of arm $k$ are updated only at its candidate rounds (Section~\ref{subsec: estimation of lambda}), through the procedure \textsc{Update}; initialization samples are kept and used throughout. In the notation of Section~\ref{sec:algo}, the values $(\hat{\lambda}_k,\hat{p}_k,n_k,c_k)$ held at the beginning of round $t$ are $(\hat{\lambda}_{t-1,k},\hat{p}_{t-1,k},n_{t-1,k},c_{t-1,k})$, and the sums defining $\hat p_k$ run over $\mathcal{C}_k(t-1)$.

\begin{algorithm}[H]
\caption{\texttt{RA-UCB} (Extended Pseudocode)}
\label{algo:RA_UCB_extended}
\footnotesize
\setlength{\abovedisplayskip}{3pt}
\setlength{\belowdisplayskip}{3pt}
\setlength{\abovedisplayshortskip}{2pt}
\setlength{\belowdisplayshortskip}{2pt}

\begin{algorithmic}[1]
\STATE \textbf{Input:} $T,K,B$, truncation level $b$.
\STATE For every $k\in[K]$: $c_k\gets0$ (candidate rounds), $n_k\gets0$ (informative samples), $\hat{\mu}_k\gets0$, and arbitrary $\hat{\lambda}_k\in\Lambda$, $\hat{p}_k\in[0,1]$.
\STATE \textbf{Procedure} \textsc{Update}$(k,x,F)$: \hfill (called exactly once at each candidate round of arm $k$)
\STATE \quad $c_k\gets c_k+1$; store $\tilde{x}_{k,c_k}\gets x$ and $\phi_{k,c_k}\gets\mathds{1}\{F\leq B\}$.
\STATE \quad $\omega\gets\mathds{1}\{x\geq b\}\,\mathds{1}\{F\leq b\}$; $n_k\gets n_k+\omega$; if $\omega=1$, $\hat{\mu}_k\gets\hat{\mu}_k+(F-\hat{\mu}_k)/n_k$.
\STATE \quad If $n_k\geq1$: $\hat{\lambda}_k\gets\mu^{-1}\big(\Pi_{\mu(\Lambda)}(\hat{\mu}_k)\big)$, where $\mu(\lambda)=\mathbb{E}_\lambda[X\mid X\leq b]$.
\STATE \quad $\hat{p}_k\gets\min\Big\{1,\ \dfrac{\sum_{s=1}^{c_k}\phi_{k,s}}{\sum_{s=1}^{c_k}G(\tilde{x}_{k,s},\hat{\lambda}_k)}\Big\}$.
\STATE $t\gets1$.
\STATE \textbf{Initialization Phase} ($t=1,\ldots,K\lfloor\log T\rfloor$):
\FOR{$k\in[K]$}
    \FOR{$s=1,\ldots,\lfloor\log T\rfloor$}
        \STATE Play $\bm{x}_t$ with $x_{t,k}=B$ and $x_{t,j}=0$ for $j\neq k$; observe $\{F_{t,j}\}_{j\in[K]}$.
        \STATE \textsc{Update}$(k,B,F_{t,k})$.
        \STATE $r_t\gets\sum_{j=1}^{K}\mathds{1}\{F_{t,j}\leq B\}$; $t\gets t+1$.
    \ENDFOR
\ENDFOR
\STATE \textbf{Main Phase} ($t=K\lfloor\log T\rfloor+1,\ldots,T$):
\STATE $i\gets1$.
\WHILE{$t\leq T$}
    \STATE Compute confidence bounds $\{\hat{p}_{k}^{\pm},\hat{\lambda}_{k}^{\pm}\}_{k\in[K]}$ from $(\hat{p}_k,\hat{\lambda}_k,n_k)$ as in Section~\ref{subsec:algo-radii}.
    \STATE Boost arm $i$: $(\bar{\lambda}_{i},\bar{\lambda}'_{i},\bar{p}_{i}) \gets (\hat{\lambda}_{i}^{-},\hat{\lambda}_{i}^{+},\hat{p}_{i}^{+})$.
    \STATE Penalize other arms: $(\bar{\lambda}_{k},\bar{\lambda}'_{k},\bar{p}_{k}) \gets (\hat{\lambda}_{k}^{+},\hat{\lambda}_{k}^{-},\hat{p}_{k}^{-})$ for $k\neq i$.
    \STATE $\bm{x}_t\gets \widetilde{\mathcal{P}}(\bar{\bm{\lambda}},\bar{\bm{\lambda}}',\bar{\bm{p}})$; play $\bm{x}_t$ and observe $\{F_{t,k}\}_{k\in[K]}$.
    \STATE \textsc{Update}$(i,x_{t,i},F_{t,i})$. \hfill (round $t$ is the estimation round of arm $i$; other arms are not updated)
    \STATE $r_t\gets\sum_{k=1}^{K}\mathds{1}\{F_{t,k}\leq B\}$; $t\gets t+1$; $i\gets i+1$; if $i>K$, $i\gets1$.
\ENDWHILE
\STATE \textbf{Output:} $\sum_{t=1}^{T}r_t$.
\end{algorithmic}
\end{algorithm}

\subsection{Proof of Lemma \ref{Lemma: structural property of the Opt Problem}} \label{app: proof of structural property of Opt Problem}

By Consistency (Property~\ref{prp: Consistency}), for every $k\in[K]$ we have $\widetilde{G}(x,\lambda_k,\lambda_k)=G(x,\lambda_k)$. Hence the true optimum $\bm{x}=\mathcal{P}(\bm{\lambda},\bm{p})$ is equivalently a maximizer of $\widetilde{\mathcal{P}}(\bm{\lambda},\bm{\lambda},\bm{p})$.

Define the separable objectives over $\mathcal{X}=\{\bm{x}\in\mathbb{R}^K_+:\sum_{k=1}^K x_k\le B\}$:
$$
J(\bm{x})=\sum_{k=1}^K p_k\,\widetilde{G}(x_k,\lambda_k,\lambda_k), \qquad \bar{J}(\bm{x})=\sum_{k=1}^K \bar{p}_k\,\widetilde{G}(x_k,\bar{\lambda}_k,\bar{\lambda}'_k)
$$

\paragraph{KKT under Strong Concavity}
By the strong concavity assumption of $\widetilde{G}(x,\lambda,\lambda')$ we can state that both $J$ and $\bar J$ are strongly concave on $\mathcal{X}$, hence their maximizers are unique and KKT conditions are necessary and sufficient.

Moreover, since each summand is non-decreasing in $x$ and $p_k,\bar p_k\ge 0$,
any maximizer saturates the budget:
$$
\sum_{k=1}^K x_k=\sum_{k=1}^K \bar x_k=B
$$
Let $\mu \geq 0$ and $\bar\mu\geq 0$ be the multipliers of the budget constraint for $J$ and $\bar J$, respectively. Define the marginals
$$
g_k(x):=\frac{\partial J}{\partial x_k}=p_k\,\widetilde g(x,\lambda_k,\lambda_k), \qquad \bar g_k(x):=\frac{\partial \bar J}{\partial x_k} =\bar p_k\,\widetilde g(x,\bar\lambda_k,\bar\lambda'_k)
$$
with $\widetilde g(x,\lambda,\lambda')=\partial_x\widetilde G(x,\lambda,\lambda')$.
Strong concavity implies $\partial_{xx}\widetilde G<0$, hence $\widetilde g(x,\lambda,\lambda')$ is strictly decreasing in $x$ and so are $g_k(\cdot)$ and $\bar g_k(\cdot)$.

The KKT conditions yield, for all $k$,
$$
g_k(x_k)\le \mu,\qquad x_k(\mu-g_k(x_k))=0
$$
and
$$
\bar g_k(\bar x_k)\le \bar\mu,\qquad \bar x_k(\bar\mu-\bar g_k(\bar x_k))=0
$$
In particular, if $x_k>0$ then $g_k(x_k)=\mu$, and if $\bar x_k>0$ then $\bar g_k(\bar x_k)=\bar\mu$.

\paragraph{Dominance Relations from Boosting Monotonicity}
By construction of the two parameter tuples, for the boosted arm $i$:
$$
\bar\lambda_i=\lambda_i^-\le \lambda_i,\qquad \bar\lambda'_i=\lambda_i^+\ge \lambda_i,\qquad \bar p_i=p_i^+\ge p_i
$$
while for every $j\neq i$:
$$
\bar\lambda_j=\lambda_j^+\ge \lambda_j,\qquad \bar\lambda'_j=\lambda_j^-\le \lambda_j,\qquad \bar p_j=p_j^-\le p_j
$$
Using Boosting Monotonicity (Property~\ref{prp: Boosting Monotonicity}), $\widetilde g$ is non-decreasing in $\lambda'$ and non-increasing in $\lambda$, hence for all $x\in[0,B]$,
$$
\widetilde g(x,\bar\lambda_i,\bar\lambda'_i)\ge \widetilde g(x,\lambda_i,\lambda_i), \qquad \widetilde g(x,\bar\lambda_j,\bar\lambda'_j)\le \widetilde g(x,\lambda_j,\lambda_j)
$$
Multiplying by $\bar p_i\ge p_i$ and $\bar p_j\le p_j$ gives the dominance relations
$$
\bar g_i(x)\ge g_i(x)\quad \forall x\in[0,B], \qquad \bar g_j(x)\le g_j(x)\quad \forall x\in[0,B],\ \forall j\neq i
$$

\paragraph{Contradiction Argument}
If $x_i=0$, then $\bar x_i\ge 0=x_i$ and we are done. Assume $x_i>0$. Let us proceed by contradiction: assume $\bar x_i<x_i$. Since both solutions saturate the budget, there exists $j\neq i$ such that $\bar x_j>x_j$ (and in particular $\bar x_j>0$).

Since $\bar g_j$ is strictly decreasing and $\bar x_j>x_j$, $\bar g_j(\bar x_j)<\bar g_j(x_j)$. Therefore
$$
\bar\mu=\bar g_j(\bar x_j) < \bar g_j(x_j)\le g_j(x_j)\le \mu
$$
where the last inequality comes from KKT for $(J,\bm{x})$ (always $g_j(x_j)\le\mu$).

Since $x_i>0$, KKT for $(J,\bm{x})$ gives $g_i(x_i)=\mu$. By dominance relation, $\bar g_i(x_i)\ge g_i(x_i)=\mu$.
Because $\bar g_i$ is strictly decreasing and $\bar x_i<x_i$, we have $\bar g_i(\bar x_i)>\bar g_i(x_i)\ge \mu$.
If $\bar x_i>0$, then KKT for $(\bar J,\bar{\bm{x}})$ gives $\bar\mu=\bar g_i(\bar x_i)>\mu$, contradicting $\bar\mu < \mu$.
If $\bar x_i=0$, then KKT implies $\bar g_i(0)\le \bar\mu$, but strict decrease yields $\bar g_i(0)>\bar g_i(x_i)\ge\mu$, hence $\bar\mu>\mu$ again, the same contradiction.

Then we conclude $\bar x_i\ge x_i$.

\hfill$\square$

\subsection{How to build $\widetilde{G}$}\label{app: how to build G tilde}

We discuss a simple and broadly applicable way to construct a surrogate $\widetilde{G}(x,\lambda,\lambda')$ from a given CDF family $G(x,\lambda)$.
The goal is to obtain a function that coincides with $G$ when $\lambda'=\lambda$ (Consistency), remains a valid CDF-shaped map in $x$ (Boundedness/Monotonicity), and can be tuned so that its marginal $\partial_x\widetilde{G}$ is separately monotone in $\lambda$ and $\lambda'$ (Boosting Monotonicity), which is the key structural requirement used in the analysis. The Lipschitz continuity conditions (the last two items in Property \ref{ass:what we want}) are straightforward to verify.

Let $\bar G(x,\lambda):=1-G(x,\lambda)$ be the survival function. We define:
$$
\widetilde{G}(x,\lambda,\lambda') := 1 - s(\lambda,\lambda')\,\bar G(x,\lambda), \qquad x\in[0,B]
$$
where $s:\Lambda\times\Lambda\to\mathbb{R}_+$ satisfies $s(\lambda,\lambda)=1$. This automatically ensures consistency.
Moreover, if $s(\lambda,\lambda') \leq 1$, then $\widetilde{G} \leq 1$ and non-decreasing in $x$ whenever $x\mapsto \bar G(x,\lambda)$ is non-increasing. Notice also that $\partial_x\widetilde{G}=s(\lambda,\lambda')\,\partial_x G$. 
Then we can state the following general construction lemma for exponential family of distribution.

\begin{lemma}[Boosting monotonicity for exponential-family]
\label{lemma:boosting_monotone_exp_family}
Assume that for all $x\in[0,B]$ and $\lambda\in\Lambda$,
$$
\partial_xG(x,\lambda)=h(x)\,g(\lambda)\,\exp\!\big(\eta(\lambda)\,T(x)\big), \qquad h(x)\geq 0,\quad g(\lambda)>0
$$
Choose
$$
s(\lambda,\lambda'):=\frac{g(\lambda')}{g(\lambda)}.
$$
If $g$ is non-decreasing, $T(x)\geq 0$ on $[0,B]$, and $\eta$ is non-increasing in $\lambda \in \Lambda$, then for every fixed $x$ the map $\partial_x\widetilde{G}(x,\lambda,\lambda')$ is non-increasing in $\lambda$
and non-decreasing in $\lambda'$, satisfying Boosting Monotonicity Property \ref{prp: Boosting Monotonicity}.
\end{lemma}

\begin{proof}
With $s(\lambda,\lambda')=g(\lambda')/g(\lambda)$,
$$
\partial_x \widetilde{G}=\frac{g(\lambda')}{g(\lambda)}h(x)\,g(\lambda)\,e^{\eta(\lambda)T(x)} =h(x)\,g(\lambda')e^{\eta(\lambda)T(x)}.
$$
Fixing $(x,\lambda)$, monotonicity in $\lambda'$ follows from monotonicity of $g$.
Fixing $(x,\lambda')$, monotonicity in $\lambda$ follows since $T(x)\ge 0$ and $\eta$ is non-increasing, hence $e^{\eta(\lambda)T(x)}$ is non-increasing in $\lambda$.
\end{proof}

Lemma \ref{lemma:boosting_monotone_exp_family} ensures the Boosting Monotonicity property for a broad class of commonly used distributions, including, for instance, the exponential distribution, Weibull models with fixed shape $k$, Gamma models with fixed shape $\alpha$, as well as the half-normal, Maxwell, and Nakagami families, among others.
Moreover, even for distributions that cannot be cast directly into the canonical form required by Lemma \ref{lemma:boosting_monotone_exp_family}, one can often still design a scaling factor $s(\lambda,\lambda')$ that enforces Boosting Monotonicity; examples include the Lomax and Beta families.

If one wants to check concavity of $\widetilde{G}$ there is the following easy test. Assume $G(x,\lambda)$ is twice differentiable on $x \in [0,B]$. Since $\bar G=1-G$, we have $\bar G''(x,\lambda)=-\partial_x^2 G(x,\lambda)$ and therefore
$$
\partial_x^2 \widetilde{G}(x,\lambda,\lambda') =-s(\lambda,\lambda') \partial_x^2 \bar G(x,\lambda) = s(\lambda,\lambda') \partial_x^2 G(x,\lambda)
$$
Hence a simple sufficient condition for concavity of $\widetilde{G}$ in $x$ is:
$$
\partial_x^2 G(x,\lambda)\le 0 \quad \text{for all } x\in[0,B],\ \lambda\in\Lambda
$$
because then $\partial_x^2\widetilde{G}(x,\lambda,\lambda')\leq 0$ for every $\lambda,\lambda'$ as well. Equivalently, for each fixed $\lambda$, the function $\partial_x G(x,\lambda)$ is non-increasing on $x \in [0,B]$.
In contrast, distributions whose slope in $x$ is initially increasing (e.g., unimodal densities with a rising part near $0$) may violate concavity near the origin, and then the above construction does not yield a concave surrogate on the whole interval unless one restricts $[0,B]$ to a region where the slope is already decreasing.

\section{How to get Regret of $\mathcal{O}(\sqrt{T})$}\label{app_sec: regret of sqrt(T)}

\paragraph{1. Regret bound for the initialization phase:}
This initialization phase lasts for reward rounds $t \in \{1, \ldots, K\lfloor \log T \rfloor \}$. During this phase, we simply allocate the entire budget to a single arm for $\lfloor \log T \rfloor$ rounds and perform preliminary estimations of the parameters. Since the regret at each round is upper bounded by $K$, the total regret accumulated in this initialization phase is bounded by:
$$
\mathbb{E}[\mathcal{R}_{\{1,\ldots, K \lfloor \log T \rfloor\}}] \leq K \cdot K\log T = K^2 \log T
$$
\paragraph{2. Regret bound for the main phase:}
We now focus on an upper bound for the regret incurred during the main phase of \texttt{RA-UCB}. This phase corresponds to reward round $t \in \{ K\lfloor \log T \rfloor + 1, \ldots, T \}$.

Throughout the proof, the confidence level is fixed to $\alpha=T^{-3}$ for all rounds, which is the level used by the algorithm, so that $\log(1/\alpha)=3\log T$.
First, we define formally the good event at time $t-1$, $\mathcal{N}_{t-1}$, as the event in which the true parameters $\lambda_i, p_i$ for all $i \in [K]$ lie within the confidence intervals used by the algorithm, i.e.:
$$
\forall i \in [K], \;\;\; |\hat{\lambda}_{t-1,i}-\lambda_i| \leq \frac{b}{L_\mu} \sqrt{\frac{\log(1/\alpha)}{2n_{t-1,i}}} \; \text{ and } \; |\hat{p}_{t-1,i}-p_i| \leq \frac{L_\lambda}{L_\mu}\frac{2b}{C_b} \sqrt{\frac{\log(1/\alpha)}{2n_{t-1,i}}}
$$
with $C_b=G(b,m/B)$. Since $(1+p_i)/C_i^*\leq 2/C_b$ (Section~\ref{subsec:algo-radii}), the second condition is implied by the corresponding event of Lemma~\ref{lemma: conc ineq p_i}.
In order to bound the total expected regret, we can use the following decomposition:
\begin{equation}\label{eq: regret decomposition}
\begin{aligned}
    \mathbb{E}[\mathcal{R}_T] & = \sum_{t=K\lfloor \log T \rfloor+1}^T [\mathbb{E}[R_t|\mathcal{N}_{t-1}]\mathbb{P}(\mathcal{N}_{t-1})+\mathbb{E}[R_t|\overline{\mathcal{N}}_{t-1}]\mathbb{P}(\overline{\mathcal{N}}_{t-1})] \\
    & \leq \sum_{t=K\lfloor \log T \rfloor+1}^T \mathbb{E}[R_t|\mathcal{N}_{t-1}]+K\sum_{t=K\lfloor \log T \rfloor+1}^T\mathbb{P}(\overline{\mathcal{N}}_{t-1})
\end{aligned}
\end{equation}
where $\overline{\mathcal{N}}_{t-1}$ is the bad event. This is true since the regret for each round can be at most $K$.

Since the bad-event analysis below bounds $\mathbb{P}(\overline{\mathcal{N}}_{t-1})$ through $\mathbb{P}(\overline{\mathscr{G}}_{t-1})$ (see \eqref{eq: majoration of P(bar(N)_t)}), the same decomposition holds with the cumulative good event: $\mathbb{E}[\mathcal{R}_T]\leq\sum_t\mathbb{E}[R_t\mathds{1}_{\mathscr{G}_{t-1}}]+K\sum_t\mathbb{P}(\overline{\mathscr{G}}_{t-1})$, with the same bound on the second sum. Since $\mathscr{G}_{t-1}\in\mathcal{F}_{t-1}$ and $\mathscr{G}_{t-1}\subseteq\mathcal{N}_{t-1}$, we have $\mathbb{E}[R_t\mathds{1}_{\mathscr{G}_{t-1}}]=\mathbb{E}\big[\mathbb{E}[R_t\mid\mathcal{F}_{t-1}]\mathds{1}_{\mathscr{G}_{t-1}}\big]$, and all good-event bounds below are applied on $\mathscr{G}_{t-1}$. Accordingly, we write $\mathbb{E}[R_t\mid\mathcal{N}_{t-1}]$ for $\mathbb{E}[R_t\mathds{1}_{\mathscr{G}_{t-1}}]$, and $\mathbb{E}[1/\sqrt{n_{t-1,i}}]$, $\mathbb{E}[1/n_{t-1,i}]$ for the corresponding expectations multiplied by $\mathds{1}_{\mathscr{G}_{t-1}}$ (with $n_{t-1,i}$ replaced by $\max\{n_{t-1,i},1\}$).

Now we proceed by separating the regret bound in good and bad events.

Remind that the regret at round $t$ given the history up to time $t-1$, $\mathcal{F}_{t-1}$, is:
\begin{equation} \label{eq: reg_t_history_generic}
    \mathbb{E}[R_t|\mathcal{F}_{t-1}] = \sum_{i=1}^K \left[ p_iG(x_i^*,\lambda_i)-p_iG(x_{t,i},\lambda_i)\right]
\end{equation}
where $\bm{x^*} =(x_1^*,...,x_K^*)$ is the solution to $\mathcal{P}(\bm{\lambda}, \bm{p})$ and $\bm{x}_t = (x_{t,1},...,x_{t,K})$ is the solution to $\tilde{\mathcal{P}}(\bm{\bar{\lambda}_{t-1}}, \bm{\bar{\lambda}_{t-1}'}, \bar{\bm{p}}_{t-1})$ (the solution for round $t$ is determined by the estimators at time $t-1$).

By Assumption~\ref{ass: truncation level}, $x_i^*\geq b>0$ for every $i\in[K]$. We keep the notation $S := \{ i \in [K] : x_i^* > 0 \}$, which here equals $[K]$: all sums $\sum_{i\in S}$ below are sums over all arms, and no arm is discarded from the analysis.

\textbf{Good Event:}

We start by assuming that our estimations are good enough and so we are in a good event $\mathcal{N}_{t-1}$ at round $t-1$.
The algorithm \texttt{RA-UCB} at round $t$ solves the optimization problem $\tilde{\mathcal{P}}(\bm{\bar{\lambda}_{t-1}}, \bm{\bar{\lambda}'_{t-1}},\bar{\bm{p}}_{t-1})$ in \eqref{opt problem: surrogate algorithm maximization problem} and plays allocation $\bm{x}_t$.
The algorithm, every $K$ rounds, picks an arm $i$ and plays the following choice of parameters: $(\bm{\bar{\lambda}_{t-1}}, \bm{\bar{\lambda}'_{t-1}},\bar{\bm{p}}_{t-1}) = ( \{\hat{\lambda}_{t-1,1}^+,\hat{\lambda}_{t-1,1}^-,\hat{p}_{t-1,1}^-\},...,\{\hat{\lambda}_{t-1,i}^-,\hat{\lambda}_{t-1,i}^+,\hat{p}_{t,i}^+\},...,\{\hat{\lambda}_{t-1,K}^+,\hat{\lambda}_{t-1,K}^-,\hat{p}_{t-1,K}^-\})$. This specific configuration corresponds to playing the best parameters for arm $i$ and the worst ones for all the other arms $j \neq i$. In this way we ensure that, if we play this configuration at round $t_i$, then, by Oracle Monotonicity Assumption \ref{ass:OMA} and being under the good event $\mathcal{N}_{t_i-1}$, at that round, $x_{t_i,i} \geq x_i^*$.

On the good event, every such round is therefore selected ($\mathcal{E}_{t_i,i}$ holds, since $x_{t_i,i}\geq x_i^*\geq b$ by Assumption~\ref{ass: truncation level}) and yields an informative sample with conditional probability $q_i=p_iG(b,\lambda_i)$, independently of the exact allocation. The following lemma turns this observation into the bounds on the sample counts used below.

\begin{lemma}[Sample counts]\label{lemma: counting}
Assume $\lfloor\log T\rfloor\geq2$, and let $t\in\{K\lfloor\log T\rfloor+1,\ldots,T\}$ be a main-phase round. For every arm $i\in S$,
$$
\mathbb{E}\left[\frac{\mathds{1}_{\mathscr{G}_{t-1}}}{\sqrt{\max\{n_{t-1,i},1\}}}\right]\leq\sqrt{\frac{16K}{tq_i}},
\qquad
\mathbb{E}\left[\frac{\mathds{1}_{\mathscr{G}_{t-1}}}{\max\{n_{t-1,i},1\}}\right]\leq\frac{10K}{tq_i}.
$$
\end{lemma}

\begin{proof}
Call a round $u$ a \emph{candidate} round of arm $i$ if it is an initialization or estimation round of arm $i$; this schedule is deterministic. Let $m$ be the number of candidate rounds of arm $i$ in $\{1,\ldots,t-1\}$. All $\lfloor\log T\rfloor$ initialization rounds of arm $i$ precede $t$, and arm $i$ has one estimation round in every $K$ consecutive main-phase rounds, so
$$
m\geq\lfloor\log T\rfloor+\Big\lfloor\frac{t-1-K\lfloor\log T\rfloor}{K}\Big\rfloor\geq\Big\lfloor\frac{t-1}{K}\Big\rfloor\geq\frac{t-K}{K}\geq\frac{t}{2K},
$$
where the last inequality uses $t\geq K\lfloor\log T\rfloor+1\geq 2K$.

\emph{Coupling.} Let $(U_u)_{u\geq1}$ be i.i.d.\ uniform variables on $[0,1]$, independent of everything else. For each candidate round $u$, set $\xi_u:=\mathds{1}_{\mathcal{E}_{u,i}}W_{u,i}$ if $u$ is an initialization round or if $\mathscr{G}_{u-1}$ holds, and $\xi_u:=\mathds{1}\{U_u\leq q_i\}$ otherwise. On an initialization round, $x_{u,i}=B\geq b$; on an estimation round under $\mathscr{G}_{u-1}$, Oracle Monotonicity and Assumption~\ref{ass: truncation level} give $x_{u,i}\geq x_i^*\geq b$. In both cases $\mathcal{E}_{u,i}$ holds and, by \eqref{eq: truncated law}, $\mathbb{P}(W_{u,i}=1\mid\mathcal{F}_{u-1})=q_i$. Since $\mathscr{G}_{u-1}\in\mathcal{F}_{u-1}$, conditionally on $\mathcal{F}_{u-1}$ and $(U_v)_{v<u}$ each $\xi_u$ is Bernoulli with parameter $q_i$. Hence the $\xi_u$ over the $m$ candidate rounds are i.i.d.\ $\mathcal{B}(q_i)$, and their sum $S$ follows $\mathrm{Bin}(m,q_i)$.

On $\mathscr{G}_{t-1}$, the event $\mathscr{G}_{u-1}$ holds for every $u\leq t-1$, so $\xi_u=\mathds{1}_{\mathcal{E}_{u,i}}W_{u,i}$ for every candidate round, and $n_{t-1,i}=S$ (non-candidate rounds never belong to $\mathcal{E}_{u,i}$). Therefore, for any non-increasing $h\geq0$,
$$
\mathds{1}_{\mathscr{G}_{t-1}}\,h\big(\max\{n_{t-1,i},1\}\big)\leq h\big(\max\{S,1\}\big).
$$

\emph{Binomial bounds.} By the multiplicative Chernoff bound, $\mathbb{P}(S\leq mq_i/2)\leq(\sqrt{2}e^{-1/2})^{mq_i}$. Splitting on this event,
$$
\mathbb{E}\Big[\frac{1}{\sqrt{\max\{S,1\}}}\Big]\leq(\sqrt{2}e^{-1/2})^{mq_i}+\sqrt{\frac{2}{mq_i}}\leq\sqrt{\frac{8}{mq_i}},
\qquad
\mathbb{E}\Big[\frac{1}{\max\{S,1\}}\Big]\leq(\sqrt{2}e^{-1/2})^{mq_i}+\frac{2}{mq_i}\leq\frac{5}{mq_i},
$$
using $(\sqrt{2}e^{-1/2})^x\leq\sqrt{2/x}$ and $(\sqrt{2}e^{-1/2})^x\leq 3/x$ for all $x>0$. The claim follows from $m\geq t/(2K)$.
\end{proof}
For $T<e^2$ the regret is at most $KT<8K$, which is absorbed in the additive term of Theorems~\ref{thm: regret sqrt(T)} and~\ref{thm: regret bound of logT}.

Define the true objective (expected reward) as
$$
F(\bm{x},\bm{p},\bm{\lambda}) := \sum_{i\in S} p_i\,G(x_i,\lambda_i),
$$
and the surrogate objective optimized by the oracle as
$$
\widetilde{F}(\bm{x},\bm{p},\bm{\lambda},\bm{\lambda}') := \sum_{i\in S} p_i\,\widetilde{G}(x_i,\lambda_i,\lambda_i').
$$
By design, $\bm{x}^*$ maximizes $F(\cdot,\bm{p},\bm{\lambda})$ and $\bm{x}_t$ maximizes $\widetilde{F}(\cdot,\bar{\bm{p}}_{t-1},\bar{\bm{\lambda}}_{t-1},\bar{\bm{\lambda}}'_{t-1})$.

We know that the surrogate $\widetilde{G}$ satisfies Property \ref{prp: Uniform Bias Control} of uniform bias control, meaning for all $x\in[0,B]$,
\begin{equation}\label{eq:Delta_def}
\big|\widetilde{G}(x;\lambda,\lambda') - G(x;\lambda)\big| \leq L_\Delta |\lambda-\lambda'|
\end{equation}

From \eqref{eq: reg_t_history_generic}, using $F$ we write
$$
\mathbb{E}[R_t\mid \mathcal{F}_{t-1}] = F(\bm{x}^*,\bm{p},\bm{\lambda}) - F(\bm{x}_t,\bm{p},\bm{\lambda})
$$
Add and subtract the following terms:
\begin{equation}\label{eq:generic_add_subtract}
\begin{aligned}
F(\bm{x}^*,\bm{p},\bm{\lambda}) - F(\bm{x}_t,\bm{p},\bm{\lambda}) &\leq
\underbrace{\Big(F(\bm{x}^*,\bm{p},\bm{\lambda})-\widetilde{F}(\bm{x}^*,\bm{p},\bm{\lambda},\bm{\lambda})\Big)}_{=0} \\
&\quad+\Big(\widetilde{F}(\bm{x}^*,\bm{p},\bm{\lambda},\bm{\lambda}) - \widetilde{F}(\bm{x}^*,\bar{\bm{p}}_{t-1},\bar{\bm{\lambda}}_{t-1},\bar{\bm{\lambda}}'_{t-1})\Big) \\
&\quad+ \underbrace{\Big(\widetilde{F}(\bm{x}^*,\bar{\bm{p}}_{t-1},\bar{\bm{\lambda}}_{t-1},\bar{\bm{\lambda}}'_{t-1}) -
\widetilde{F}(\bm{x}_t,\bar{\bm{p}}_{t-1},\bar{\bm{\lambda}}_{t-1},\bar{\bm{\lambda}}'_{t-1})\Big)}_{\le 0} \\
&\quad+ \Big(\widetilde{F}(\bm{x}_t,\bar{\bm{p}}_{t-1},\bar{\bm{\lambda}}_{t-1},\bar{\bm{\lambda}}'_{t-1}) - F(\bm{x}_t,\bm{p},\bm{\lambda})\Big)
\end{aligned}
\end{equation}
The first term is zero by consistency Property \ref{prp: Consistency} of the surrogate, $\widetilde{G}(x;\lambda,\lambda)=G(x;\lambda)$. The third term is non-positive because $\bm{x}_t$ maximizes the surrogate objective.

Thus
\begin{equation}\label{eq:generic_regret_bound_step}
\mathbb{E}[R_t\mid \mathcal{F}_{t-1}] \leq A_t + B_t,
\end{equation}
where
$$
A_t := \Big|\widetilde{F}(\bm{x}^*,\bm{p},\bm{\lambda},\bm{\lambda})-\widetilde{F}(\bm{x}^*,\bar{\bm{p}}_{t-1},\bar{\bm{\lambda}}_{t-1},\bar{\bm{\lambda}}'_{t-1})\Big|
$$
and
$$
B_t := \Big| \widetilde{F}(\bm{x}_t,\bar{\bm{p}}_{t-1},\bar{\bm{\lambda}}_{t-1},\bar{\bm{\lambda}}'_{t-1}) - F(\bm{x}_t,\bm{p},\bm{\lambda})\Big|
$$

Let us start by bounding $A_t$.

\begin{equation}\label{eq: add substract and using Lipschitz}
\begin{aligned}
    \Big|\widetilde{F}(\bm{x}^*,\bm{p},\bm{\lambda},\bm{\lambda})-\widetilde{F}(\bm{x}^*,\bar{\bm{p}}_{t-1},\bar{\bm{\lambda}}_{t-1},\bar{\bm{\lambda}}'_{t-1})\Big| & = \big| \sum_{i \in S} \big[ p_i\widetilde{G}(x_i^*,\lambda_i,\lambda_i) - \bar{p}_{t-1,i}\widetilde{G}(x_i^*,\bar{\lambda}_{t-1,i}, \bar{\lambda}_{t-1,i}') \big] \big| \\
    &\leq \sum_{i \in S} \big| p_i\widetilde{G}(x_i^*,\lambda_i,\lambda_i) - \bar{p}_{t-1,i}\widetilde{G}(x_i^*,\bar{\lambda}_{t-1,i}, \bar{\lambda}_{t-1,i}') \big| \\
    & = \sum_{i \in S} \big| p_i\widetilde{G}(x_i^*,\lambda_i,\lambda_i) + p_i\widetilde{G}(x_i^*,\bar{\lambda}_{t-1,i}, \bar{\lambda}_{t-1,i}') \\
    & \quad- p_i\widetilde{G}(x_i^*,\bar{\lambda}_{t-1,i}, \bar{\lambda}_{t-1,i}') - \bar{p}_{t-1,i}\widetilde{G}(x_i^*,\bar{\lambda}_{t-1,i}, \bar{\lambda}_{t-1,i}') \big| \\
    & \leq \sum_{i \in S} p_i \big| \widetilde{G}(x_i^*, \lambda_i, \lambda_i) - \widetilde{G}(x_i^*,\bar{\lambda}_{t-1,i}, \bar{\lambda}_{t-1,i}') \big| \\
    & \quad + \sum_{i \in S} \widetilde{G}(x_i^*,\bar{\lambda}_{t-1,i}, \bar{\lambda}_{t-1,i}')\big| p_i - \bar{p}_{t-1,i} \big|
\end{aligned}
\end{equation}

So by Property \ref{prp: Boundness and Monotonicity} we know that $\widetilde{G} \leq 1$ and the fact that $\forall i \in [K]$, $p_i \leq 1$, we have
\begin{equation}
\begin{aligned}
A_t &\leq \sum_{i\in S} |p_i-\bar{p}_{t-1,i}| + \sum_{i\in S} \sup_{x\in[0,B]} \big| \widetilde{G}(x,\lambda_i,\lambda_i) - \widetilde{G}(x,\bar{\lambda}_{t-1,i},\bar{\lambda}'_{t-1,i}) \big|
\end{aligned}
\end{equation}
By Property \ref{prp:lipschitz_tildeG} of Lipschitz continuity of $\widetilde{G}$ in $\lambda$ and $\lambda'$ uniformly over $x\in[0,B]$
\begin{equation}\label{eq:tildeG_Lipschitz}
\big|\widetilde{G}(x,\lambda,\lambda)-\widetilde{G}(x,\bar{\lambda},\bar{\lambda}')\big| \leq L_{\widetilde{\lambda}}|\lambda-\bar{\lambda}| + L_{\widetilde{\lambda}'}|\lambda-\bar{\lambda}'|.
\end{equation}
Then
\begin{equation}\label{eq:At_bound}
A_t \leq \sum_{i\in S}|p_i-\bar{p}_{t-1,i}| + \sum_{i\in S} \Big(L_{\widetilde{\lambda}}|\lambda_i-\bar{\lambda}_{t-1,i}| + L_{\widetilde{\lambda}'}|\lambda_i-\bar{\lambda}'_{t-1,i}| \Big)
\end{equation}

Now let us focus on $B_t$.

Decompose
$$
B_t \leq \Big| \widetilde{F}(\bm{x}_t,\bar{\bm{p}}_{t-1},\bar{\bm{\lambda}}_{t-1},\bar{\bm{\lambda}}'_{t-1}) - F(\bm{x}_t,\bar{\bm{p}}_{t-1},\bar{\bm{\lambda}}_{t-1}) \Big| + \Big| F(\bm{x}_t,\bar{\bm{p}}_{t-1},\bar{\bm{\lambda}}_{t-1}) - F(\bm{x}_t,\bm{p},\bm{\lambda}) \Big|
$$
The first term is controlled by \eqref{eq:Delta_def}:
$$
\Big| \widetilde{F}(\bm{x}_t,\bar{\bm{p}}_{t-1},\bar{\bm{\lambda}}_{t-1},\bar{\bm{\lambda}}_{t-1}') - F(\bm{x}_t,\bar{\bm{p}},\bar{\bm{\lambda}}) \Big| \leq
L_\Delta\sum_{i\in S} \bar{p}_{t-1,i}\,|\bar{\lambda}_{t-1,i}-\bar{\lambda}'_{t-1,i}| \leq L_\Delta \sum_{i\in S} |\bar{\lambda}_{t-1,i}-\bar{\lambda}'_{t-1,i}| 
$$
since $\bar{p}_{t-1,i}\le 1$.
For the second term, applying the same procedure of Equation \eqref{eq: add substract and using Lipschitz} and using the Lipschitz continuity of $G$ in $\lambda$,
$$
\Big| F(\bm{x}_t,\bar{\bm{p}},\bar{\bm{\lambda}}) - F(\bm{x}_t,\bm{p},\bm{\lambda}) \Big| \leq \sum_{i\in S}|p_i-\bar{p}_{t-1,i}| + L_\lambda \sum_{i\in S}|\lambda_i-\bar{\lambda}_{t-1,i}|
$$
Hence
\begin{equation}\label{eq:Bt_bound}
B_t \leq L_\Delta\sum_{i\in S}|\bar{\lambda}_{t-1,i}-\bar{\lambda}'_{t-1,i}| + \sum_{i\in S}|p_i-\bar{p}_{t-1,i}| + L_\lambda \sum_{i\in S}|\lambda_i-\bar{\lambda}_{t-1,i}|
\end{equation}

Now, combining \eqref{eq:generic_regret_bound_step}, \eqref{eq:At_bound}, and \eqref{eq:Bt_bound},
there exist constants $c_1,c_2,c_3>0$ (depending only on Lipschitz constants) such that
\begin{equation}\label{eq:Rt_good_generic}
\mathbb{E}[R_t\mid \mathcal{F}_{t-1}] \leq 2\sum_{i\in S}|p_i-\bar{p}_{t-1,i}| + (L_{\tilde{\lambda}}+L_{\lambda})\sum_{i\in S}|\lambda_i-\bar{\lambda}_{t-1,i}| + L_{\tilde{\lambda}'}\sum_{i\in S}|\lambda_i-\bar{\lambda}'_{t-1,i}| + L_\Delta \sum_{i\in S}|\bar{\lambda}_{t-1,i}-\bar{\lambda}'_{t-1,i}|
\end{equation}

Under the good event $\mathcal{N}_{t-1}$, the boosted parameters satisfy
$$
|\bar{\lambda}_{t-1,i}-\lambda_i| \leq 2\cdot \frac{b}{L_\mu}\sqrt{\frac{\log(1/\alpha)}{2n_{t-1,i}}},
\qquad |\bar{p}_{t-1,i}-p_i| \leq 2\cdot \frac{L_\lambda}{L_\mu}\frac{2b}{C_b}\sqrt{\frac{\log(1/\alpha)}{2n_{t-1,i}}}
$$
and similarly
$$
|\bar{\lambda}'_{t-1,i}-\lambda_i| \leq 2\cdot \frac{b}{L_\mu}\sqrt{\frac{\log(1/\alpha)}{2n_{t-1,i}}}.
$$
Then \eqref{eq:Rt_good_generic} implies
\begin{equation}\label{eq:Rt_good_sqrt}
\mathbb{E}[R_t\mid \mathcal{N}_{t-1}] \leq \widetilde{C}_1 \sqrt{\log(1/\alpha)}\sum_{i\in S}\Bigg(\frac{2}{C_b}\Bigg)\mathbb{E}\!\left[\frac{1}{\sqrt{n_{t-1,i}}}\right] + \widetilde{C}_2 \sqrt{\log(1/\alpha)}\sum_{i\in S}\mathbb{E}\!\left[\frac{1}{\sqrt{n_{t-1,i}}}\right]
\end{equation}
for a constant $\widetilde{C}_1, \widetilde{C}_2>0$ absorbing all problem-dependent factors. Precisely:
$$
\widetilde{C}_1 = \frac{2\sqrt{2}bL_\lambda}{L_\mu} \quad ; \quad \widetilde{C}_2 = \frac{\sqrt{2}b}{L_\mu}\Bigg(L_{\tilde{\lambda}}+L_\lambda+L_{\tilde{\lambda}'} +2L_\Delta \Bigg)
$$

By Lemma~\ref{lemma: counting}, in the sense specified after \eqref{eq: regret decomposition},
$$
\mathbb{E}\left[\frac{1}{\sqrt{n_{t-1,i}}}\right]\leq\sqrt{\frac{16K}{tq_i}} .
$$
Therefore, with $\alpha = T^{-3}$, so that $\sqrt{\log(1/\alpha)}=\sqrt{3\log T}$, summing \eqref{eq:Rt_good_sqrt} over $t$ yields
$$
\sum_{t=K\lfloor \log T \rfloor+1}^T \mathbb{E}[R_t\mid \mathcal{N}_{t-1}] \leq C_{\text{sqrt}}\,\sqrt{KT\log T}
$$
for a problem-dependent constant $C_{\text{sqrt}}>0$:
$$
C_{\text{sqrt}} = \sqrt{192}\Bigg(\widetilde{C}_1\sum_{i \in S}\Big(\frac{2}{C_b\sqrt{q_i}}\Big)+\widetilde{C}_2\sum_{i \in S} \Big(\frac{1}{q_i}\Big) \Bigg)
$$
where we used the inequality $\sum_{t=1}^T 1/\sqrt{t} \leq 2\sqrt{T}$.

\textbf{Bad event:}

Now we want to bound the part of the regret depending on the probability of being in a bad event at round $t$ in Equation \eqref{eq: regret decomposition}.

Remind that $\mathscr{G}_t = \{ \mathcal{N}_1 \cap \ldots \cap \mathcal{N}_t \}$ and then $\overline{\mathscr{G}}_t$ represents the event for which up to round $t$ there was at least one bad event.

So it's true that:
\begin{equation}\label{eq: majoration of P(bar(N)_t)}
\mathbb{P}(\overline{\mathcal{N}}_t) \leq \mathbb{P}(\overline{\mathscr{G}}_t) \leq \sum_{u=1}^t \mathbb{P}(\overline{\mathcal{N}}_u \cap \mathscr{G}_{u-1})
\end{equation}

Now it's possible to compute the term $\mathbb{P}(\overline{\mathcal{N}}_t \cap \mathscr{G}_{t-1})$. Recall that here $S=[K]$, so $|S|=K$. Define the deviation events
$$
\begin{aligned}
E^\lambda_{t,i}
&:=
\left\{
|\hat{\lambda}_{t,i}-\lambda_i|
\geq
\frac{b}{L_\mu}
\sqrt{\frac{\log(1/\alpha)}{2n_{t,i}}}
\right\},\\
E^p_{t,i}
&:=
\left\{
|\hat{p}_{t,i}-p_i|
\geq
\frac{L_\lambda}{L_\mu}
\frac{2b}{C_b}
\sqrt{\frac{\log(1/\alpha)}{2n_{t,i}}}
\right\}.
\end{aligned}
$$
Since $(1+p_i)/C_i^*\leq 2/C_b$, the event $E^p_{t,i}$ is contained in the deviation event of Lemma~\ref{lemma: conc ineq p_i}, so the bound of that lemma applies to $E^p_{t,i}\cap\mathscr{G}_{t-1}$.
By the union bound and the concentration inequalities,
\begin{equation}
\begin{aligned}
\mathbb{P}
\bigl(\bar{\mathcal{N}}_t\cap\mathscr{G}_{t-1}\bigr)
&=
\mathbb{P}\left(
\left[
\bigcup_{i\in S}
\bigl(E^\lambda_{t,i}\cup E^p_{t,i}\bigr)
\right]
\cap\mathscr{G}_{t-1}
\right)\\
&\leq
\sum_{i\in S}
\mathbb{P}\left(
\bigl(E^\lambda_{t,i}\cup E^p_{t,i}\bigr)
\cap\mathscr{G}_{t-1}
\right)\\
&\leq
\sum_{i\in S}
\mathbb{P}\bigl(E^\lambda_{t,i}\cap\mathscr{G}_{t-1}\bigr)+\sum_{i\in S} \mathbb{P}\bigl(E^p_{t,i}\cap\mathscr{G}_{t-1}\bigr)\\
&\leq \sum_{i\in S}\mathbb{P}\bigl(E^\lambda_{t,i}\bigr)+
\sum_{i\in S} \mathbb{P}\bigl(E^p_{t,i}\cap\mathscr{G}_{t-1}\bigr)\\
&= \sum_{i\in S}\sum_{s=1}^{t} \mathbb{P}\bigl(E^\lambda_{t,i}\cap\{n_{t,i}=s\}\bigr)+ \sum_{i\in S} \mathbb{P}\bigl(E^p_{t,i}\cap\mathscr{G}_{t-1}\bigr)\\
&\leq
\sum_{i\in S}\sum_{s=1}^{t}2\alpha + \sum_{i\in S} \left( 2\alpha^{\frac{b^2L_\lambda^2}{4L_\mu^2}} +2\alpha t \right)\\
&= 2|S|\left( 2\alpha t+ \alpha^{\frac{b^2L_\lambda^2}{4L_\mu^2}}
\right)\\
&\leq 2K\left( 2\alpha t+ \alpha^{\frac{b^2L_\lambda^2}{4L_\mu^2}}
\right).
\end{aligned}
\end{equation}
Let $c:=b^2L_\lambda^2/(4L_\mu^2)$. The bound above reads $\mathbb{P}(\overline{\mathcal{N}}_u\cap\mathscr{G}_{u-1})\leq 2K(2\alpha u+\alpha^{c})$ for every $u$. Since $\alpha=T^{-3}$ is the same at every round, using Equation \eqref{eq: majoration of P(bar(N)_t)}, for every $t\leq T$:
\begin{equation}
\begin{aligned}
    \mathbb{P}(\overline{\mathscr{G}}_t) & \leq \sum_{u=1}^t \mathbb{P}(\overline{\mathcal{N}}_u \cap \mathscr{G}_{u-1}) \leq 2K\sum_{u=1}^t \big(2\alpha u+\alpha^{c}\big) = 2K\big(\alpha t(t+1)+t\alpha^{c}\big)\\
    & \leq 2K\Big(\frac{T(T+1)}{T^3}+T^{1-3c}\Big) \leq 2K\Big(\frac{2}{T}+\frac{1}{T}\Big) = \frac{6K}{T},
\end{aligned}
\end{equation}
where we used $t\leq T$ and $3c\geq 2$, i.e., $T^{1-3c}\leq T^{-1}$. The condition $3c\geq2$ is equivalent to
$$
\frac{b^2L_\lambda^2}{L_\mu^2} \geq \frac{8}{3},
$$
which is the assumption $bL_\lambda/L_\mu\geq\sqrt{8/3}$ of Theorem~\ref{thm: regret sqrt(T)}.

According to Equation \eqref{eq: regret decomposition}, written with the cumulative good event, we need to upper bound $\sum_{t} \mathbb{P}(\overline{\mathscr{G}}_{t-1})$. Hence:
\begin{equation} \label{eq: regret bound for bad event}
        \sum_{t=K\lfloor \log T \rfloor+1}^T \mathbb{P}(\overline{\mathscr{G}}_{t-1}) \leq \sum_{t=1}^{T} \frac{6K}{T} = 6K .
\end{equation}
Since the regret of each round is at most $K$, the regret contribution of bad events is at most $6K^2$.

\paragraph{Conclusion}

Combining the regret bounds for initialization, good events in the main phase, and bad events in the main phase, we conclude that, for $T\geq e^2$,
$$
\mathbb{E}[\mathcal{R}_T]\leq C_{\text{sqrt}}\,\sqrt{KT\log T}+K^2 \log T +6K^2
$$
which is $\widetilde{O}(\sqrt{KT})$.
\hfill$\square$

\section{How to get Regret of $\mathcal{O}(\mathrm{poly}\log(T))$}\label{app:logT-regret-ub-generic}

The parameters are in a compact set such that for every arm $i \in [K]$:
$$
\lambda_i \in \Big[\frac{m}{B},\frac{M}{B}\Big], \qquad p_i \in [p_{\min} ,1]
$$
by the parameter-domain condition in Section~\ref{sec: known budget} and Assumption~\ref{ass:uniform_sc_smooth}(i). Throughout this section, Assumption~\ref{ass: truncation level} holds, so the optimal allocation is interior and $S=[K]$. Allocations satisfy $x_i\in[0,B]$.
We work with the surrogate objective used by the algorithm
$$
\widetilde{F}(\bm{x},\bm{\theta}) := \sum_{i=1}^K p_i\,\widetilde{G}(x_i,\lambda_i,\lambda_i'), \qquad \bm{\theta} =(\bm{p},\bm{\lambda},\bm{\lambda}'), \qquad \bm{x}\in\mathcal{X}:=\Big\{\bm{x}\in\mathbb{R}^K_+:\sum_{i=1}^K x_i\le B\Big\}.
$$
By Consistency Property \ref{prp: Consistency}, $\widetilde{F}(\bm{x},\bm{p},\bm{\lambda},\bm{\lambda})=\sum_i p_i G(x_i,\lambda_i)$ is the true expected reward.

\begin{remark}
For the true parameter vector $\bm{\theta}=(\bm{p},\bm{\lambda},\bm{\lambda})$, Assumption~\ref{ass:uniform_sc_smooth} implies that $\widetilde{F}(\cdot,\bm{\theta})$ is $\mu$-strongly concave and $L$-smooth on $\mathcal{X}$ with
$$
\mu:=p_{\min}\mu_G, \qquad L:=L_G,
$$
since $-\nabla^2_{\bm{x}}\widetilde{F}$ is diagonal and each diagonal entry is in $[p_{\min}\mu_G,L_G]$. The boosted objective is concave, but need not have the same strong-concavity constant because a pessimistic activation estimate may be zero. The following lemma only requires strong concavity of the true objective.
\end{remark}

\begin{lemma}[Stability of the argmax under gradient perturbations]
\label{lemma:stability_argmax}
Let $\mathcal{X}\subseteq\mathbb{R}^K$ be nonempty, closed, and convex.
Let $\widetilde{F}(\cdot,\theta)$ be differentiable and $\mu$-strongly concave on $\mathcal{X}$ for some $\mu>0$, and let $\widetilde{F}(\cdot,\bar{\theta})$ be differentiable and concave. Assume that both maximization problems attain their optima.

Define
$$
\bm{x}(\theta)\in\arg\max_{\bm{x}\in\mathcal{X}} \widetilde{F}(\bm{x},\theta), \qquad \bm{x}(\bar\theta)\in\arg\max_{\bm{x}\in\mathcal{X}} \widetilde{F}(\bm{x},\bar\theta)
$$
Then
$$
\|\bm{x}(\bar\theta)-\bm{x}(\theta)\|_2 \leq \frac{1}{\mu}\sup_{\bm{x}\in\mathcal{X}} \|\nabla_{\bm{x}} \widetilde{F}(\bm{x},\bar\theta)  -\nabla_{\bm{x}} \widetilde{F}(\bm{x},\theta)\|_2
$$
\end{lemma}

\begin{proof}
Consider $f(\bm{x},\theta):=-\widetilde{F}(\bm{x},\theta)$. Then $f(\cdot,\theta)$ is $\mu$-strongly convex and $f(\cdot,\bar{\theta})$ is convex on $\mathcal{X}$, and
$$
\arg\max_{\bm{x}\in\mathcal{X}}\widetilde{F}(\bm{x},\theta)=\arg\min_{\bm{x}\in\mathcal{X}}f(\bm{x},\theta)
$$
Since $\mathcal{X}$ is convex and $f(\cdot,\theta)$ is differentiable and convex, the first-order optimality condition in variational inequality form yields
$$
\langle \nabla f(\bm{x}(\theta),\theta), \bm{z}-\bm{x}(\theta)\rangle \geq 0 \qquad \forall \bm{z}\in\mathcal{X}
$$
Applying the same condition to $f(\cdot,\bar{\theta})$ gives
$$
\langle \nabla f(\bm{x}(\bar{\theta}),\bar{\theta}), \bm{z}-\bm{x}(\bar{\theta})\rangle \geq 0 \qquad \forall \bm{z}\in\mathcal{X}
$$

Choosing $\bm{z}=\bm{x}(\bar{\theta})$ in the first inequality and $\bm{z}=\bm{x}(\theta)$ in the second, and rearranging terms, we obtain
$$
\langle \nabla f(\bm{x}(\theta),\theta) - \nabla f(\bm{x}(\bar{\theta}),\bar{\theta}), \bm{x} (\bar{\theta})-\bm{x}(\theta)\rangle \geq 0
$$

Let $\bm{d} := \bm{x}(\bar{\theta})-\bm{x}(\theta)$. Adding and subtracting $\nabla f(\bm{x}(\bar{\theta}),\theta)$ yields
$$
\langle \nabla f(\bm{x}(\theta),\theta) - \nabla f(\bm{x}(\bar{\theta}),\theta), \bm{d}\rangle \geq -\langle \nabla f(\bm{x}(\bar{\theta}),\theta)- \nabla f(\bm{x}(\bar{\theta}),\bar{\theta}), \bm{d}\rangle.
$$

Since $f(\cdot,\theta)$ is $\mu$-strongly convex on $\mathcal{X}$, its gradient is $\mu$-strongly monotone, and therefore
$$
\langle \nabla f(\bm{x}(\theta),\theta) - \nabla f(\bm{x}(\bar{\theta}),\theta), \bm{x}(\theta)-\bm{x}(\bar{\theta})\rangle \geq \mu\|\bm{x}(\theta)-\bm{x}(\bar{\theta})\|_2^2
$$
Recalling that $\bm{x}(\theta)-\bm{x}(\bar{\theta})=-\bm{d}$, this implies
$$
\langle \nabla f(\bm{x}(\theta),\theta) - \nabla f(\bm{x}(\bar{\theta}),\theta), \bm{d}\rangle \leq -\mu\|\bm{d}\|_2^2
$$

Combining the previous inequalities yields
$$
\mu\|\bm{d}\|_2^2 \leq \langle \nabla f(\bm{x}(\bar{\theta}),\theta) - \nabla f(\bm{x}(\bar{\theta}),\bar{\theta}), \bm{d}\rangle
$$

Applying the Cauchy-Schwarz inequality, we obtain
$$
\langle \nabla f(\bm{x}(\bar{\theta}),\theta) - \nabla f(\bm{x}(\bar{\theta}),\bar{\theta}), \bm{d}\rangle \leq \|\nabla f(\bm{x}(\bar{\theta}),\theta) - \nabla f(\bm{x}(\bar{\theta}),\bar{\theta})\|_2 \,\|\bm{d}\|_2
$$

If $\bm{d}\neq \bm{0}$, dividing both sides by $\|\bm{d}\|_2$ yields
$$
\|\bm{x}(\bar{\theta})-\bm{x}(\theta)\|_2 \leq \frac{1}{\mu} \|\nabla f(\bm{x}(\bar{\theta}),\theta) - \nabla f(\bm{x}(\bar{\theta}),\bar{\theta})\|_2
$$
The claim follows by taking the supremum over $\bm{x}\in\mathcal{X}$. The case $\bm{d}=\bm{0}$ is trivial.
\end{proof}

\subsection{Proof of Theorem \ref{thm: regret bound of logT}}

\paragraph{Quadratic bound}

The expected regret at time $t$ conditioned on the history $\mathcal{F}_{t-1}$ is:
$$
\mathbb{E}[R_t\mid\mathcal{F}_{t-1}] = \sum_{i=1}^K p_i\Big(G(x_i^*,\lambda_i)-G(x_{t,i},\lambda_i)\Big)
$$
By Consistency Property \ref{prp: Consistency}, this is:
$$
\mathbb{E}[R_t\mid\mathcal{F}_{t-1}] = \widetilde{F}(\bm{x}^*,\bm{p},\bm{\lambda},\bm{\lambda})-\widetilde{F}(\bm{x}_t,\bm{p},\bm{\lambda},\bm{\lambda})
$$
Define again the convex function:
$$
f(\bm{x}):=-\widetilde{F}(\bm{x},\bm{p},\bm{\lambda},\bm{\lambda})
$$
To be coherent we would have $f(\bm{x})=f(\bm{x},\bm{p},\bm{\lambda},\bm{\lambda})$, but let us forget about the dependence on the parameters in the following discussion for the sake of simplicity.
By Assumption \ref{ass:uniform_sc_smooth}, $f$ is $L$-smooth with $L=L_G$.

The true objective is strictly increasing in each coordinate on $[0,B)$: by consistency and Assumption~\ref{ass:uniform_sc_smooth}, its component derivatives are nonnegative and strictly decreasing. Consequently, the true optimum exhausts the budget. By the budget-saturating oracle convention, $\bm{x}_t$ also satisfies $\sum_i x_{t,i}=B$, even if some boosted activation estimates are zero.

Assumption~\ref{ass: truncation level} gives $x_i^*\geq b>0$ for every $i\in[K]$. Hence $S=[K]$, and we retain all coordinates throughout the argument. No interiority assumption is imposed on the algorithm's allocation $\bm{x}_t$.

Let $\nu\ge 0$ and $\xi_i\ge 0$ denote the corresponding Lagrange multipliers.
The Lagrangian is given by
$$
\mathcal{L}(\bm{x},\nu,\bm{\xi}) = f(\bm{x}) + \nu\Big(\sum_{i=1}^K x_i - B\Big) - \sum_{i=1}^K \xi_i x_i 
$$

By the KKT conditions, stationarity implies
$$
\nabla_{\bm{x}} \mathcal{L}(\bm{x}^*,\nu,\bm{\xi}) = \bm{0}
$$
that is, for each $i\in[K]$,
$$
\partial_i f(\bm{x}^*) + \nu - \xi_i = 0
$$
Since $x_i^*\geq b>0$ for every $i\in[K]$ by Assumption~\ref{ass: truncation level}, complementary slackness yields $\xi_i=0$, and therefore
$$
\partial_i f(\bm{x}^*) + \nu = 0 \qquad \forall i
$$
Equivalently, in vector form,
$$
\nabla f(\bm{x}^*) = -\nu \bm{1}
$$

hence for any $\bm{x}$ with $\sum_i x_i=B$,
$$
\langle \nabla f(\bm{x}^*),\bm{x}-\bm{x}^*\rangle =-\nu\Big(\sum_i x_i-\sum_i x_i^*\Big)=0.
$$

Recall the standard smoothness inequality: if a function $f$ is $L$-smooth, then for any $\bm{x},\bm{y}\in\mathbb{R}^K$,
$$
f(\bm{x}) \leq f(\bm{y}) + \langle \nabla f(\bm{y}), \bm{x}-\bm{y} \rangle + \frac{L}{2}\|\bm{x}-\bm{y}\|_2^2 
$$

Using this smoothness inequality for $f$ at $\bm{y}=\bm{x}^*$ and $\bm{x}=\bm{x}_t$, plus the above consideration on the linear term, yields
$$
f(\bm{x}_t)-f(\bm{x}^*) \leq \frac{L}{2}\|\bm{x}_t-\bm{x}^*\|_2^2.
$$
Since $f(\bm{x}_t)-f(\bm{x}^*)=\widetilde{F}(\bm{x}^*)-\widetilde{F}(\bm{x}_t)=\mathbb{E}[R_t\mid\mathcal{F}_{t-1}]$, we obtain
$$
\mathbb{E}[R_t\mid\mathcal{F}_{t-1}] \leq \frac{L}{2}\|\bm{x}_t-\bm{x}^*\|_2^2.
$$

\paragraph{Stability of the argmax}
Observe that:
$$
\bm{x}^* \in \arg\max_{\bm{x}\in\mathcal{X}} \widetilde{F}(\bm{x},\bm{p},\bm{\lambda},\bm{\lambda}), \qquad \bm{x}_t \in \arg\max_{\bm{x}\in\mathcal{X}} \widetilde{F}(\bm{x},\bm{\bar p}_{t-1},\bm{\bar\lambda}_{t-1},\bm{\bar\lambda}'_{t-1}),
$$
where $(\bm{\bar p}_{t-1},\bm{\bar\lambda}_{t-1},\bm{\bar\lambda}'_{t-1})$ are the boosted parameters used by the algorithm at round $t$.

The true objective is $\mu$-strongly concave with $\mu=p_{\min}\mu_G$. The boosted objective is concave because its weights satisfy $\bar p_{t-1,i}\in[0,1]$. Therefore Lemma~\ref{lemma:stability_argmax} applies without requiring $\bar p_{t-1,i}\geq p_{\min}$ or uniqueness of the boosted optimizer.
Applying Lemma \ref{lemma:stability_argmax} gives
$$
\|\bm{x}_t-\bm{x}^*\|_2 \leq \frac{1}{\mu}\sup_{\bm{x}\in\mathcal{X}} \|\nabla_{\bm{x}} \widetilde{F}(\bm{x},\bm{\bar\theta}_{t-1})-\nabla_{\bm{x}} \widetilde{F}(\bm{x},\bm{\theta})\|_2
$$

We now bound the gradient perturbation.
For each coordinate,
$$
\frac{\partial \widetilde{F}}{\partial x_i}(\bm{x},\bm{\theta}) = p_i\,\widetilde{g}(x_i,\lambda_i,\lambda_i')
$$
Hence, for any $\bm{x}\in\mathcal{X}$,
\begin{align*}
\Delta_i(\bm{x}) &:= \left|\frac{\partial \widetilde{F}}{\partial x_i}(\bm{x},\bm{\theta})
-\frac{\partial \widetilde{F}}{\partial x_i}(\bm{x},\bm{\bar\theta}_{t-1})\right| \\
&= \left|p_i\widetilde{g}(x_i,\lambda_i,\lambda_i') -\bar p_{t-1,i}\widetilde{g}(x_i,\bar\lambda_{t-1,i},\bar\lambda'_{t-1,i})\right| \\
&\leq |p_i-\bar p_{t-1,i}|\cdot|\widetilde{g}(x_i,\lambda_i,\lambda_i')| +\bar p_{t-1,i}\cdot \left|\widetilde{g}(x_i,\lambda_i,\lambda_i')-\widetilde{g}(x_i,\bar\lambda_{t-1,i},\bar\lambda'_{t-1,i})\right|
\end{align*}
Using Assumption~\ref{ass:uniform_sc_smooth}(iii) (Lipschitz continuity of $\widetilde{g}$) and $\bar p_{t-1,i} \leq 1$ we obtain the uniform bound
$$
\sup_{\bm{x}\in\mathcal{X}}\Delta_i(\bm{x}) \leq G_{\max}|\Delta p_{t-1,i}| +L_{g,\lambda}|\Delta\lambda_{t-1,i}| +L_{g,\lambda'}|\Delta\lambda'_{t-1,i}|
$$
where
$$
\Delta p_{t-1,i}:=p_i-\bar p_{t-1,i}, \qquad \Delta\lambda_{t-1,i}:=\lambda_i-\bar\lambda_{t-1,i}, \qquad \Delta\lambda'_{t-1,i}:=\lambda_i-\bar\lambda'_{t-1,i}
$$
Therefore,
$$
\sup_{\bm{x}\in\mathcal{X}} \|\nabla_{\bm{x}}\widetilde{F}(\bm{x},\bm{\bar\theta}_{t-1})-\nabla_{\bm{x}}\widetilde{F}(\bm{x},\bm{\theta})\|_2 \leq \|\bm{\delta}_{t-1}\|_2
$$
with coordinates
$$
\delta_{t-1,i}:= G_{\max}|\Delta p_{t-1,i}| +L_{g,\lambda}|\Delta\lambda_{t-1,i}| +L_{g,\lambda'}|\Delta\lambda'_{t-1,i}|
$$

Combining with stability:
$$
\|\bm{x}_t-\bm{x}^*\|_2 \leq \frac{1}{\mu}\|\bm{\delta}_{t-1}\|_2.
$$

\paragraph{Regret upper bound under good events}

From the quadratic bound and stability,
\begin{align*}
\mathbb{E}[R_t\mid\mathcal{F}_{t-1}] &\leq \frac{L}{2}\|\bm{x}_t-\bm{x}^*\|_2^2 \leq \frac{L}{2\mu^2}\|\bm{\delta}_{t-1}\|_2^2 \\
&=\frac{L}{2\mu^2}\sum_{i \in S} \Big(G_{\max}|\Delta p_{t-1,i}|+L_{g,\lambda}|\Delta\lambda_{t-1,i}| +L_{g,\lambda'}|\Delta\lambda'_{t-1,i}|\Big)^2
\end{align*}

On the good event $\mathcal{N}_{t-1}$, assume the boosted parameters stay within twice the confidence radii:
$$
|\Delta\lambda_{t-1,i}|\le \varepsilon_{t-1,i}, \qquad |\Delta\lambda'_{t-1,i}|\le \varepsilon_{t-1,i}, \qquad |\Delta p_{t-1,i}|\le \eta_{t-1,i}
$$
where
$$
\varepsilon_{t-1,i} := \frac{2b}{L_\mu}\sqrt{\frac{\log(1/\alpha)}{2n_{t-1,i}}}, \qquad \eta_{t-1,i} := 2\frac{L_\lambda}{L_\mu}\frac{2b}{C_b}\sqrt{\frac{\log(1/\alpha)}{2n_{t-1,i}}}
$$
Then under $\mathcal{N}_{t-1}$,
$$
\delta_{t-1,i}\le G_{\max}\eta_{t-1,i}+(L_{g,\lambda}+L_{g,\lambda'})\varepsilon_{t-1,i}.
$$
Hence
$$
\mathbb{E}[R_t\mid\mathcal{N}_{t-1}] \leq \frac{L}{2\mu^2} \sum_{i \in S} \Big(G_{\max}\eta_{t-1,i}+(L_{g,\lambda}+L_{g,\lambda'})\varepsilon_{t-1,i}\Big)^2 = \log(1/\alpha)\sum_{i \in S} \frac{\Gamma_i}{n_{t-1,i}}
$$
where
$$
\Gamma_i=\frac{Lb^2}{\mu^2L_\mu^2} \Bigg(\frac{4G_{\max}^2L_\lambda^2}{C_b^{2}}+\big(L_{g,\lambda}+L_{g,\lambda'}\big)^2+\frac{4G_{\max}\big(L_{g,\lambda}+L_{g,\lambda'}\big)L_\lambda}{C_b}\Bigg)
$$
(with the conservative radius of Section~\ref{subsec:algo-radii}, $\Gamma_i$ no longer depends on $i$; we keep the index for consistency with the notation of the sums).

\paragraph{Bounding $\mathbb{E}\big[1/n_{t,i}\big]$}

The remainder of the proof is identical to the proof of Theorem \ref{thm: regret sqrt(T)}.
Recall that $q_i=p_i G(b,\lambda_i)$ is the probability of an informative sample on a selected round. Under the present assumptions, $q_i>0$ and $C_i^*=G(x_i^*,m/B)>0$ for every arm: the optimal allocations are positive, and consistency, monotonicity, and strong concavity ensure $G(x,\lambda)>0$ for $x>0$. Thus the sampling bounds below involve all $K$ arms and finite instance-dependent denominators.
By Lemma~\ref{lemma: counting}, in the sense specified after \eqref{eq: regret decomposition},
$$
\mathbb{E}\Big[\frac{1}{n_{t-1,i}}\Big]\le \frac{10K}{tq_i}.
$$

\paragraph{Conclusion}
With $\alpha=T^{-3}$, as in the proof of Theorem~\ref{thm: regret sqrt(T)}, we have $\log(1/\alpha)=3\log T$, and summing over $t$ in the main phase gives
$$
\sum_{t=K\lfloor \log T \rfloor+1}^T \mathbb{E}[R_t\mid\mathcal{N}_{t-1}] \leq 3\log T\Big(\sum_{i \in S} \frac{10K\Gamma_i}{q_i}\Big)\sum_{t=1}^T \frac{1}{t} \leq \Big(\sum_{i \in S} \frac{30K\Gamma_i}{q_i}\Big) \big(\log T + (\log T)^2\big)
$$
This is true since $\sum_{t=1}^T 1/t \leq 1+\log T$. The factor $3$ comes from $\log(1/\alpha)$, which is not included in $\Gamma_i$.
Adding the contribution of the initialization phase ($K^2\log T$) and of bad events ($6K^2$), which are identical to the proof of Theorem \ref{thm: regret sqrt(T)}, we finally get, for $T\geq e^2$:

$$
\mathbb{E}[\mathcal{R}_T] \leq C_{\text{log}} \big((\log T)^2+\log T\big) + K^2\log T+6K^2
$$

where $C_{\text{log}}=\sum_{i \in S} \frac{30K\Gamma_i}{q_i}$.
\hfill$\square$

\section{Unknown Budget}\label{app_sec: unknown budget}
Assume for every round $t \in [T]$ the algorithm deals with unknown budget $B_t$ such that $0 < B_t \leq B_{\max}$. Define boosted parameters:
\begin{align*}
\hat{\lambda}_{t,i}^\pm &= \min\left\{\frac{M}{B_{\max}},\max\left\{\frac{m}{B_{\max}}, \hat{\lambda}_{t,i} \pm \frac{b}{L_\mu} \sqrt{\frac{3\log T}{2n_{t,i}}} \right\} \right\}, \\
\hat{p}_{t,i}^\pm &= \min\left\{1, \max\left\{0, \hat{p}_{t,i} \pm \frac{L_\lambda}{L_\mu}\frac{2b}{C_b} \sqrt{\frac{3\log T}{2n_{t,i}}} \right\} \right\}
\end{align*}
where now $C_b:=G(b,m/B_{\max})$, the analogue of the constant of Section~\ref{subsec:algo-radii} on the parameter domain $[m/B_{\max},\allowbreak M/B_{\max}]$. The factor $b$ comes from the range of the truncated observations and is not replaced by $B_{\max}$.

\begin{algorithm}[H]
\caption{\texttt{MG-UCB} (Pseudocode)}
\label{algo: MG_UCB pseudocode}
\footnotesize
\begin{algorithmic}[1]
\STATE \textbf{Input:} $T,K$, truncation level $b$.
\STATE $t\gets1$ (round index), $t'\gets1$ (estimation index), $i\gets1$.
\STATE \textbf{Initialization Phase}
($t=1,\ldots,K\lfloor\log T\rfloor$):
\STATE For each arm $i\in[K]$, allocate until the budget is exhausted
for $\lfloor\log T\rfloor$ rounds, observe rewards, and update parameters.
\STATE \textbf{Main Phase}
($t=K\lfloor\log T\rfloor+1,\ldots,T$):
\STATE $i\gets1$.
\WHILE{$t\leq T$}
    \STATE Compute confidence bounds
    $\{\hat{p}_{t',k}^{\pm},\hat{\lambda}_{t',k}^{\pm}\}_{k\in[K]}$.
    \STATE $(\bar{\lambda}_{t',i},\bar{\lambda}'_{t',i},\bar{p}_{t',i})
    \gets(\hat{\lambda}_{t',i}^{-},\hat{\lambda}_{t',i}^{+},
    \hat{p}_{t',i}^{+})$.
    \STATE $(\bar{\lambda}_{t',k},\bar{\lambda}'_{t',k},\bar{p}_{t',k})
    \gets(\hat{\lambda}_{t',k}^{+},\hat{\lambda}_{t',k}^{-},
    \hat{p}_{t',k}^{-})$ for all $k\neq i$.
    \STATE \textbf{Within-round allocation (water-filling):}
    \STATE $s\gets0$; $x_{t,k}(0)\gets0$ for all $k\in[K]$.
    \WHILE{$s<B_t$}
        \STATE Choose
        $I(s)\in\arg\max_{k\in[K]}
        \bar{p}_{t',k}\,
        \widetilde{g}(x_{t,k}(s),
        \bar{\lambda}_{t',k},\bar{\lambda}'_{t',k})$.
        \STATE Allocate $ds$ to $I(s)$:
        $x_{t,I(s)}(s+ds)\gets x_{t,I(s)}(s)+ds$.
        \STATE $s\gets s+ds$.
    \ENDWHILE
    \STATE Obtain
    $\bm{x}_t=(x_{t,1}(B_t),\ldots,x_{t,K}(B_t))$.
    \STATE Play allocation $\bm{x}_t$, observe feedback and reward.
    \STATE Update $\hat{p}_{t'+1,i},\hat{\lambda}_{t'+1,i}$ (for $\hat{\lambda}$, only if $x_{t,i}\geq b$ and $F_{t,i}\leq b$; see Section~\ref{subsec: estimation of lambda})
    only for arm $i$.
    \STATE $i\gets i+1$; $t\gets t+1$.
    \IF{$i>K$}
        \STATE $i\gets1$; $t'\gets t'+1$.
    \ENDIF
\ENDWHILE
\end{algorithmic}
\end{algorithm}

\begin{algorithm}[H]
\caption{\texttt{MG-UCB-$\Delta$} (Pseudocode)}
\label{algo: MG_UCB_discretized pseudocode}
\footnotesize
\begin{algorithmic}[1]
\STATE \textbf{Input:} $T,K$, truncation level $b$, step size $\Delta>0$.
\STATE $t\gets1$ (round index), $t'\gets1$ (estimation index), $i\gets1$.
\STATE \textbf{Initialization Phase}
($t=1,\ldots,K\lfloor\log T\rfloor$):
\STATE For each arm $i\in[K]$, allocate until the budget is exhausted
for $\lfloor\log T\rfloor$ rounds, observe rewards, and update parameters.
\STATE \textbf{Main Phase}
($t=K\lfloor\log T\rfloor+1,\ldots,T$):
\STATE $i\gets1$.
\WHILE{$t\leq T$}
    \STATE Compute confidence bounds
    $\{\hat{p}_{t',k}^{\pm},\hat{\lambda}_{t',k}^{\pm}\}_{k\in[K]}$.
    \STATE $(\bar{\lambda}_{t',i},\bar{\lambda}'_{t',i},\bar{p}_{t',i})
    \gets(\hat{\lambda}_{t',i}^{-},\hat{\lambda}_{t',i}^{+},
    \hat{p}_{t',i}^{+})$.
    \STATE $(\bar{\lambda}_{t',k},\bar{\lambda}'_{t',k},\bar{p}_{t',k})
    \gets(\hat{\lambda}_{t',k}^{+},\hat{\lambda}_{t',k}^{-},
    \hat{p}_{t',k}^{-})$ for all $k\neq i$.
    \STATE \textbf{Within-round allocation (discretized water-filling):}
    \STATE $s\gets0$; $x_{t,k}(0)\gets0$ for all $k\in[K]$.
    \WHILE{$s<B_t$}
        \STATE $\delta\gets\min\{\Delta,B_t-s\}$.
        \STATE Choose
        $I(s)\in\arg\max_{k\in[K]}
        \bar{p}_{t',k}\,
        \widetilde{g}(x_{t,k}(s),
        \bar{\lambda}_{t',k},\bar{\lambda}'_{t',k})$.
        \STATE Allocate $\delta$ to $I(s)$:
        $x_{t,I(s)}(s+\delta)\gets x_{t,I(s)}(s)+\delta$.
        \STATE $s\gets s+\delta$.
    \ENDWHILE
    \STATE Obtain
    $\bm{x}_t=(x_{t,1}(B_t),\ldots,x_{t,K}(B_t))$.
    \STATE Play allocation $\bm{x}_t$, observe feedback and reward.
    \STATE Update $\hat{p}_{t'+1,i},\hat{\lambda}_{t'+1,i}$ (for $\hat{\lambda}$, only if $x_{t,i}\geq b$ and $F_{t,i}\leq b$; see Section~\ref{subsec: estimation of lambda})
    only for arm $i$.
    \STATE $i\gets i+1$; $t\gets t+1$.
    \IF{$i>K$}
        \STATE $i\gets1$; $t'\gets t'+1$.
    \ENDIF
\ENDWHILE
\end{algorithmic}
\end{algorithm}

\texttt{MG-UCB} is designed to address the unknown-budget setting under the assumption that infinitesimal allocations are allowed within each round. The equivalence between greedy marginal allocation and separable concave maximization is classical \cite{FOX1966,IBARAKI1988}; we include a self-contained proof adapted to our setting. Our goal is to establish an equivalence between an algorithm that allocates resources incrementally at an infinitesimal intra-round scale until the budget is exhausted, and an algorithm that, given the budget in advance, solves the corresponding constrained optimization problem directly.

To establish this result, we proceed in two stages. We first analyze a discretized version of the problem, in which allocations are made in quanta of size $\Delta>0$, and prove an equivalence between the discrete water-filling allocation induced by \texttt{MG-UCB} and the optimizer of the corresponding discretized surrogate problem. We then pass to the continuous setting by letting $\Delta \to 0$, thereby recovering the desired continuous-time characterization.

This equivalence is formalized in Lemma \ref{lemma: equivalence RA-UCB and MG-UCB}, which establishes that \texttt{RA-UCB} and \texttt{MG-UCB} produce identical allocations under the same parameter estimates and realized budget, despite operating under different informational assumptions.

Fix a budget $B > 0$ and a step size $\Delta > 0$, define $M := \lfloor B/\Delta \rfloor$ (discretized budget). For each arm $i \in [K]$, let $f_i: [0,B] \rightarrow \mathbb{R}$ be non decreasing and concave. Define $\Delta$-marginal gains:
$$
m_i(n) := f_i\big( (n+1)\Delta \big)-f_i\big( n\Delta \big), \;\; n=0,1,\ldots,M-1
$$
Then concavity of $f_i$ implies $m_i(0) \geq m_i(1) \geq \ldots \geq m_i(M-1)$ for every $i \in [K]$. Consider the discretized Oracle problem:
\begin{equation} \label{opt: discretized Oracle}
(P): \quad
\begin{aligned}
&\underset{(n_1,\ldots,n_K) \in \{0,\ldots,M\}^K}{\text{max}} \quad \sum_{i=1}^K f_i(n_i\Delta) \\
&\text{subject to} \quad \sum_{i=1}^K n_i \leq M
\end{aligned}
\end{equation}
with a fixed deterministic tie-breaking rule to select a unique optimizer when the solution is not unique. 

Consider also the Greedy $\Delta$-discretized procedure: start from counters $N_i^0=0$ for every $i$ and for each step  $\ell=0,1,\ldots,M-1$ select arm:
$$
I_\ell \in \underset{i \in [K]}{\text{arg max }} \;m_i(N_i^\ell)
$$
with the same deterministic tie-breaking rule. Then update $N_{I_\ell}^{\ell+1}=N_{I_\ell}^{\ell}+1$ and keep counters for all arm $i \neq I_\ell$ unchanged. Let $N_i^M$ be the final counters.

\begin{lemma}[Equivalence of Oracle and $\Delta$-discretized procedure]\label{lemma: discretized equivalence}
   The Greedy $\Delta$-discretized procedure solves $(P)$, yielding an optimal allocation $(N_1^M,N_2^M,\ldots,N_K^M)$.
\end{lemma}

\begin{proof}
First we can notice that for any $n_i \in \{0,\ldots, M\}$:
$$
f_i(n_i \Delta) = f_i(0)+\sum_{q=1}^{n_i-1}\Big(f_i\big((q+1)\Delta\big)-f_i\big(q\Delta\big)\Big) = f_i(0)+ \sum_{q=0}^{n_i-1}m_i(q)
$$
So maximizing $\sum_i f_i(n_i \Delta)$ is equal, up to a constant, to maximizing $\sum_i \sum_q^{n_i-1} m_i(q)$ subject to $\sum_{i=1}^K n_i \leq M$.

In particular, this implies that the procedure selects at most $M$ marginals.

Let us define the multiset of marginals, labeled by $(i,q)$:
$$
\mathcal{M} = \{m_i(q): i \in [K], q = 0,1,\ldots,M-1 \}
$$
$\mathcal{M}$ is basically a $K \times M$ matrix. Let us sort these values in non increasing order (with deterministic tie-breaking rule on $(i,q)$). Let $S_M$ be the set of the first $M$ labeled elements in this total order, namely the top $M$ marginals. Now let us prove that in a run, the Greedy $\Delta$-discretized procedure selects exactly $S_M$. At step $\ell$, Greedy can choose only the current available marginals $\{m_i(N_i^\ell)\}_{i=1}^K$ i.e. the first not yet chosen element of each arm's marginal sequence. Because each sequence is non increasing (by concavity of $f_i$), the maximum among all not yet chosen elements of $\mathcal{M}$ must coincide with maximum among currently available ones.

Therefore Greedy always picks the globally largest remaining marginal (with fixed predetermined tie-breaking). By induction on $\ell$, after $M$ steps Greedy has selected exactly the top $M$ marginals, i.e. $S_M$.

In order to complete the argument, we want now to show that any feasible solution of $(P)$ selects the top $M$ marginals and therefore concluding that, under the same tie-breaking rule, running Greedy $\Delta$-discretized procedure and solving $(P)$ give the same result.

Let us take any feasible $(n_1,\ldots,n_K)$ with $\sum_i n_i \leq M$. Since each $f_i$ is not decreasing, w.l.o.g we can assume $\sum_i n_i = M$ (if it's not the case it means that exists a $j$ such that $n_j < M$, then, since $f_j$ is not decreasing, we can increase $n_j$ of 1 without making the objective decrease). The set of marginals selected by $(n_1,\ldots,n_K)$ is exactly the set:
$$
\mathcal{A} := \bigcup_{i=1}^K \{ (i,0),(i,1), \ldots, (i,n_i-1) \}
$$

and its value up to a constant is $\sum_{(i,q) \in \mathcal{A}}m_i(q)$. 

If $\mathcal{A} \neq S_M$, then there exists:
\begin{itemize}
    \item $(j, q_j) \in \mathcal{A} \backslash S_M$, a selected marginal not in the top $M$.
    \item $(k, q_k) \in S_M \backslash \mathcal{A}$, a top $M$ marginal not selected by $\mathcal{A}$.
\end{itemize}
By definition of $S_M$ we have $m_k(q_k) \geq m_j(q_j)$.

Now let us perform the following feasible exchange of resources:
$$
n_j' = n_j-1, \quad n_k'=n_k+1, \quad n_i'=n_i \;\; \forall i \notin \{j,k\}
$$
This exchange is feasible since it preserves $\sum_i n_i' = M$.
The objective difference is:

\begin{align}
    \sum_i f_i(n_i'\Delta)-\sum_if_i(n_i\Delta)&=f_j((n_j-1)\Delta )-f_j(n_j\Delta)+f_k((n_k+1)\Delta)-f_k(n_k\Delta ) \\&= m_k(n_k)-m_j(n_j-1)
\end{align}

Since $(k,q_k) \notin \mathcal{A}$ then $q_k \geq n_k$ (if $q_k$ marginal of arm $k$ has not been picked it means that $n_k$ stopped before). Therefore since $m_k(\cdot)$ is not increasing we have:
$$
m_k(n_k) \geq m_k(q_k)
$$
Similarly $(j,q_j) \in \mathcal{A}$ implies $q_j \leq n_j-1$, hence:
$$
m_j(n_j-1) \leq m_j(q_j)
$$
Therefore:
$$
m_k(n_k)-m_j(n_j-1) \geq m_k(q_k)-m_j(q_j)\geq 0
$$
Basically exchanging resources in this way does not decrease the objective while still being feasible. Iterating such exchange transforms any feasible solution into one that selects exactly $S_M$ without decreasing the objective. Hence any optimizer must select $S_M$.

Finally Greedy $\Delta$-discretized procedure selects $S_M$ and every optimizer selects $S_M$, we can then conclude that Greedy is optimal for $(P)$ i.e. $(N_1^M, \ldots, N_K^M)$ is an optimizer and Greedy allocation coincides with Oracle allocation.
\end{proof}

Define the continuous feasible set:
$$
\mathcal{X}(B) := \Big\{ \bm{x} \in \mathbb{R}_+^K: \sum_{i=1}^K x_i \leq B \Big\}
$$
and the continuous oracle problem:
$$
(C): \quad \max_{\bm{x}\in \mathcal{X}(B)} \sum_{i=1}^K f_i(x_i)
$$
Let $\bm{x^*} \in \underset{\bm{x} \in \mathcal{X}(B)}{\text{arg max}}\sum_i f_i(x_i)$ be any maximizer of $(C)$. Assume also that each $f_i$ is continuously differentiable on $[0,B]$ and that there exists a constant $G_{\max} < \infty$ such that:
$$
0 \leq f_i'(x) \leq G_{\max}, \quad \forall x \in [0,B], \; \forall i \in [K]
$$
This property is the same we ask to the surrogate objective of $\widetilde{\mathcal{P}}$ in Assumption \ref{ass:uniform_sc_smooth}.

\begin{proposition}\label{proposition: passing from discrete to continuous}
    For each $\Delta >0$, let $\bm{N}^M(\Delta):=(N_1^M,N_2^M,\ldots,N_K^M)$ be the counters produced by Greedy above and define corresponding allocations:
    $$
    \bm{x}^\Delta := \Delta \bm{N}^M(\Delta)=(\Delta N_1^M,\ldots,\Delta N_K^M) \in \mathcal{X}(B)
    $$
    Let 
    $$\text{OPT}(B):= \max_{\bm{x} \in \mathcal{X}(B)} \sum_{i=1}^K f_i(x_i)$$ $$\text{OPT}_\Delta(B):= \max_{\underset{\sum_in_i \leq M}{(n_1,\ldots,n_K) \in \{0,\ldots,M\}^K}} \sum_{i=1}^K f_i(n_i \Delta)$$
    Then the following statements hold:
    \begin{enumerate}[label=(\roman*)]
        \item 
        For every $\Delta >0$:
        $$
        0 \leq \text{OPT}(B)-\text{OPT}_\Delta(B) \leq KG_{\max} \Delta
        $$
        \item Let $\Delta_m$ be a sequence such that $\Delta_m \to 0$ when $m \to \infty$ and consider corresponding Greedy allocations $\bm{x}^{\Delta_m}$. Every limit point $\bm{\bar{x}}$ of $\{\bm{x}^{\Delta_m}\}_m$ satisfies $\bm{\bar{x}} \in \underset{\bm{x}\in \mathcal{X}(B)}{\text{arg max}} \sum_{i=1}^K f_i(x_i)$, where $\bm{x}^{\Delta_m}= \Delta_m (N_1^{M_m},\ldots,N_K^{M_m})$ and $M_m = \lfloor B/\Delta_m\rfloor$. Basically any infinitesimal Greedy allocation obtained as $\Delta \to 0$ limit is a solution to $(C)$.
        \item If in addition, the continuous objective $\sum_{i=1}^K f_i(x_i)$ is strongly concave on $\mathcal{X}(B)$, then optimizer $\bm{x^*}$ of $(C)$ is unique and
        $$
        \bm{x}^\Delta \to \bm{x^*} \;\; \text{as } \Delta \to 0
        $$
    \end{enumerate}
\end{proposition}

\begin{proof}
\textbf{(i):} Let $\bm{x^*}$ be any maximizer of $(C)$. Define the rounded-down solution
$$
\tilde{x}_i := \Big\lfloor \frac{x_i^*}{\Delta} \Big\rfloor \Delta, \quad \forall i \in [K]
$$
$\bm{\tilde{x}}=(\tilde{x}_1,\ldots,\tilde{x}_K)$ is feasible for the discretized problem since $\tilde{x}_i \in \{0,\Delta,\ldots,M\Delta\}$ and
$$
\sum_{i=1}^K \tilde{x}_i \leq \sum_{i=1}^Kx_i^* \leq B
$$
Moreover for every $i$: $0 \leq x_i^*-\tilde{x}_i<\Delta$. Hence $\sum_{i=1}^K(x_i^*-\tilde{x}_i) < K\Delta$. Since $0 \leq f_i'(x)\leq G_{\max}$ each $f_i$ is $G_{\max}$-Lipschitz on $[0,B]$:
$$
f_i(x_i^*)-f_i(\tilde{x}_i) \leq G_{\max} (x_i^*-\tilde{x}_i)
$$
Summing over $i$:
$$
\sum_{i=1}^K f_i(x_i^*)-\sum_{i=1}^Kf_i(\tilde{x}_i) \leq G_{\max} \sum_{i=1}^K ( x_i^* -\tilde{x}_i) < KG_{\max} \Delta
$$
By definition of $\text{OPT}_\Delta(B) \geq \sum_{i=1}^K f_i(\tilde{x}_i)$. Therefore:
$$
0 \leq \text{OPT}(B)-\text{OPT}_\Delta(B) = \sum_{i=1}^K f_i(x_i^*)- \text{OPT}_\Delta(B) \leq \sum_{i=1}^K f_i(x_i^*)-\sum_{i=1}^K f_i(\tilde{x}_i) < KG_{\max} \Delta
$$
and this proves (i).

\textbf{(ii):} Set $\mathcal{X}(B)$ is compact, hence any sequence $\{\bm{x}^{\Delta_m}\}_m \subseteq \mathcal{X}(B)$ admits a convergent subsequence with $\bm{x}^{\Delta_m} \to \bm{\bar{x}} \in \mathcal{X}(B)$. By continuity of $\sum_{i=1}^Kf_i(x_i)$:
$$
\sum_{i=1}^Kf_i(\bar{x}_i) = \lim_{m \to \infty} \sum_{i=1}^K f_i(x_i^{\Delta_m})
$$
Now by Lemma \ref{lemma: discretized equivalence}, for each $\Delta_m$, Greedy allocation $\bm{x}^{\Delta_m}$ achieves the discretized optimum, i.e.
$$
\sum_{i=1}^K f_i(x_i^{\Delta_m}) = \text{OPT}_{\Delta_m}(B)
$$
Combining this with (i) and taking $\Delta_m \to 0$ ($m \to \infty$) we get:
$$
\lim_{m \to \infty}\sum_{i=1}^K f_i(x_i^{\Delta_m}) = \lim_{m \to \infty} \text{OPT}_{\Delta_m}(B) = \text{OPT}(B)
$$
Therefore $\sum_{i=1}^K f_i(\bar{x}_i) = \text{OPT}(B)$, namely $\bm{\bar{x}}$ is a maximizer of $(C)$.

\textbf{(iii):} If $\sum_{i=1}^K f_i(x_i)$ is strongly concave on $[0,B]$ then the maximizer $\bm{x^*}$ of $(C)$ is unique. Since every limit point of $\{\bm{x}^\Delta\}_{\Delta \to 0}$ is a maximizer by (ii), all limit points must coincide with $\bm{x^*}$. Hence the whole family converges:
$$
\bm{x}^\Delta \to \bm{x^*} \text{ as } \Delta \to 0
$$
\end{proof}

\subsection{Proof of Lemma \ref{lemma: equivalence RA-UCB and MG-UCB}}
Fix a round $t$ and a realized budget value $B>0$. Condition on the filtration $\mathcal{F}_{t-1}$ generated by past allocations and observations, so that the estimators (and hence the boosted parameters)
$$
\widehat{\bm{\theta}}_{t-1}\qquad\text{and}\qquad\bar{\bm{\theta}}_{t-1}=(\bar{\bm{p}}_{t-1},\bar{\bm{\lambda}}_{t-1},\bar{\bm{\lambda}}'_{t-1})
$$
are deterministic. 
We recall the surrogate objective
\[
\widetilde{F}(\bm{x}; \bar{\bm{\theta}}_{t-1})
\;:=\;
\sum_{i=1}^K \bar p_{t-1,i}\,
\widetilde{G}\!\big(x_i,\bar\lambda_{t-1,i},\bar\lambda'_{t-1,i}\big),
\]
for any allocation \(\bm{x}\) in the feasible set
\[
\mathcal{X}(B)
\;:=\;
\Big\{
\bm{x}\in\mathbb{R}^K_+
\;:\;
\sum_{i=1}^K x_i \le B
\Big\}.
\]
By definition, \texttt{RA-UCB} (with known budget $B$) outputs an optimizer of:
$$
\max_{\bm{x}\in\mathcal{X}(B)}\widetilde{F}(\bm{x},\bar{\bm{\theta}}_{t-1})
$$
with the fixed deterministic tie-breaking rule.

Define for each arm $i\in[K]$ the univariate function
\begin{equation}\label{eq:def_fi_surrogate}
f_i(x):=\bar p_{t-1,i}\,\widetilde{G}\!\big(x,\bar\lambda_{t-1,i},\bar\lambda'_{t-1,i}\big),\qquad x\in[0,B]
\end{equation}
Then
$$
\widetilde{F}(\bm{x};\bar{\bm{\theta}}_{t-1})=\sum_{i=1}^K f_i(x_i).
$$
Under the standing assumptions of Section~\ref{sec: unknown budget}, for every admissible $\lambda,\lambda'$, $\widetilde{G}(x,\lambda,\lambda')$ is differentiable and concave in $x$ on $[0,B_{\max}]$ and non-decreasing (since its derivative is non-negative).
Hence, each $f_i$ is differentiable, concave and non-decreasing on $[0,B]$.

Moreover, by Assumption \ref{ass:uniform_sc_smooth} (bounded derivative of $\widetilde{g}=\partial_x\widetilde{G}$) and $\bar p_{t-1,i}\in[0,1]$, there exists $G_{\max}<\infty$ such that for all $x\in[0,B]$,
\begin{equation}\label{eq:fi_derivative_bound}
0\le f_i'(x)=\bar p_{t-1,i}\,\widetilde{g}\!\big(x,\bar\lambda_{t-1,i},\bar\lambda'_{t-1,i}\big)\le G_{\max}.
\end{equation}
Therefore the assumptions listed at the end of the previous subsection (concavity/monotonicity and bounded derivative) hold for the choice \eqref{eq:def_fi_surrogate}.

Fix a step size $\Delta>0$ and define $M:=\lfloor B/\Delta\rfloor$.
Consider the discretized oracle problem $(P)$ in \eqref{opt: discretized Oracle} with the functions $f_i$ defined in \eqref{eq:def_fi_surrogate}.
By Lemma \ref{lemma: discretized equivalence}, the Greedy $\Delta$-procedure that sequentially allocates one unit of $\Delta$ to the arm with largest discrete marginal
$$
m_i(n)=f_i((n+1)\Delta)-f_i(n\Delta)
$$
produces an optimizer of $(P)$ (under the same deterministic tie-breaking rule).

Observe now that the within-round discretized water-filling rule of \texttt{MG-UCB-$\Delta$} is exactly this Greedy $\Delta$-procedure, because at an intra-round state where arm $i$ has received $n_i$ increments (so $x_i=n_i\Delta$), allocating the next increment $\Delta$
to arm $i$ increases the objective by precisely $m_i(n_i)$, and concavity ensures the marginal sequence is non-increasing.
Hence, for every $\Delta>0$, the allocation $\bm{x}^\Delta$ produced by \texttt{MG-UCB-$\Delta$} maximizes the discretized objective over $\mathcal{X}(B)$ at grid resolution $\Delta$.

Let $\bm{x}^{\Delta}$ denote the allocation produced by \texttt{MG-UCB-$\Delta$} and write $\text{OPT}(B)$ and $\text{OPT}_\Delta(B)$ as in Proposition \ref{proposition: passing from discrete to continuous}.
Since $\bm{x}^{\Delta}$ is optimal for the discretized problem, we have
$$
\sum_{i=1}^K f_i(x_i^\Delta)=\text{OPT}_\Delta(B)
$$
By Proposition \ref{proposition: passing from discrete to continuous}\,(ii), for any sequence $\Delta_m\to 0$, every limit point $\bar{\bm{x}}$ of $\{\bm{x}^{\Delta_m}\}_m$ is a maximizer of the continuous problem
$$
\max_{\bm{x}\in\mathcal{X}(B)} \sum_{i=1}^K f_i(x_i)\;=\;\max_{\bm{x}\in\mathcal{X}(B)} \widetilde{F}(\bm{x};\bar{\bm{\theta}}_{t-1})
$$
On the other hand, the continuous-time water-filling procedure defining \texttt{MG-UCB} is precisely the $\Delta\to 0$ limit of \texttt{MG-UCB-$\Delta$},
and thus also outputs a maximizer of the same continuous problem (with the same tie-breaking convention).

Therefore, \texttt{MG-UCB} produces an optimizer of $\max_{\bm{x}\in\mathcal{X}(B)}\widetilde{F}(\bm{x};\bar{\bm{\theta}}_{t-1})$.
By definition, \texttt{RA-UCB} (when given budget $B$) also outputs an optimizer of the same problem under the same deterministic tie-breaking rule.
Consequently, the two allocations coincide:
$$
\bm{x}_t^{\texttt{MG-UCB}}(B,\widehat{\bm{\theta}}_{t-1})=\bm{x}_t^{\texttt{RA-UCB}}(B,\widehat{\bm{\theta}}_{t-1})
$$
which proves Lemma \ref{lemma: equivalence RA-UCB and MG-UCB}.

\hfill$\square$

\subsection{On the discretized variant \texttt{MG-UCB-$\Delta$} (heuristic discussion)}\label{app: discretization heuristic}
The regret guarantees of Section~\ref{sec: unknown budget} are stated for the idealized, continuous-increment algorithm \texttt{MG-UCB}. Here we explain why its implementable discretization \texttt{MG-UCB-$\Delta$} only adds a small, controlled error.

For fixed parameters, the effect of the discretization is fully controlled. If the surrogate marginal gains are bounded by $G_{\max}$, Lemma~\ref{lemma: discretized equivalence} and Proposition~\ref{proposition: passing from discrete to continuous}\,(i), together with the $G_{\max}$-Lipschitz continuity of the value function in the budget, show that at every round the allocation of \texttt{MG-UCB-$\Delta$} is optimal for the surrogate objective up to an additive error of at most $(K+1)G_{\max}\Delta$. In other words, \texttt{MG-UCB-$\Delta$} behaves as \texttt{RA-UCB} run with a $(K+1)G_{\max}\Delta$-approximate oracle. In the regret decomposition of Appendix~\ref{app_sec: regret of sqrt(T)}, the term that is non-positive for an exact oracle is then bounded by this quantity, which suggests an additional regret of order $\Delta T$, negligible for $\Delta\lesssim T^{-1/2}$.

Writing this argument out in full is a routine adaptation of the analysis of Appendix~\ref{app_sec: regret of sqrt(T)}, which we omit: the argument is carried out along the trajectory generated by \texttt{MG-UCB-$\Delta$} itself, and the selection condition $x_{t,i}\geq b$ of Section~\ref{subsec: estimation of lambda} must account for the rounding of allocations to the grid. We do not claim the polylogarithmic rate for \texttt{MG-UCB-$\Delta$}.

\section{Exponential Distribution Case}\label{app_sec: exponential distribution case}
We now specialize to the case where $G(x,\lambda)$ is the cumulative distribution function of an exponential distribution, and we derive the corresponding estimators and regret bounds. The exponential model is a convenient case study: besides being a common choice to model user engagement with an advertisement (in the online advertising framework, see e.g. \cite{CHAPELLEexponentialdistribution,YOSHIKAWAnonparametricdistribution}), its surrogate objective is smooth and strongly concave. The polylogarithmic guarantee additionally requires positive activation probabilities and an interior optimal allocation; it is not asserted for all exponential instances.
The exponential CDF is:
$$
G(x,\lambda_i) = 1 - e^{-\lambda_i x}
$$
We can first notice that $G$ is Lipschitz continuous in $\lambda$ with $L_\lambda = B$ and is non-decreasing in $\lambda$.

\paragraph{Estimators:}
According to Section \ref{subsec: estimation of lambda}, we define:
$$
\mu_i = \mathbb{E}[X_{t,i} \mid X_{t,i} \leq b] = \frac{1}{\lambda_i} - b \cdot \frac{e^{-\lambda_i b}}{1 - e^{-\lambda_i b}}
$$
Define the function $g$ by:
$$
g(x) = \frac{1}{x} - \frac{e^{-x}}{1 - e^{-x}}
$$
Thus, we can express $\mu_i$ as:
$$
\mu_i = b \cdot g(\lambda_i b)
$$
By inverting this Equation, we can write $\lambda_i$ as a function of $\mu_i$:
$$
\lambda_i = \frac{1}{b} g^{-1}\left( \frac{\mu_i}{b} \right)
$$
Therefore, the estimator of $\lambda_i$ at round $t$ is
$$
\hat{\lambda}_{t,i} = \frac{1}{b} g^{-1}\left( \frac{\Pi_{\mu(\Lambda)}(\hat{\mu}_{t,i})}{b} \right), \qquad \mu(\Lambda)=\big[b\,g(Mb/B),\, b\,g(mb/B)\big],
$$
with the projection of Section \ref{subsec: estimation of lambda} (recall that $g$ is decreasing).

\begin{lemma}
    Assumption \ref{ass: lipschitz of mu-1} is satisfied with $L_\mu = b^2 D$ where $D = |g'(Mb/B)|$.
\end{lemma}

\begin{proof}
Recall that $g(v)=\frac{1}{v}-\frac{e^{-v}}{1-e^{-v}}$. By direct computation, its first derivative is
$$
g'(v)=-\frac{1}{v^2}+\frac{e^{-v}}{(1-e^{-v})^2}
$$

We first show that $g$ is invertible on a compact domain. Recall that $v=\lambda b$ and, by assumption, each $\lambda_i$ is bounded in $\Lambda = [m/B,M/B]$, which implies $v\in[mb/B,Mb/B]=:[m_b,M_b]$, a compact interval.

A sufficient condition for invertibility on this interval is strict monotonicity. To establish this, observe that:
$$
(1-e^{-v})=e^{-v/2}(e^{v/2}-e^{-v/2})=2e^{-v/2}\sinh(v/2)
$$
which allows us to rewrite the derivative as:
$$
g'(v)= \frac{1}{4\sinh^2(v/2)}-\frac{1}{v^2}
$$
It is well known that $\sinh(x)>x$ for all $x>0$, and therefore:
$$
\sinh(v/2) > \frac{v}{2} \implies \sinh^2(v/2) > \frac{v^2}{4}
$$
This immediately yields:
$$
\frac{1}{4\sinh^2(v/2)} < \frac{1}{v^2} \;\; ,\forall v>0
$$
and hence $g'(v)<0$ for all $v>0$. We conclude that $g$ is strictly decreasing on $(0,\infty)$ and therefore invertible on the positive real axis, in particular on $[m_b,M_b]$.

We now turn to the Lipschitz continuity of the inverse function. To apply the \textit{inverse function theorem}, it is sufficient to show the existence of a constant $D>0$ such that $|g'(v)|\geq D$ for all $v\in[m_b,M_b]$. Since $|g'(v)|>0$ for all $v>0$ and $|g'|$ is continuous, the compactness of $[m_b,M_b]$ implies by the Weierstrass theorem that such a constant $D$ exists.

We can be more explicit by showing that $|g'(v)|$ is strictly decreasing on $(0,\infty)$, which implies that the minimum over $[m_b,M_b]$ is attained at $v=M_b$. To this end, we study the second derivative of $g$, given by:
$$
g''(v)=\frac{2}{v^3}-\frac{e^v(e^v+1)}{(e^v-1)^3}
$$
Setting $t=v/2$, we observe that:
$$
\frac{\cosh(t)}{4\sinh^3(t)}=\frac{e^v(e^v+1)}{(e^v-1)^3}
$$
so that $g''(v)$ can be rewritten as:
$$
g''(v) = \frac{2}{v^3}-\frac{\cosh(v/2)}{4\sinh^3(v/2)}
$$
Showing that $g''(v)>0$ for all $v>0$ is equivalent to proving:
$$
\frac{2}{v^3}>\frac{\cosh(v/2)}{4\sinh^3(v/2)}
$$
which, after the change of variable $x=v/2$, reduces to:
$$
\Big( \frac{\sinh(x)}{x}\Big)^3 > \cosh(x)
$$
By Lazarevic's inequality, for all $x>0$ and all $p\ge3$ it holds that:
$$
\Big( \frac{\sinh(x)}{x}\Big)^p > \cosh(x)
$$
and therefore the above inequality is satisfied. We conclude that $g''(v)>0$ for all $v>0$.

Since $g'(v)<0$ and $g''(v)>0$, the derivative $g'$ is strictly increasing while remaining negative, which implies that $|g'(v)|=-g'(v)$ is strictly decreasing on $(0,\infty)$. Consequently:
$$
D=\min_{v\in[m_b,M_b]}|g'(v)|=|g'(M_b)|
$$

Since the derivative of $g$ is thus uniformly bounded away from zero on $[m_b,M_b]$, the inverse function theorem ensures that $g^{-1}$ is Lipschitz continuous on this interval with Lipschitz constant $1/D$. Hence we obtain
$$
|\hat{\lambda}_{t,i}-\lambda_i| = \frac{1}{b} \left| g^{-1}\left(\frac{\hat{\mu}_{t,i}}{b}\right)-g^{-1}\left(\frac{\mu_i}{b}\right)\right|
\leq \frac{1}{b^2 D} |\hat{\mu}_{t,i}-\mu_i|
$$
So the inverse function $\mu^{-1}$ of Assumption \ref{ass: lipschitz of mu-1} is Lipschitz continuous with Lipschitz constant $1/L_\mu$ where $L_\mu = b^2 D$ and $D=|g'(M_b)|$.
This concludes the proof.
\end{proof}

\paragraph{Surrogate Objective:}
Now the optimum solves the optimization problem:
\begin{equation} 
\mathcal{P}(\bm{\lambda}, \bm{p}): \quad
\begin{aligned}
&\underset{\boldsymbol{x} \in \mathbb{R}^K_{+}}{\text{arg max}} \quad \sum_{i=1}^K p_i \big(1-e^{-\lambda_ix_i} \big) \\
&\text{subject to} \quad \sum_{i=1}^K x_i \leq B
\end{aligned}
\end{equation}
Now we can define the surrogate objective, by adding the extra parameter vector $\bm{\lambda'}$ as:
\begin{equation} 
\tilde{\mathcal{P}}(\bm{\lambda}, \bm{\lambda'}, \bm{p}): \quad
\begin{aligned}
&\underset{\boldsymbol{x} \in \mathbb{R}^K_{+}}{\text{arg max}} \quad \sum_{i=1}^K p_i \Big(1-\frac{\lambda_i'}{\lambda_i}e^{-\lambda_i x_i} \Big) \\
&\text{subject to} \quad \sum_{i=1}^K x_i \leq B
\end{aligned}
\end{equation}
Hence, consider the surrogate cumulative distribution function:
$$
\widetilde{G}(x,\lambda_i,\lambda_i')=1-\frac{\lambda_i'}{\lambda_i}e^{-\lambda_i x}
$$
Notice that $\widetilde{G}$ satisfies Consistency (Property \ref{prp: Consistency}), Boundedness and Monotonicity (Property \ref{prp: Boundness and Monotonicity}), Boosting Monotonicity (Property \ref{prp: Boosting Monotonicity}), and Uniform Bias Control (Property \ref{prp: Uniform Bias Control}) with $L_\Delta=B/m$. Moreover, it satisfies Lipschitz continuity in $\lambda$ and $\lambda'$ (Property \ref{prp:lipschitz_tildeG}) with constants $L_{\tilde{\lambda}} = MB(m+1)/m^2$ and $L_{\tilde{\lambda}'}=B/m$.

In addition, $\widetilde{G}$ is strongly concave and smooth on $[0,B]$, so Assumption \ref{ass:uniform_sc_smooth} holds with $\mu_G=m^2e^{-M}/B^2$ and $L_G=M^2/B^2$. Finally, the gradient Lipschitz continuity with respect to the parameters (Assumption~\ref{ass:uniform_sc_smooth}(iii)) holds with $L_{g,\lambda}=M$ and $L_{g,\lambda'}=1$.

\paragraph{Regret bounds:}
The regularity constants above allow us to specialize Theorem~\ref{thm: regret sqrt(T)} to the exponential model. The polylogarithmic bound in Theorem~\ref{thm: regret bound of logT} additionally requires positive activation probabilities and an interior optimal allocation (guaranteed by Assumption~\ref{ass: truncation level}); these conditions do not follow from the exponential form alone.
To do so, it remains to check the technical condition $bL_\lambda/L_\mu \geq \sqrt{8/3}$.
With $L_\lambda=B$ and $L_\mu=b^2D$, this condition becomes
$$
\frac{B}{b}\cdot\frac{1}{|g'(M_b)|}\geq \sqrt{\frac{8}{3}} .
$$
Since $|g'|$ is decreasing on $(0,\infty)$ with $\lim_{v\to0^+}|g'(v)|=1/12$, we have $|g'(M_b)|\leq 1/12$, and $B/b\geq1$; the condition therefore always holds in the exponential model.

In particular, by Theorem \ref{thm: regret sqrt(T)} we obtain
\begin{equation}\label{eq: regret sqrt(T) exponential case}
\mathbb{E}[\mathcal{R}_T] \leq \frac{20B}{bD}\left( \frac{4}{C_b}\sum_{i=1}^K\frac{1}{\sqrt{q_i}}+\Big( \frac{m^2+3m+M(m+1)}{m^2}\Big) \sum_{i=1}^K\frac{1}{q_i}\right)\cdot\sqrt{KT\log T}+K^2\log T +6K^2
\end{equation}
for $T\geq e^2$, where $q_i=p_i(1-e^{-\lambda_i b})$, $D=|g'(Mb/B)|$, and $C_b=G(b,m/B)=1-e^{-mb/B}$.
For exponential instances satisfying Assumption~\ref{ass:uniform_sc_smooth}, Theorem~\ref{thm: regret bound of logT} gives:
\begin{equation}\label{eq: regret log T exponential case}
\mathbb{E}[\mathcal{R}_T] \leq \Bigg(\frac{30KM^2\beta}{\mu^2B^2} \sum_{i=1}^K \frac{1}{q_i} \Bigg) (\log T)^2 + \Bigg(\frac{30KM^2\beta}{\mu^2B^2} \sum_{i=1}^K \frac{1}{q_i} + K^2 \Bigg)\log T+6K^2
\end{equation}
for $T\geq e^2$, with:
$$
\beta = \frac{4M^2}{b^2C_b^{2}D^2}+\frac{(M+1)^2}{b^2D^2}+\frac{4M(M+1)}{C_b\,b^2D^2}
$$
where $C_b = 1-e^{-mb/B}$ and $\mu = p_{\min} \mu_G = p_{\min}m^2e^{-M}/B^2$ is the strong concavity constant of the surrogate objective.

\paragraph{Uniformity over instances.}
Under Assumption~\ref{ass: truncation level}, the instance-dependent quantities entering the bound \eqref{eq: regret sqrt(T) exponential case} are controlled by the known level $b$: the constant $C_b=1-e^{-mb/B}$ is explicit, and for every $i\in[K]$, $q_i=p_i(1-e^{-\lambda_ib})\geq p_{\min}C_b$ when Assumption~\ref{ass:uniform_sc_smooth}(i) holds. Hence the $\widetilde{\mathcal{O}}(\sqrt{KT})$ bound holds with a constant depending only on $(K,B,b,m,M,p_{\min})$, uniformly over all exponential instances satisfying Assumption~\ref{ass: truncation level} with the same $b$ and $p_i\geq p_{\min}$. A bound that is uniform over all instances, including those whose optimal allocations have arbitrarily small or zero coordinates, is not covered by the present analysis: it would require letting $b$ tend to zero, which degrades both the sampling rate $q_i$ and the inverse-Lipschitz constant $L_\mu$.

\end{document}